%% file: version_arxiv.tex
\documentclass{article}

\PassOptionsToPackage{numbers, compress}{natbib}

\usepackage[final]{neurips_2019}

\usepackage[utf8]{inputenc} 
\usepackage[T1]{fontenc}    
\usepackage{hyperref}       
\usepackage{url}            
\usepackage{booktabs}       
\usepackage{amsfonts}       
\usepackage{nicefrac}       
\usepackage{microtype}      

\usepackage{copied_style}

\usepackage{epsfig}
\usepackage{graphicx}
\usepackage{amsmath}
\usepackage{amssymb}

\usepackage{stmaryrd}
\usepackage{amsthm}
\usepackage[dvipsnames]{xcolor}
\usepackage{bbold}
\usepackage{subcaption}

\graphicspath{{figs/}}

\newcommand{\R}{\mathbb{R}}
\newcommand{\Scal}{\mathcal{S}}
\newcommand{\x}{\mathbf{x}}

\newcommand{\vv}{\mathbf{v}}

\newcommand{\epsi}{\varepsilon}
\DeclareMathOperator*{\E}{\mathbb{E}}
\DeclareMathOperator*{\var}{\mathrm{var}}
\newcommand{\Id}{\mathrm{Id}}
\newcommand{\Tr}{\mathrm{Tr}}
\newcommand{\rot}{\mathrm{rot}}

\newcommand{\yt}{\widetilde{y}}
\newcommand{\yh}{\widehat{y}}

\newtheorem{theorem}{Theorem}

\newtheorem{corollary}{Corollary}

\newcommand{\g}[1]{{\color{Green} #1}}

\title{Input Similarity from the Neural Network Perspective}

\author{Guillaume Charpiat$^1$\\
   \And
   Nicolas Girard$^2$\\
   \And
   Loris Felardos$^1$\\
   \And
   Yuliya Tarabalka$^{2,3}$\\
   \and
   $^1$ TAU team,  Inria Saclay, LRI, Université Paris-Sud \\
   $^2$ Université Côte d'Azur, Inria Sophia-Antipolis\\
   $^3$ LuxCarta Technology\\
   {\tt\small firstname.lastname@inria.fr}
}

\begin{document}

\maketitle

\begin{abstract}
  We first exhibit a multimodal image registration task, for which a neural network trained on a dataset with noisy labels reaches almost perfect accuracy, far beyond noise variance. This surprising auto-denoising phenomenon can be explained as a noise averaging effect over the labels of similar input examples.
  This effect theoretically grows with the number of similar examples; the question is then to define and estimate the \emph{similarity} of examples.

We express a proper definition of similarity, from the neural network perspective, \ie we quantify how undissociable two inputs $A$ and $B$ are, taking a machine learning viewpoint: how much a parameter variation designed to change the output for $A$ would impact the output for $B$ as well?

We study the mathematical properties of this similarity measure, and show how to use it on a trained network to estimate sample density, in low complexity, enabling new types of statistical analysis for neural networks.
We analyze data by retrieving samples perceived as similar by the network, and are able to quantify the denoising effect without requiring true labels.
We also propose, during training, to enforce that examples known to be similar should also be seen as similar by the network, and notice speed-up training effects for certain datasets.

\end{abstract}

\input{self_denoising}


\input{similarity_def}

\input{higher_dim}

\input{neighbors}

\input{enforce_sim}

\input{back_to_self_denoising}

\input{conclusion}



{\small
\bibliographystyle{plainnat}
\bibliography{similarity,not_anonymous}
}

\setcounter{theorem}{0}
\setcounter{corollary}{0}
\input{annexe}

\end{document}

%% file: self_denoising.tex
\section{Motivation: Dataset self-denoising}
\label{sec:denoising}


In remote sensing imagery, data is abundant but noisy \cite{mnih2012learning}. For instance RGB satellite images and binary cadaster maps (delineating buildings) are numerous but badly aligned for various reasons (annotation mistakes, atmosphere disturbance, elevation variations...). In a recent preliminary work~\cite{anonymous}, we tackled the task of automatically registering these two types of images together with neural networks, considering as ground truth a dataset of hand-picked relatively-well-aligned areas
\cite{maggiori2017dataset}, and hoping the network would be able to learn from such a dataset of imperfect alignments.
Learning with noisy labels is indeed an active topic of research \cite{sukhbaatar2014training,natarajan2013learning,li2017learning}.

For this, we designed an iterative approach: train, then test on the training set and re-align it accordingly; repeat (for 3 iterations).
The results were surprisingly good, yielding far better alignments than the ground truth it learned from, both qualitatively (Figure~\ref{fig:qualitative_results}) and quantitatively (Figure~\ref{fig:accuracies}, obtained on manually-aligned data): the median registration error dropped from 18 pixels to 3.5 pixels, which is the best score one could hope for, given intrinsic ambiguities in such registration task.
To check that this performance was not due to a subset of the training data that would be perfectly aligned, we added noise to the ground truth and re-trained from it: the new results were about as good again (dashed lines). Thus the network did learn almost perfectly just from noisy labels.



\begin{figure}
\begin{minipage}[c]{0.42\linewidth}
  \includegraphics[width=0.99\linewidth]{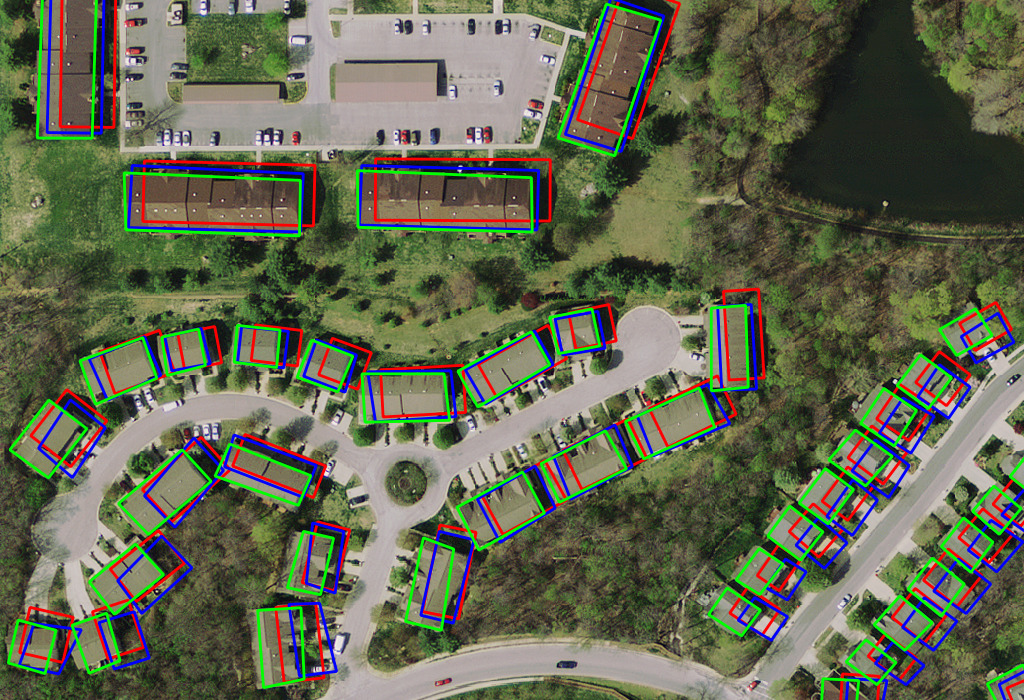}

  \smallskip\vspace{1mm}
        \caption{Qualitative alignment results \cite{anonymous} on a crop of bloomington22 from the Inria dataset \cite{maggiori2017dataset}. \textcolor{red}{Red: initial dataset annotations}; \textcolor{blue}{blue: aligned annotations round 1}; \g{green:  aligned annotations round 2}.}
        \label{fig:qualitative_results}
\end{minipage} \hfill
\begin{minipage}[c]{0.55\linewidth}
        \vspace{-3mm}
        \hspace{-2.5mm}\includegraphics[width=1.05\linewidth]{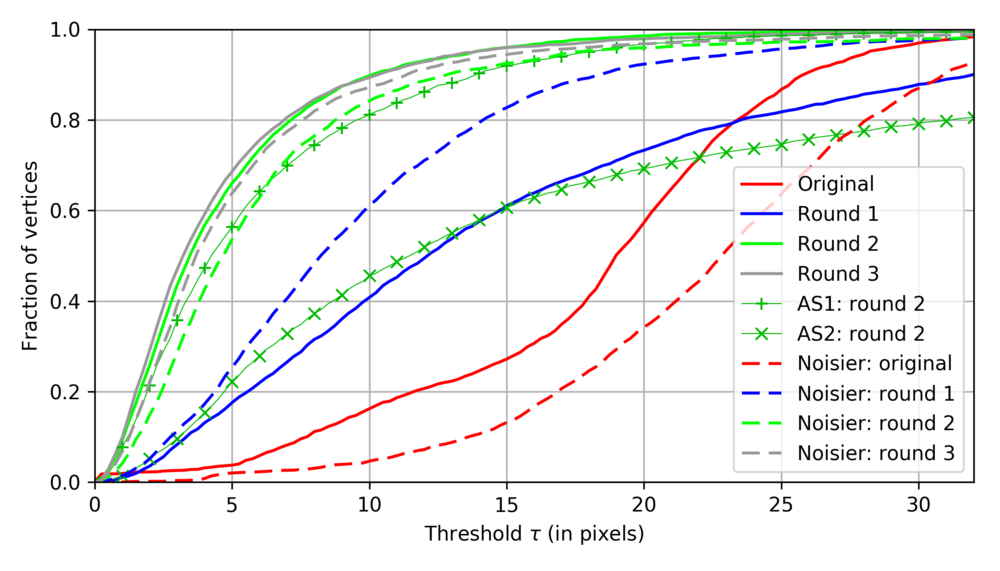}
        \vspace{-4mm}
        \caption{Accuracy cumulative distributions \cite{anonymous} measured with the manually-aligned annotations of bloomington22 \cite{maggiori2017dataset}. Read as: fraction of image pixels whose registration error is less than threshold $\tau$.}
        \label{fig:accuracies}
\end{minipage}
\end{figure}


An explanation for this self-denoising phenomenon is proposed in \cite{noise2noise} as follows.
Let us consider a regression task, with a $L^2$ loss, and where true labels $y$ were altered with i.i.d.~noise $\epsi$ of variance $v$. Suppose a same input $\x$ appears $n$ times in the training set, thus with $n$ different labels $y_i = y + \epsi_i$. The network can only output the same prediction for all these $n$ cases (since the input is the same), and the best option, considering the $L^2$ loss, is to predict the average $\frac{1}{n} \sum_i y_i$, whose distance to the true label $y$ is $O(\frac{v}{\sqrt{n}})$. Thus a denoising effect by a factor $\sqrt{n}$ can be observed.
However, the exact same point $\x$ is not likely to appear several times in a dataset (with different labels). Rather, relatively \emph{similar} points may appear, and the amplitude of the self-denoising effect will be a function of their number. Here, the similarity should reflect the neural network perception (similar inputs yield the same output) and not an \emph{a priori} norm chosen on the input space.

The purpose of this article is to express the notion of similarity from the network's point of view. We first define it, and study it mathematically, in Section~\ref{sec:sim}, in the one-dimensional output case for the sake of simplicity.
Higher-dimensional outputs are dealt with in Section~\ref{sec:higher}.
We then compute, in Section~\ref{sec:density}, the number of neighbors (\ie, of similar samples), and propose for this a very fast estimator. This brings new tools to analyze already-trained networks.
As they are differentiable and fast to compute, they can be used during training as well, \eg, to enforce that given examples should be perceived as similar by the network (\cf Section~\ref{sec:enforce}).
Finally, in Section~\ref{backtodenoise}, we apply the proposed tools to analyze a network trained with noisy labels for a remote sensing image alignment task, and formalize the self-denoising phenomenon, quantifying
its effect,
extending~\cite{noise2noise} to real datasets.


%% file: similarity_def.tex
\section{Similarity}
\label{sec:sim}

\subsection{Notions of similarities}

The notion of similarity between data points is an important topic in the machine learning literature, obviously in domains such as image retrieval, where images similar to a query have to be found; but not only. For instance when training auto-encoders, the quality of the reconstruction is usually quantified as the $L^2$ norm between the input and output images. Such a similarity measure is however questionable, as color comparison, performed pixel per pixel, is a poor estimate of human perception: the $L^2$ norm can vary a lot with transformations barely noticeable to the human eye such as small translations or rotations (for instance on textures), and does not carry semantic information, \ie whether the same kind of objects are present in the image.

Therefore, so-called \emph{perceptual losses} \cite{johnson2016perceptual} were introduced to quantify image similarity: each image is fed to a standard pre-trained network such as VGG, and the activations in a particular intermediate layer are used as descriptors of the image \cite{gatys2015texture,gatys2015neural}. The distance between two images is then set as the $L^2$ norm between these activations. Such a distance carries implicitly semantic information, as the VGG network was trained for image classification. However, the choice of the layer to consider is arbitrary. In the ideal case, one would wish to combine the information from all layers, as some are more abstract and some more detail-specific. But then the particular weights chosen to combine the different layers would also be arbitrary. Would it be possible to get a canonical similarity measure, well posed theoretically?

More importantly, the previous litterature does not consider the notion of input similarity from the point of view of the neural network that is being used, but from the point of view of another one (typically, VGG) which aims at imitating human perception. A notable exception~\cite{koh2017understanding} transposes to machine learning the concept of influence functions in statistics~\cite{Hampel1974}. The differences with our definition of similarity might seem slight at first glance but they have important consequences: first, making use of the loss (and of its gradient and its Hessian) in the similarity measure has the issue that the expressed quantities are not intrinsic to the neural network but also depend on the optimization criterion used during training, which is problematic in the case of noisy labels as, at training convergence, the gradient of the loss with respect to the output points in random directions (remaining label noise that the network is not able to overfit). Second, the inverse of the Hessian appears in influence functions, while our definition makes use of gradients only. Another interesting related work~\cite{NNKernel} expresses neural networks as a kernel between test point and training points. Once again however the kernel definition relies on the training criterion.

As a supplementary motivation for this study, neural networks are black boxes difficult to interpret, and showing which samples a network considers as similar would help to explain its decisions. Also, the number of such similar examples would be a key element for confidence estimation at test time.

In this section we define a proper, intrinsic notion of similarity as seen by the network, relying on how easily it can distinguish different inputs.

\subsection{Similarity from the point of view of the parameterized family of functions}

Let $f_\theta$ be a parameterized function, typically a neural network already trained for some task, and $\x, \x'$ possible inputs, for instance from the training or test set.
For the sake of simplicity, let us suppose in a first step that $f_\theta$ is real valued.
To express the similarity between $\x$ and $\x'$, as seen by the network, one could compare the output values $f_\theta(\x)$ and $f_\theta(\x')$. This is however not very informative, and a same output might be obtained for different reasons.

Instead, we define similarity as the influence of $\x$ over $\x'$,
by quantifying how much an additional training step for $\x$ would change the output for $\x'$ as well. If $\x$ and $\x'$ are very different from the point of view of the neural network, changing $f_\theta(\x)$ will have little consequence on $f_\theta(\x')$. Vice versa, if they are very similar, changing $f_\theta(\x)$ will greatly affect $f_\theta(\x')$ as well.

\begin{figure}[ht]
\hfill \begin{minipage}[c]{0.4\linewidth}
  \includegraphics[width=4.2cm]{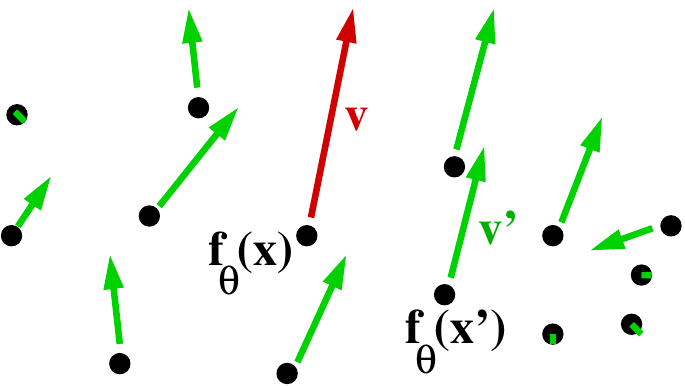}
\end{minipage}
\begin{minipage}[r]{0.38\linewidth}
  \caption{\label{fig:kernel} Moves in the space of outputs. We quantify the influence of a data point $\x$ over another one $\x'$ by how much the tuning of parameters $\theta$ to obtain a desired output change $\vv$ for $f_\theta(\x)$ will affect $f_\theta(\x')$ as well.} 
\end{minipage}
\end{figure}

Formally, if one wants to change the value of $f_\theta(\x)$ by a small quantity $\epsi$, one needs to update $\theta$ by $\delta\theta = \epsi\, \frac{ \nabla_{\!\theta} f_\theta(\x) }{ \| \nabla_{\!\theta} f_\theta(\x) \|^2} $. Indeed, after the parameter update, the new value at $\x$ will be:
$$f_{\theta + \delta \theta} (\x) \;\;=\;\; f_\theta(\x) + \nabla_{\!\theta} f_\theta(\x) \cdot \delta \theta + O(\|\delta\theta\|^2)
\;\;=\;\; f_\theta(\x) + \epsi  + O(\epsi^2).$$
This parameter change induces a value change at any other point $\x'$ :
$$f_{\theta + \delta \theta} (\x') \;=\; f_\theta(\x') + \nabla_{\!\theta} f_\theta(\x') \cdot \delta \theta + O(\|\delta\theta\|^2)
\;=\; f_\theta(\x') + \epsi  \frac{ \nabla_{\!\theta} f_\theta(\x') \cdot \nabla_{\!\theta} f_\theta(\x)  }{ \| \nabla_{\!\theta} f_\theta(\x) \|^2}  + O(\epsi^2).$$
Therefore the kernel $\displaystyle k^N_\theta(\x,\x') = \frac{ \nabla_{\!\theta} f_\theta(\x) \cdot \nabla_{\!\theta} f_\theta(\x')   }{ \| \nabla_{\!\theta} f_\theta(\x) \|^2}$
represents the influence of $\x$ over $\x'$: if one wishes to change the output value $f_\theta(\x)$ by $\epsi$, then $f_\theta(\x')$ will change by $\epsi\, k^N_\theta(\x,\x')$. In particular, if $k^N_\theta(\x,\x')$ is high, then $\x$ and $\x'$ are not distinguishable from the point of view of the network, as any attempt to move $f_\theta(\x)$ will move $f_\theta(\x')$ as well (see Fig.~\ref{fig:kernel}). We thus see $k^N_\theta(\x,\x')$ as a measure of similarity. Note however that $k^N_\theta(\x,\x')$ is not symmetric.

\newpage
  
\paragraph{Symmetric similarity: correlation}
Two symmetric kernels natural arise:
the inner product:
\begin{equation}
  \label{eq:innerproductkern}
  k^I_\theta(\x,\x') \;=\; \nabla_{\!\theta} f_\theta(\x)  \cdot   \nabla_{\!\theta} f_\theta(\x')
\end{equation}
and its normalized version, the correlation:
\begin{equation}
  \label{eq:symkern}
  k^C_\theta(\x,\x') = \frac{ \nabla_{\!\theta} f_\theta(\x)  }{ \| \nabla_{\!\theta} f_\theta(\x) \|}  \cdot \frac{ \nabla_{\!\theta} f_\theta(\x')   }{ \| \nabla_{\!\theta} f_\theta(\x') \|}
\end{equation}
which has the advantage of being bounded (in $[-1,1]$), 
thus expressing similarity in a usual meaning.


\subsection{Properties for vanilla neural networks}

Intuitively, inputs that are similar from the network perspective should produce similar outputs; we can check that $k^C_\theta$ is a good similarity measure in this respect (all proofs are deferred to the Appendix):
\begin{theorem} \label{basicnet}
  For any real-valued neural network $f_\theta$
  whose last layer is a linear layer (without any parameter sharing) or a standard activation function thereof (sigmoid, tanh, ReLU...),
  and for any inputs $\x$ and $\x'$,
    $$\nabla_{\!\theta} f_\theta(\x) = \nabla_{\!\theta} f_\theta(\x') \;\;\implies\;\;f_\theta(\x) = f_\theta(\x') \,.$$ 
\end{theorem}

\begin{corollary} \label{alphasim}
Under the same assumptions, for any inputs $\x$ and $\x'$, 
$$\begin{array}{crcl}
& k^C_\theta(\x, \x') = 1  & \implies & \nabla_{\!\theta} f_\theta(\x) = \nabla_{\!\theta} f_\theta(\x') \,, \vspace{1mm} \\
\mathrm{hence} & k^C_\theta(\x, \x') = 1  & \implies & f_\theta(\x) = f_\theta(\x') \,. \\
\end{array}$$
\end{corollary}
  
Furthermore,
\begin{theorem} \label{basicnet2}
  For any real-valued neural network $f_\theta$ without parameter sharing,
  if $\nabla_{\!\theta} f_\theta(\x) = \nabla_{\!\theta} f_\theta(\x')$ for two inputs $\x, \x'$,
  then all useful activities computed when processing $\x$ are equal to the ones obtained when processing $\x'$.
\end{theorem}

We name \emph{useful} activities all activities $a_i(\x)$ whose variation would have an impact on the output, \ie all the ones satisfying $\frac{d f_\theta(\x) }{da_i} \neq 0$. This condition is typically not satisfied when the activity
is negative and followed by a ReLU, or when it is multiplied by a 0 weight,
or when all its contributions to the output cancel one another (\eg, a sum of two neurons with opposite weights: $f_\theta(\x) = \sigma( a_i(\x) ) - \sigma( a_i(\x) )$).




\paragraph{Link with the \emph{perceptual loss}}
For a
vanilla network without parameter sharing, the gradient $\nabla_{\!\theta} f_\theta(\x)$ is  a list of
coefficients $\nabla_{\!w_i^j} f_\theta(\x) = \frac{d f_\theta(\x) }{db_j}\, a_i(\x)$, where $w_i^j$ is the parameter-factor that multiplies the input activation $a_i(\x)$ in neuron $j$, and of coefficients $\nabla_{\!b_j} f_\theta(\x) = \frac{d f_\theta(\x) }{db_j}$ for neuron biases, which we will consider as standard parameters $b_j = w_0^j$ that act on a constant activation $a_0(\x) = 1$, yielding $\nabla_{\!w_0^j} f_\theta(\x) = \frac{d f_\theta(\x) }{db_j}\, a_0(\x)$.
Thus the gradient $\nabla_{\!\theta} f_\theta(\x)$ can be seen as a list of 
all activation values $a_i(\x)$ multiplied by the potential impact on the output $f_\theta(\x)$ of the neurons $j$ using them, \ie $\frac{d f_\theta(\x) }{db_j}$. Each activation appears in this list as many times as it is fed to different neurons.
The similarity between two inputs then rewrites:
$$k^I_\theta(\x,\x') = \!\!\!\sum_{\mathrm{activities\;} i}\!\!\! \lambda_i(\x,\x')\; a_i(\x) \, a_i(\x')
\;\;\;\;\;\; \text{where}
\;\;\;\;\;\; 
\lambda_i(\x,\x') = \!\!\!\!\!\!\sum_{\mathrm{neuron\;} j \mathrm{\;using\;} a_i}
\frac{d f_\theta(\x) }{db_j} \frac{d f_\theta(\x') }{db_j}$$
are data-dependent importance weights.
Such weighting schemes on activation units naturally arise when expressing intrinsic quantities; the use of natural gradients would bring invariance to re-parameterization \cite{RiemanNN_I, RiemanNN_II}.
On the other hand, the inner product related to the perceptual loss would be
$$\sum_{\mathrm{activities\;} i \neq 0}\;\lambda_{\mathrm{layer}(i)} \;a_i(\x) \, a_i(\x')$$
for some arbitrary fixed layer-dependent weights $\lambda_{\mathrm{layer}(i)}$. 


\subsection{Properties for parameter-sharing networks}
When sharing weights, as in convolutional networks,
the gradient $\nabla_{\!\theta} f_\theta(\x)$
is made of the same coefficients 
(impact-weighted activations)
but summed over shared parameters.
Denoting by $\mathcal{S}(i)$ the set of (neuron, input activity) pairs where the parameter $w_i$ is involved,
$$k^I_\theta(\x,\x') \;\;=\;\; \sum_{\text{params}\;i} \left( \sum_{(j,k)\in\mathcal{S}_i} a_{k}(\x) \frac{d f_\theta(\x) }{db_{j}} \right) \left( \sum_{(j,k)\in\mathcal{S}_i}  a_{k}(\x') \frac{d f_\theta(\x') }{db_{j}} \right)$$
Thus, in convolutional networks, $k^I_\theta$ similarity does not imply similarity of first layer activations anymore, but only of their (impact-weighted) spatial average.
More generally, any invariance introduced by a weight sharing scheme in an architecture will be reflected in the similarity measure $k^I_\theta$, which is expected as $k^I_\theta$ was defined as the input similarity \emph{from the neural network perspective}.

Note that this type of objects was recently studied from an optimization viewpoint
under the name of Neural Tangent Kernel~\cite{NTK,LazyTraining} in the infinite layer width limit.






%% file: higher_dim.tex
\section{Higher output dimension}
\label{sec:higher}

Let us now study the more complex case where $f_\theta(\x)$ is a vector $\left( f^i_\theta(\x) \right)_{i \in [1,d]}$ in $\R^d$ with $d > 1$.
Under a mild hypothesis on the network (output expressivity), always satisfied unless specially designed not to:
\begin{theorem} \label{th:multidim}
The optimal parameter change $\delta \theta$ to push $f_\theta(\x)$ in a direction $\vv \in \R^d$ (with a force $\epsi \in \R$), \ie such that $f_{\theta + \delta \theta} (\x) - f_\theta(\x) = \epsi \vv$,
  induces at any other point $\x'$ the following output variation:
\begin{equation}
  \label{eq:multidim}
f_{\theta + \delta \theta} (\x') - f_\theta(\x')  = \epsi \, K_\theta(\x',\x)\,  K_\theta(\x,\x)^{-1}\, \vv \, +\, O(\epsi^2)
\end{equation}
where the $d \times d$ kernel matrix $K_\theta(\x',\x)$ is defined by $K^{ij}_\theta(\x',\x) =  \nabla_{\!\theta} f^i_\theta(\x') \cdot \nabla_{\!\theta} f^j_\theta(\x)$.
\end{theorem}
The similarity kernel is now a matrix and not just a single value, as it describes the relation between moves $\vv \in \R^d$.
Note that these matrices $K_\theta$ are only $d \times d$ where $d$ is the output dimension. They are thus generally small and easy to manipulate or inverse.

\paragraph{Normalized similarity matrix}
The unitless symmetrized, normalized version of the kernel (\ref{eq:multidim}) is:
\begin{equation}
  \label{eq:multidimkern}
  K_\theta^C(\x,\x') \;=\;  K_\theta(\x,\x)^{-1/2}\; K_\theta(\x,\x')\;  K_\theta(\x',\x')^{-1/2} \;.
\end{equation}
It has the following properties:
its coefficients are bounded, in $[-1,1]$;
its trace is at most $d$;
its (Frobenius) norm is at most $\sqrt{d}$;
self-similarity is identity: $\forall \x, \,\;K_\theta^C(\x,\x) = \Id$;
the kernel is symmetric, in the sense that  $K_\theta^C(\x',\x) = K_\theta^C(\x,\x')^T$.

\paragraph{Similarity in a single value}
\label{sec:singlevalue}
To summarize the similarity matrix $K_\theta^C(\x,\x')$ into a single real value in $[-1,1]$, we consider:\vspace{-2mm}
\begin{equation}
  \label{eq:multidimkernsum}
  k_\theta^C(\x,\x') \;=\;
  \frac{1}{d} \,\Tr\, K^C_\theta(\x,\x') \;.
\end{equation}
It can be shown indeed that if $k_\theta^C(\x,\x')$ is close to 1, then $K_\theta^C(\x,\x')$ is close to $\Id$, and reciprocally. See Appendix~\ref{sec:high2} for more details and a discussion about the links between $\frac{1}{d} \,\Tr\, K^C_\theta(\x,\x')$ and $\big\| K_\theta^C(\x,\x') - \Id\big\|_F$.

\paragraph{Metrics on output: rotation invariance}
Similarity in $\R^d$ might be richer than just estimating distances in $L^2$ norm. For instance, for our 2D image registration task, the network could be known (or desired) to be equivariant to rotations.
The similarity between two output variations $\vv$ and $\vv'$ can be made rotation-invariant by applying the rotation that best aligns $\vv$ and $\vv'$ beforehand. This can actually be easily computed in closed form and yields:
$$  k^{C, \rot}_\theta(\x,\x') \;=\; \frac{1}{2}  \sqrt{ \big\| K_\theta^C(\x,\x') \big\|_F^2 + 2 \det K_\theta^C(\x,\x')} \; . $$

Note that other metrics are possible in the output space. For instance, the loss metric quantifies the norm of a move $\vv$ by its impact on the loss $\frac{dL(y)}{dy}\big|_{f_\theta(\x)}(\vv)$. It has a particular meaning though, is not intrinsic, and is not always relevant, \eg in the noisy label case seen in Section \ref{sec:denoising}.


\paragraph{The case of classification tasks}
When the output of the network is a probability distribution $p_{\theta,\x}(c)$, over a finite number of given classes $c$ for example, it is natural from an information theoretic point of view to rather consider $f^c_\theta(\x) =  - \log p_{\theta,\x}(c)$. This is actually the quantities computed in the pre-softmax layer from which common practice directly computes the cross-entropy loss.

It turns out that the $L^2$ norm of variations $\delta \!f$ in this space naturally corresponds to the Fisher information metric, which quantifies the impact of parameter variations $\delta \theta$ on the output probability $p_{\theta,\x}$, as $\mathrm{KL}(p_{\theta,\x}||p_{\theta+\delta\theta,\x})$. The matrices $K_{\theta}(\x,\x) = \big(\, \nabla_{\theta} f_\theta^c(\x) \cdot \nabla_{\theta} f_\theta^{c'}(\x) \,\big)_{c,c'}$ and $F_{\theta,\x} = \E_c \left[ \nabla_{\theta} f_\theta^c(\x)\; \nabla_{\theta} f_\theta^c(\x)^T \right]$ are indeed to each other what correlation is to covariance.
Thus the quantities defined in Equation (\ref{eq:multidimkernsum}) already take into account information geometry when applied to the pre-softmax layer, and do not need supplementary metric adjustment.

\paragraph{Faster setup for classification tasks with many classes}
\label{Lorisproj}
In a classification task in $d$ classes with large $d$, the computation of $d \times d$ matrices may be prohibitive. As a workaround, for a given input training sample $\x$, the classification task can be seen as a binary one (the right label $c_R$ \vs the other ones), in which case the $d$ outputs of the neural network can be accordingly combined in a single real value.
The 1D similarity measure can then be used to compare any training samples of the same class.

When making statistics on similarity values $\E_{\x'}\big[ k^C_\theta(\x,\x') \big]$, another possible task binarization approach is to sample an adversary class $c_A$ along with $\x'$,  and hence consider $\nabla_{\!\theta} f^{c_R}_\theta(\x) - \nabla_{\!\theta} f^{c_A}_\theta(\x)$. Both approaches will lead to similar results in Section~\ref{sec:enforce}.

%% file: neighbors.tex
\section{Estimating density}
\label{sec:density}

In this section, we use similarity to estimate input neighborhoods and perform statistics on them.

\subsection{Estimating the number of neighbors}

Given a point $\x$, how many samples $\x'$ are similar to $\x$ according to the network? This can be measured by computing $k^C_\theta(\x,\x')$ for all $\x'$ and picking the closest ones, \ie \eg the $\x'$ such that $k^C_\theta(\x,\x') \geqslant 0.9$.
More generally, for any data point $\x$, the histogram of the similarity $k^C_\theta(\x,\x')$ over all $\x'$ in the dataset (or a representative subset thereof) can be drawn, and turned into an estimate of the number of neighbors of $\x$. To do this, several types of estimates are possible:
\begin{itemize} \setlength\itemsep{0em}  \setlength{\parskip}{1pt} 
\item hard-thresholding, for a given threshold $\tau \in [0,1]$: \hfill $N_\tau(\x) = \sum_{\x'} \mathbb{1}_{k^C_\theta(\x,\x') \geqslant \tau}$ \vspace{0.5mm}
\item soft estimate: \hfill $N_S(\x) \; = \; \sum_{\x'} k_\theta^C(\x,\x')$ \vspace{0.5mm}
\item less-soft positive-only estimate ($\alpha > 0$): \hfill $N^+_\alpha(\x) \; = \; \sum_{\x'} \mathbb{1}_{k_\theta^C(\x,\x') > 0}\; k_\theta^C(\x,\x')^\alpha$
\end{itemize}
In practice we observe that $k_\theta^C$ is very rarely negative, and thus the soft estimate $N_S$ can be justified as an average of the hard-thresholding estimate $N_\tau$ over all possible thresholds $\tau$:
$$\int_{\tau = 0}^1 \!\!\!N_\tau(\x) d\tau \;=\; \sum_{\x'} \int_{\tau = 0}^1 \!\!\mathbb{1}_{k^C_\theta(\x,\x') \geqslant \tau} \, d\tau \;=\; \sum_{\x'} k^C_\theta(\x,\x')\, \mathbb{1}_{k^C_\theta(\x,\x') \geqslant 0} \;=\; N^+_1(\x) \;\simeq\; N_S(\x)$$

\subsection{Low complexity of the soft estimate $N_S(\x)$}

\label{sec:NNcomplexity}

The soft estimate $N_S(\x)$ is rewritable as:
$$\sum_{\x'} k^C_\theta(\x,\x') =  \sum_{\x'} \frac{ \nabla_\theta f_\theta(\x) }{\| \nabla_\theta f_\theta(\x) \|} \cdot \frac{ \nabla_\theta f_\theta(\x') }{\|\nabla_\theta f_\theta(\x')\|} = \frac{ \nabla_\theta f_\theta(\x) }{\| \nabla_\theta f_\theta(\x) \|} \cdot \mathbf{g} \;\;\;\;\mathrm{with}\;\;\;\mathbf{g}=\sum_{\x'} \frac{ \nabla_\theta f_\theta(\x') }{\|\nabla_\theta f_\theta(\x')\|}$$
and consequently $N_S(\x)$ can be computed jointly for all $\x$ in linear time $O(|\mathcal{D}|p)$ in the dataset size $|\mathcal{D}|$ and in the number of parameters $p$, in just two passes over the dataset, when the output dimension is 1. For higher output dimensions $d$, a similar trick can be used and the complexity becomes $O(|\mathcal{D}|d^2p)$. For classification tasks with a large number $d$ of classes, the complexity can be reduced to $O(|\mathcal{D}|p)$ through an approximation consisting in binarizing the task (\cf end of Section~\ref{Lorisproj}).



\subsection{Test of the various estimators}

In order to rapidly test the behavior of all possible estimators, we applied them to a toy problem where the network's goal is to predict a sinusoid. To change the difficulty of the problem, we vary its frequency, while keeping the number of samples constant. Appendix~\ref{sec:estim2} gives more details and results for the toy problem. Fig.\ref{fig:toy_avg_all_measures} shows for each estimator (with different parameters when relevant), the result of their neighbor count estimation. When the frequency $f$ of the sinusoid to predict increases, the number of neighbors decreases in $\frac{1}{f}$ for every estimator. This aligns with our intuition that as the problem gets harder, the network needs to distinguish input samples more to achieve a good performance, thus the amount of neighbors is lower. In particular we observe that the proposed $N_S(\x)$ estimator behaves well, thus we will use that one in bigger studies requiring an efficient estimator.





\subsection{Further potential uses for fitness estimation}

When the number of neighbors of a training point $\x$ is very low, 
the network is able to set any label to $\x$, as this won't interfere with other points, by definition of our similarity criterion $k_\theta(\x,\x')$.
This is thus a typical overfit case, where the network can learn by heart a label associated to a particular, isolated point.

On the opposite, when the set of neighbors of $\x$ is a large fraction of the dataset, comprising varied elements, by definition of $k_\theta(\x,\x')$ the network is not able to distinguish them, and consequently it can only provide a common output for all of them. Therefore it might not be able to express variety enough, which would be a typical underfit case.


The quality of fit can thus be observed by monitoring the number of neighbors together with the variance of the desired labels in the neighborhoods (to distinguish underfit from just high density).


\begin{figure}
\begin{minipage}[c]{0.6\linewidth}
  \includegraphics[width=0.99\linewidth]{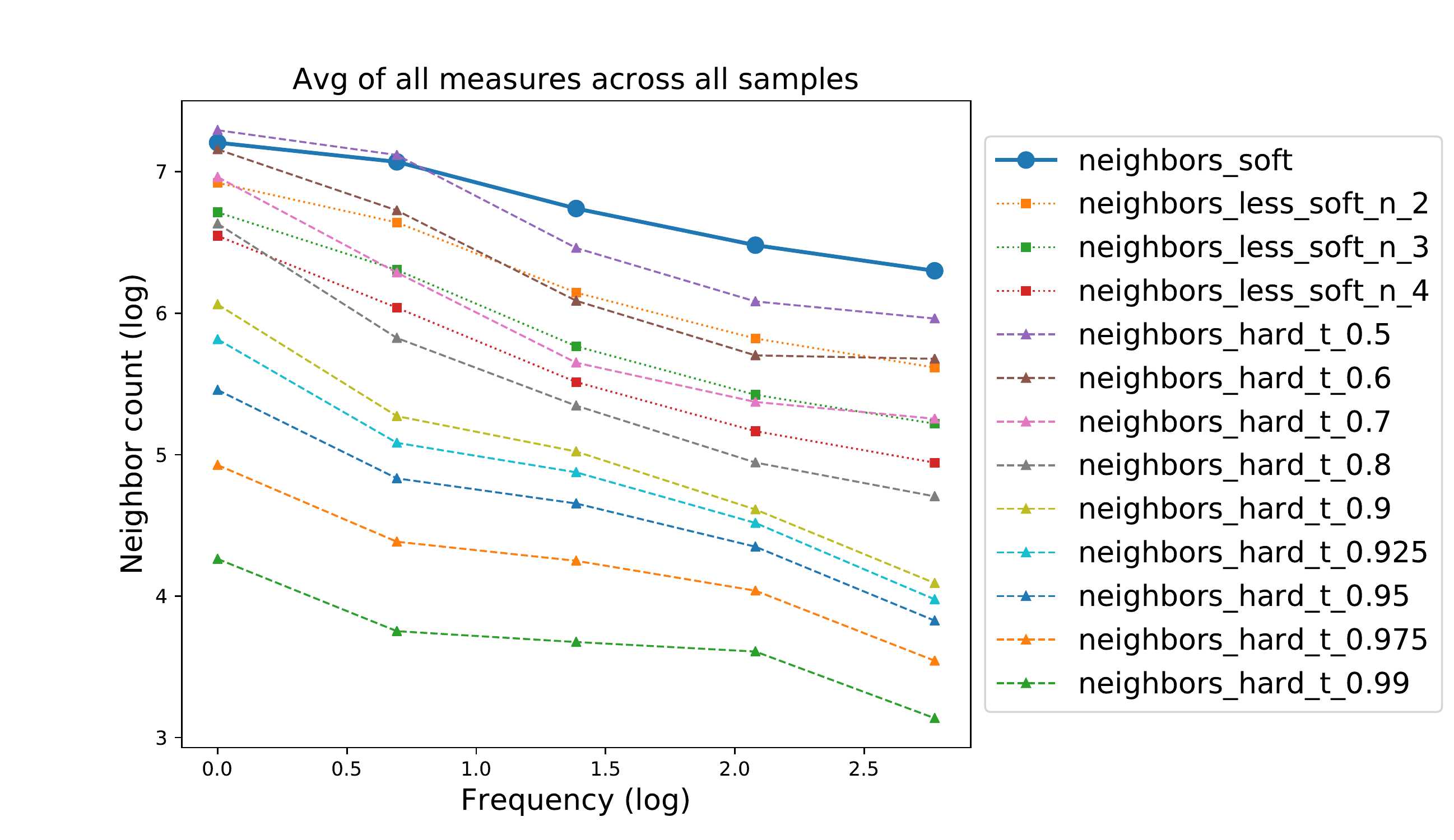}
  \caption{Density estimation using the various approaches (log scale).
    All approaches behave similarly and show good results, except the ones with extreme thresholds.}
  \label{fig:toy_avg_all_measures}
\end{minipage} \hfill
\begin{minipage}[c]{0.35\linewidth}
    \includegraphics[width=\linewidth]{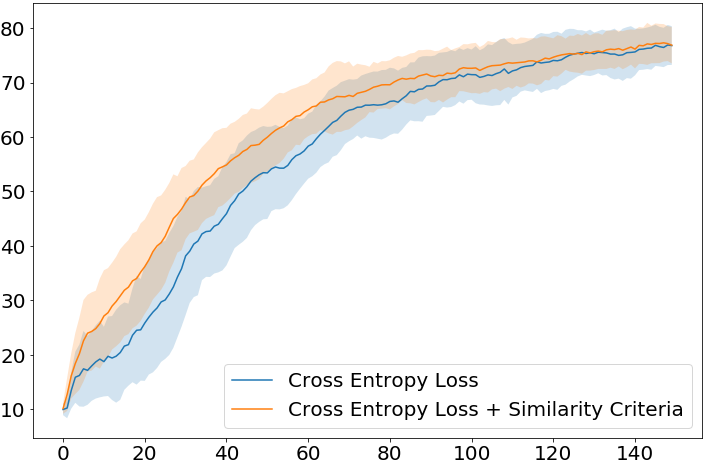}
    \caption{Validation accuracy of a neural network trained on MNIST with and without the similarity criterion (note that the x-axis is the number of minibatches presented to the network, not of epochs).}
    \label{fig:dyn}
\end{minipage}
\end{figure}


\paragraph{Prediction uncertainty} A measure of the uncertainty of a prediction $f_\theta(\x)$ could be to
check how easy it would have been to obtain another value during training, without disturbing the training of other points.
A given change $\vv$ of $f_\theta(\x)$ induces changes $\frac{k^I_\theta(\x,\x')}{\| \nabla_\theta f_\theta(\x) \|^2} \vv$ over other points $\x'$ of the dataset, creating a total $L^1$ disturbance $\sum_{\x'} \|\frac{k^I_\theta(\x,\x')}{\| \nabla_\theta f_\theta(\x) \|^2} \vv\|$.
The uncertainty factor would then be the norm of $\vv$ affordable within a disturbance level, and quickly approximable as $\frac{ \| \nabla_\theta f_\theta(\x) \|^2 }{ \sum_{\x'} k^I_\theta(\x,\x')}$.

%% file: enforce_sim.tex
\section{Enforcing similarity}
\label{sec:enforce}

The similarity criterion we defined could be used not only to estimate how similar two samples are perceived, after training, but also to incite the network, during training, to evolve in order to consider these samples as similar.

\paragraph{Asking two samples to be treated as similar}

If two inputs $\x$ and $\x'$ are known to be similar (from a human point of view), one can enforce their similarity from the network perspective, by adding to the loss the term:\vspace{-3mm}
$$- k^C_\theta(\x,\x') \; .$$

\paragraph{Asking a distribution of samples to be treated as similar}

By extension, to enforce the similarity of a subset $\Scal$ of training samples, of size $n = |\Scal|$, one might consider the average pairwise similarity $k_\theta^C$ over all pairs, or the standard deviation of the gradients. Both turn out to be equivalent to maximizing the
norm of the gradient mean $\mu = \frac{1}{n} \sum_{i\in\Scal} \frac{ \nabla_{\!\theta} f_\theta(\x_i) }{ \| \nabla_{\!\theta} f_\theta(\x_i) \|}$:
$$\frac{1}{n (n-1)} \sum_{i,j\in\Scal, i\neq j} \!\!\!k^C_\theta(x_i,x_j) \;=\;
\frac{n}{n-1} \|\mu\|^2 - \frac{1}{n-1} \;\;\;\;\;\;\mathrm{and}\;\;\;\;\;\; \var_{i\in\S}\, \frac{ \nabla_{\!\theta} f_\theta(\x_i) }{ \| \nabla_{\!\theta} f_\theta(\x_i) \|} = 1 - \|\mu\|^2\,.$$

In practice, common deep learning platforms are much faster when using mini-batches, but then return only the gradient sum $\sum_{i\in\mathcal{B}} \nabla_{\!\theta} f_\theta(\x_i)$ over a mini-batch $\mathcal{B}$, not individual gradients, preventing the normalization of each of them to compute $k_\theta^C$ or $\mu$. So instead we compare means of un-normalized gradients, over two mini-batches $\mathcal{B}_1$ and $\mathcal{B}_2$ comprising each $n_B$ samples from $\Scal$, which yields the criterion:
$$n_B \, \frac{ \| \mu_1 - \mu_2 \|^2 }{ \| \mu_1 \| \| \mu_2 \| } \;\;\;\;\;\;\;\;\mathrm{where}\;\;\;\;\mu_k = \frac{1}{n} \sum_{i\in\mathcal{B}_k} \nabla_{\!\theta} f_\theta(\x_i) \, .$$
The factor $n_B$ counterbalances the $\frac{1}{\sqrt{n_B}}$ variance reduction effect due to averaging over $n_B$ samples.

\paragraph{Group invariance} The distributions of samples asked to be seen as similar could be group orbits~\cite{cohen2016group}. A differential formulation of group invariance enforcement is also proposed in Appendix~\ref{sec:group2}.

\paragraph{Complexity}
The \emph{double-backpropagation} routine, available on common deep learning platforms, allows the optimization of
such criteria~\cite{drucker1991double,hochreiter1995simplifying,rifai2011contractive,gulrajani2017improved}, roughly doubling the computational time of a gradient step.

\paragraph{Dynamics of learning}
Our approach enforces similarity not just at the output level, but within the whole internal computational process. Therefore, during training, information is provided directly to each parameter instead of being back-propagated through possibly many layers. Thus the dynamics of learning are expected to be different, especially for deep networks.

To test this hypothesis, we train a small network on MNIST with and without the similarity criteria acting as an auxiliary loss (see Fig.~\ref{fig:dyn}). As a result, we observe an acceleration of the convergence very early in the learning process. It is worth noting that this effect can be observed across a wide range of different neural architectures. We performed additional experiments on toy datasets as well as on CIFAR10 with no or only negligible improvements. All together this suggests that using the similarity criteria during training may be beneficial to specific datasets as opposed to specific architectures, and indeed, as the class intra-variability in CIFAR10 is known to be high, considering all examples of a class of CIFAR10 as similar is less relevant.

%% file: back_to_self_denoising.tex
\section{Dataset self-denoising}
\label{backtodenoise}

We now go back to the task described in Section~\ref{sec:denoising} and show how input similarity can be used to analyse experimental results and bring theoretical guarantees about robustness to label noise.

\subsection{Similarity experimentally observed between patches}
We studied the multi-round training scheme of \cite{anonymous} by applying our similarity measure to a sampling of input patches of the training dataset for one network per round. The principle of the multiple round training scheme is to reduce the noise of the annotations, obtaining aligned annotations in the end (more details in Appendix~\ref{sec:denoise2}). For a certain input patch, we computed its similarity with all the other patches for the 3 networks. With those similarities we can compute the nearest neighbors of that patch, see Fig. \ref{fig:k_nearest}. The input patch is of a suburb area with sparse houses and individual trees. The closest neighbors look similar as they usually feature the same types of buildings, building arrangement and vegetation. However sometimes the network sees a patch as similar when it is not clear from our point of view (for example patches with large buildings).

For more in-depth results, we computed the histogram of similarities for the same patch, see Fig.~\ref{fig:bloomington22_individual_hist_02}.
We observe that round 2 shows different neighborhood statistics, in that the patch is closer to all other patches than in other rounds. We observe the same behavior in 19 other input patches (see Appendix~\ref{sec:denoise2}). An hypothesis for this phenomenon is that the average gradient was not 0 at the end of that training round (due to optimization convergence issues, e.g.), which would shift all similarity histograms by a same value.

Qualitatively, for patches randomly sampled, their similarity histograms tend to be approximately symmetric in round 2, but with a longer left tail in round 1 and a longer right tail in round 3.
Neighborhoods thus seem to change across the rounds, with fewer and fewer close points (if removing the global histogram shift in round 2).
A possible interpretation is that this would reflect an increasing ability of the network to distinguish between different patches, with finer features  in later training rounds.

\begin{figure}
  \center
  \includegraphics[width=\linewidth]{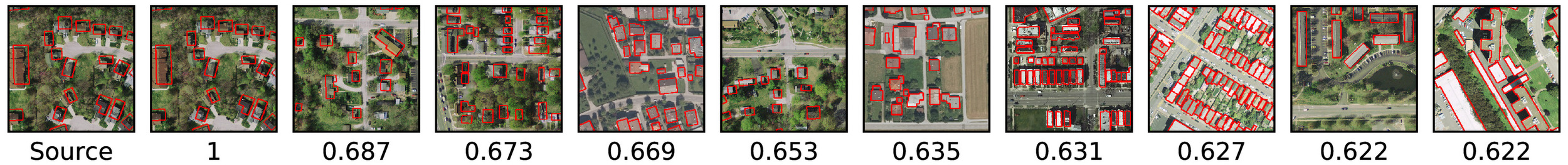}
  \includegraphics[width=\linewidth]{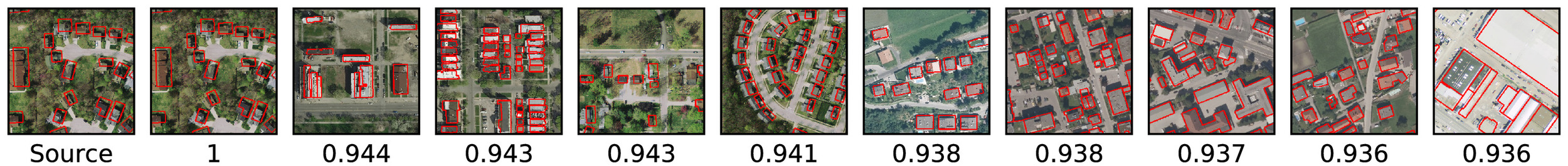}
  \includegraphics[width=\linewidth]{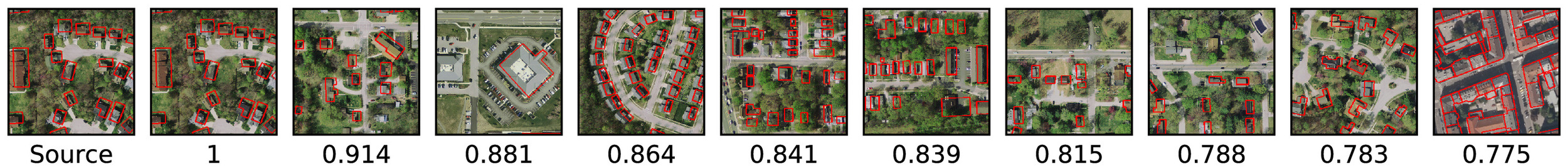}
  \caption{Example of nearest neighbors for a patch. Each line corresponds to a round. Each patch has its similarity written under it.}
  \label{fig:k_nearest}
\end{figure}

\begin{figure}
	\centering
	\begin{subfigure}[b]{0.3\textwidth}
		\centering
		\caption{Round 1}		
		\includegraphics[width=\linewidth]{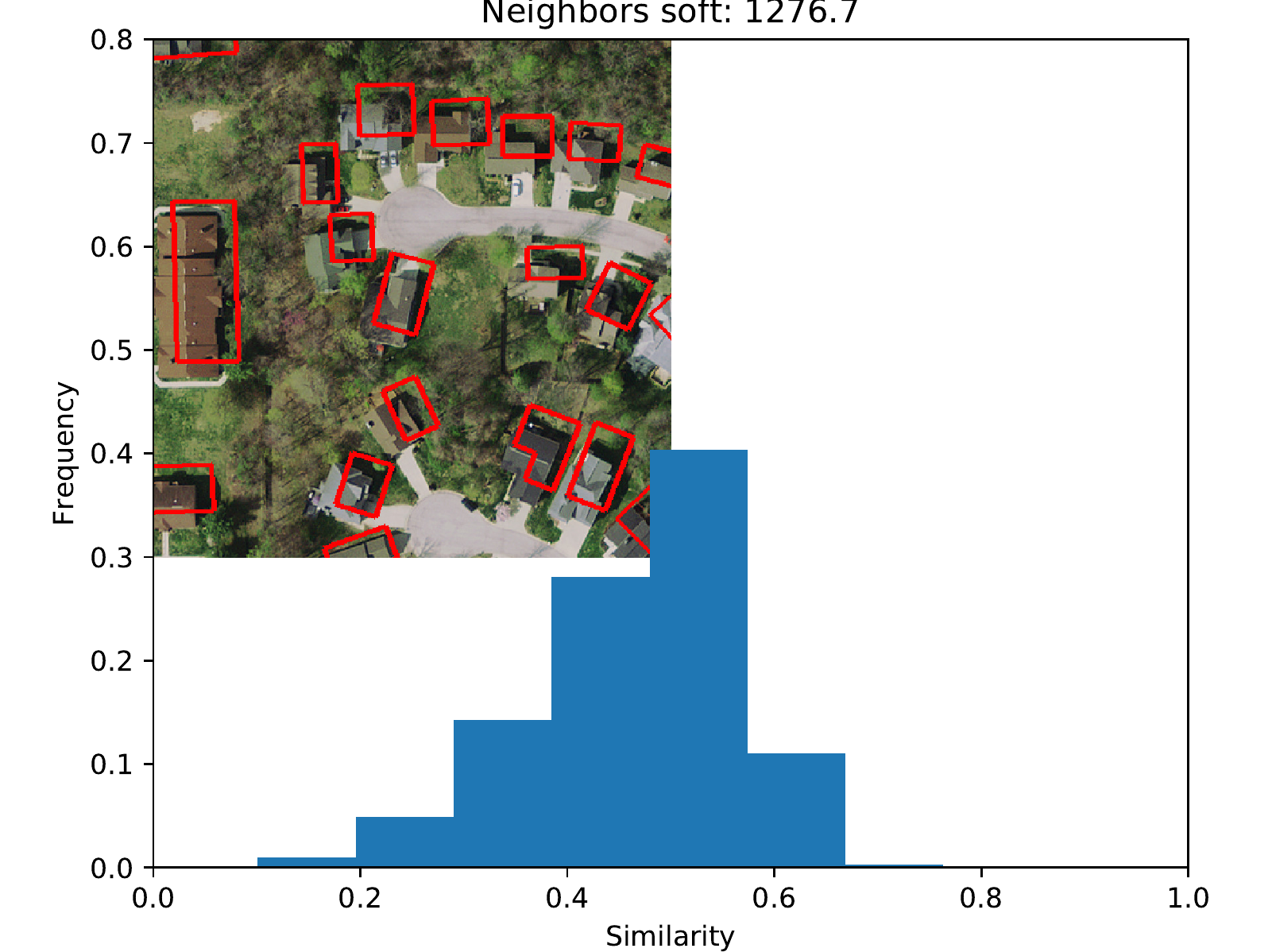}
	\end{subfigure}
	\begin{subfigure}[b]{0.3\textwidth}
		\centering
		\caption{Round 2}		
		\includegraphics[width=\linewidth]{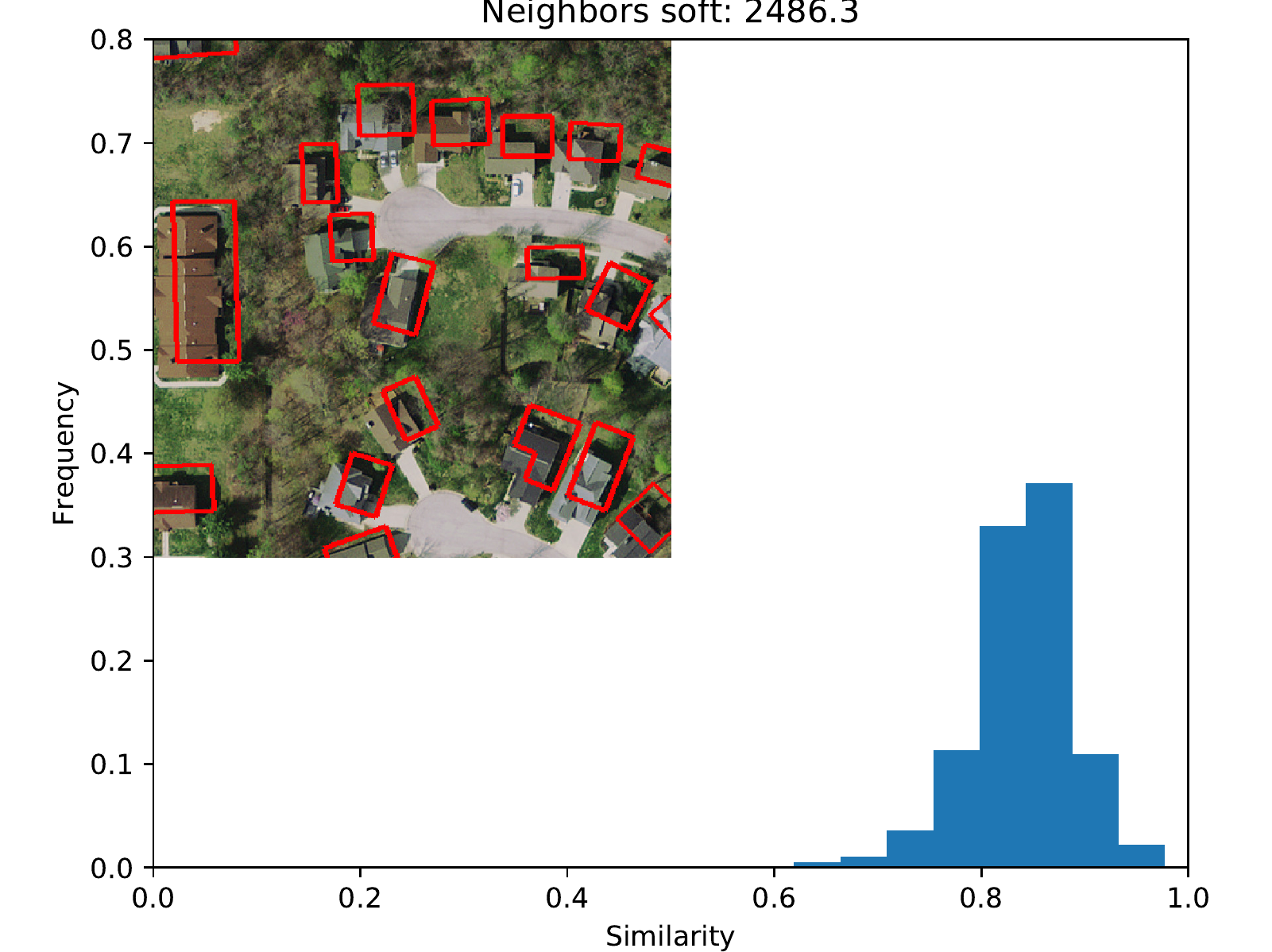}
	\end{subfigure}
	\begin{subfigure}[b]{0.3\textwidth}
		\centering
		\caption{Round 3}		
		\includegraphics[width=\linewidth]{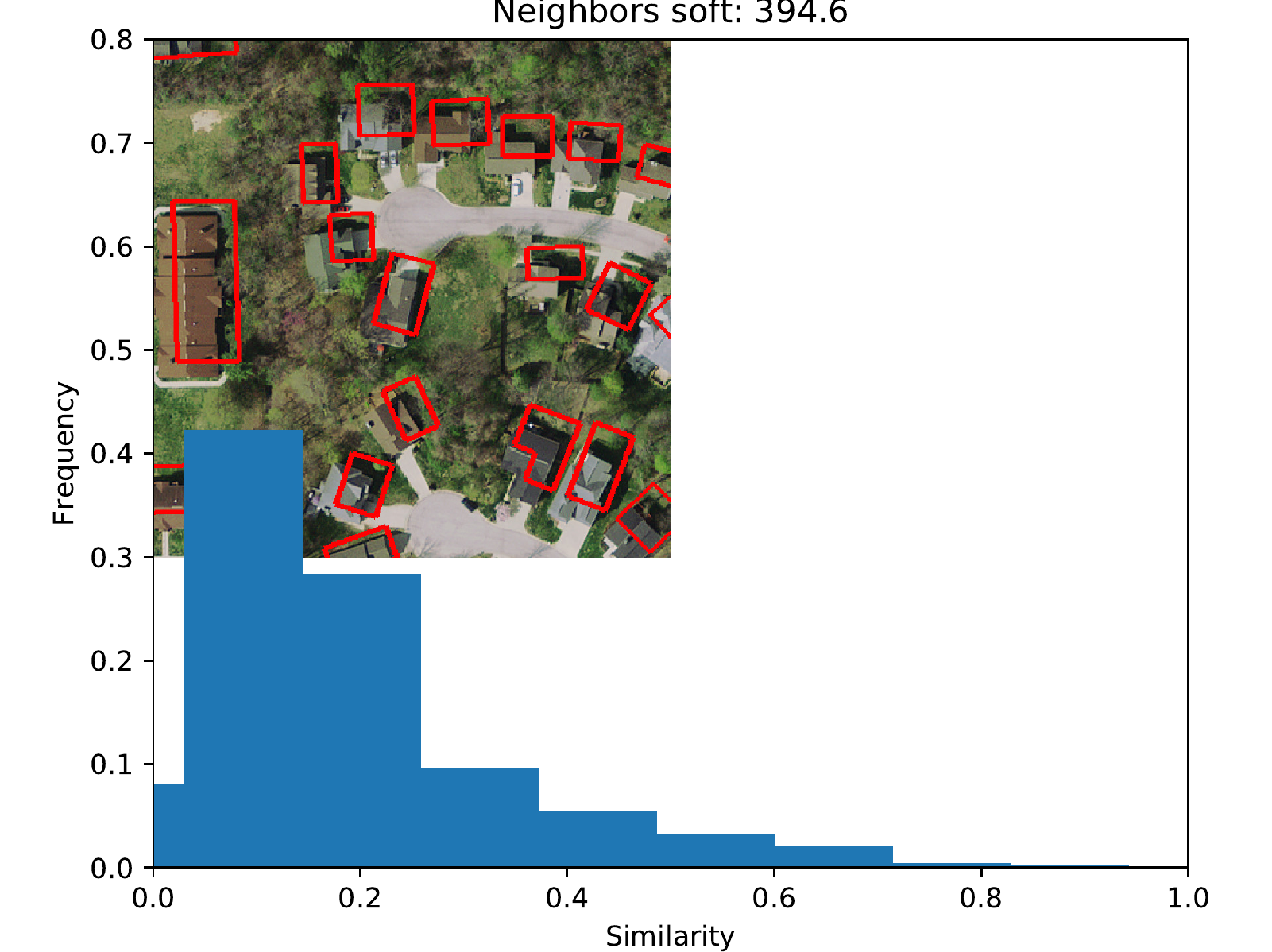}
	\end{subfigure}
	\caption{Histograms of similarities for one patch across rounds.}
	\label{fig:bloomington22_individual_hist_02}
\end{figure}


\subsection{Comparison to the \emph{perceptual loss}}

We compare our approach to the \emph{perceptual loss} on a nearest neighbor retrieval task. We notice that the \emph{perceptual loss} sometimes performs reasonably well, but often not. For instance, we show in Fig.~\ref{fig:comparison_percept} the closest neighbors to a structured residential area image, for the \emph{perceptual loss}
(first row: not making sense) and for our similarity measure (second row: similar areas).

\begin{figure}
	\rotatebox[origin=l]{90}{$\!$Perceptual}\includegraphics[width=0.98\linewidth]{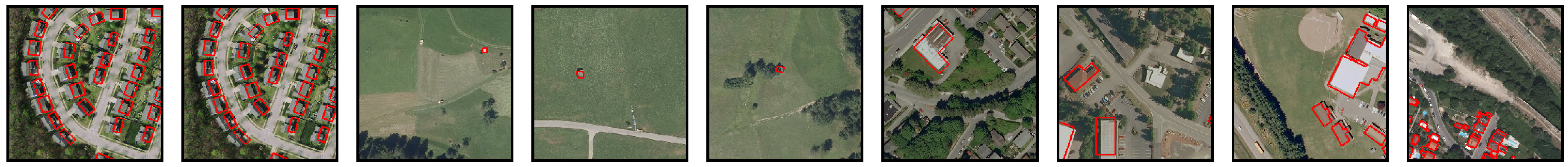}
	\rotatebox[origin=l]{90}{$\!\!$Similarity}\includegraphics[width=0.98\linewidth]{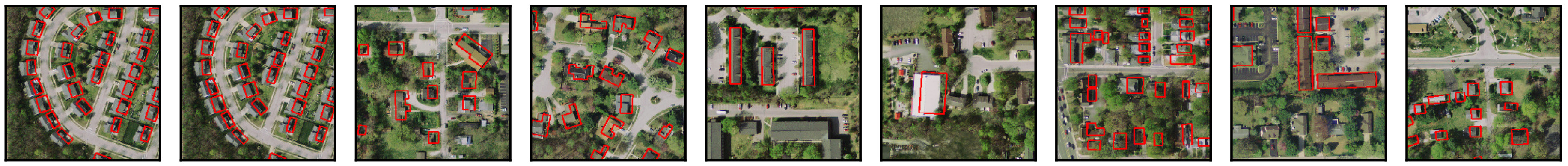}
	\hspace*{5mm}Source\hspace*{2.2mm} | \hspace*{3.2mm}Closest neighbor patches
	\caption{Closest neighbors to the leftmost patch, using the \emph{perceptual loss} (first row) and our similarity definition (second row).}
	\label{fig:comparison_percept}
\end{figure}



\subsection{From similarity statistics to self-denoising effect estimation}

We now show how such similarity experimental computations can be used to solve the initial problem of Section~\ref{sec:denoising}, by explicitly turning similarity statistics into a quantification of the self-denoising effect. 

Let us denote by $y_i$ the true (unknown) label for input $\x_i$, by $\yt_i$ the noisy label given in the dataset, and by $\yh_i = f_\theta(\x_i)$ the label predicted by the network. 
We will denote the (unknown) noise by  $\epsi_i = \yt_i - y_i$ and assume it is centered and i.i.d., with finite variance $\sigma_\epsi$.
The training criterion is $ E(\theta) = \sum_j || \yh_j - \yt_j ||^2 $.
At convergence, the training leads to a local optimum of the energy landscape:
$ \nabla_{\!\theta} E = 0 $, that is, $ \sum_j (\yh_j - \yt_j) \nabla_{\!\theta} \yh_j = 0 $.
Let's choose any sample $i$ and multiply by $\nabla_{\!\theta} \yh_i$ : using $\;k^I_\theta(\x_i,\x_j) = \nabla_{\!\theta} \yh_i . \nabla_{\!\theta} \yh_j\,$, we get:
$$\;\; \sum_j (\yh_j - \yt_j) \, k^I_\theta(\x_j, \x_i) = 0.$$ 
Let us denote by $ k^{IN}_\theta(\x_j,\x_i) =  k^{I}_\theta(\x_j,\x_i) \big(\sum_j k^I_\theta(\x_j,\x_i)\big)^{-1} $ the column-normalized kernel, and by  $ \E_k [ a ] =\, \sum_j\, a_j\, k^{IN}_\theta(\x_j,\x_i)$ the mean value of $a$ in the neighborhood of $i$, that is, the weighted average of the $a_j$ with weights $k^I_\theta(\x_j,\x_i)$ normalized to sum up to 1. This is actually a kernel regression, in the spirit of Parzen-Rosenblatt window estimators. Then the previous property can be rewritten as $ \,\E_k[ \yh ] = \E_k[ \yt ]\, $.
As   $\, \E_k[ \yt ] = \E_k[ y ] + \E_k[ \epsi ] \,$, this yields:
$$\;\;\; \yh_i - \E_k[ y ]  =  \E_k[ \epsi ]  +  ( \yh_i - \E_k[ \yh ] ) $$
\ie the difference between the predicted $\yh_i$ and the average of the true labels in the neighborhood of $i$ is equal to the average of the noise in the neighborhood of $i$, up to the deviation of the prediction $\yh_i$ from the average prediction in its neighborhood.

We want to bound the error $\| \yh_i - \E_k[ y ]\|$ without knowing neither the true labels $y$ nor the noise~$\epsi$. One can show that $\E_k[ \epsi ] \propto \var_\epsi(\E_k[ \epsi ])^{1/2} = \sigma_\epsi \, \| k^{IN}_\theta (\cdot,\x_i) \|_{L2}$. The denoising factor is thus the similarity kernel norm $\| k^{IN}_\theta (\cdot,\x_i)\|_{L2}$, 
which is between $1/\sqrt{N}$ and 1, depending on the neighborhood quality. It is $1/\sqrt{N}$ when all $N$ data points are identical, i.e. all satisfying $k^C_\theta(\x_i,\x_j) = 1$. On the other extreme, this factor is 1 when all points are independent: $k^I_\theta(\x_i,\x_j) =~0 \;\;$ $\forall i \neq j$. 
This way we extend \emph{noise2noise} \cite{noise2noise}
to real datasets with
non-identical
inputs.

In our remote sensing experiment, we estimate this way a denoising factor of 0.02, consistent across all training rounds and inputs ($\pm 10\%$), implying that each training round contributed equally to denoising the labels. This is confirmed by Fig.~\ref{fig:accuracies}, which shows the error steadily decreasing, on a control test where true labels are known.
The shift $( \yh_i - \E_k[ \yh ] )$ on the other hand can be directly estimated given the network prediction. In our case, it is $4.4$px on average, which is close to the observed median error for the last round in Fig.~\ref{fig:accuracies}.
It is largely
input-dependent,
with variance $3.2$px, which is reflected by the spread distribution of errors in Fig.~\ref{fig:accuracies}. This input-dependent shift thus provides a hint about prediction reliability.


It is also possible to bound $( \yh_i - \E_k[ \yh ] ) = \E_k[ \yh_i - \yh ]$ using only similarity information (without predictions $\yh$). Theorem \ref{basicnet} implies that the application: $\frac{\nabla_{\!\theta} f_\theta(\x)}{\|\nabla_{\!\theta} f_\theta(\x)\|} \mapsto f_\theta(\x)$ is well-defined, and it can actually be shown to be 
Lipschitz with a network-dependent constant (under mild hypotheses). Thus
$$\| f_\theta(\x) -  f_\theta(\x')\| \leqslant C \left\| \frac{\nabla_{\!\theta} f_\theta(\x)}{\|\nabla_{\!\theta} f_\theta(\x)\|} - \frac{\nabla_{\!\theta} f_\theta(\x')}{\|\nabla_{\!\theta} f_\theta(\x')\|}  \right\| = \sqrt{2} C \sqrt{1 - k_\theta^C(\x,\x')}\;,$$
yielding $\|  \yh_i - \yh_j \|  \leqslant \sqrt{2} C \sqrt{1 - k_\theta^C(\x_i,\x_j)} $ and thus $\big|\E_k[ \yh_i - \yh ]\,\big| \leqslant \sqrt{2} C \E_k\!\left[ \sqrt{1 - k_\theta^C(\x_i,\cdot)}\,\right]$.

%% file: conclusion.tex
\section{Conclusion}

We defined a proper notion of input similarity as perceived by the neural network, based on the ability of the network to distinguish the inputs.
This brings a new tool to analyze trained networks, in plus of visualization tools such as grad-CAM \cite{gradcam}.
We showed how to turn it into a density estimator, which was validated on a controlled experiment, and usable to perform fast statistics on large datasets. It opens the door to underfit/overfit/uncertainty analyses or even control during training, as it is differentiable and computable at low cost.
We also showed that any desired similarity could be enforced during training, at reasonable cost, and noticed a dataset-dependent boosting effect that should be further studied along with
robustness to adversarial attacks,
as such training differs significantly from usual methods.
Finally, we extended \emph{noise2noise} \cite{noise2noise} to the case of non-identical inputs, thus expressing self-denoising effects as a function of inputs' similarities.
The code is available at
\url{https://github.com/Lydorn/netsimilarity}~.

\section*{Acknowledgments}

We thank Victor Berger and Adrien Bousseau for useful discussions.
This work benefited from the support of the project EPITOME ANR-17-CE23-0009 of the French National Research Agency~(ANR).

%% file: annexe.tex
\appendix

\section{Code}

The whole code (image registration, experiments to test density estimators, enforcing similarity...)  is available on the following github repository: \url{https://github.com/Lydorn/netsimilarity} .

\section{Proofs of the properties of the 1D similarity kernel}

We give here the proofs at the properties of the 1-dimensional-output similarity kernel.

\subsection{Proof of Theorem \ref{basicnet}}

\begin{theorem} 
  For any real-valued neural network $f_\theta$
  whose last layer is a linear layer (without any parameter sharing) or a standard activation function thereof (sigmoid, tanh, ReLU...),
  and for any inputs $\x$ and $\x'$,
  $$\nabla_{\!\theta} f_\theta(\x) = \nabla_{\!\theta} f_\theta(\x') \;\;\implies\;\;f_\theta(\x) = f_\theta(\x') \,.$$
\end{theorem}
\begin{proof}
If the last layer is linear, the output is of the form $f_\theta(\x) = \sum_i w_i a_i(\x) + b$, where $w_i$ and $b$ are parameters in $\R$ and $a_i(\x)$ activities from previous layers.
The gradient $\nabla_{\!\theta} f_\theta(\x)$ contains in particular as coefficients the derivatives $\frac{d f_\theta(\x) }{dw_i} = a_i(\x)$.
Thus $\nabla_{\!\theta} f_\theta(\x) = \nabla_{\!\theta} f_\theta(\x') \implies a_i(\x) = a_i(\x') \;\forall i$ in the last layer.
The outputs can be then rebuilt: $f_\theta(\x) = \sum_i w_i a_i(\x) + b = \sum_i w_i a_i(\x') + b = f_\theta(\x')$.

If the output is of the form $f_\theta(\x) = \sigma(c(\x))$ with $c(\x) = \sum_i w_i a_i(\x) + b$, then the gradient equality implies $\frac{d f_\theta(\x) }{db} = \frac{d f_\theta(\x') }{db}$, whose value is $\sigma'(c(\x)) = \sigma'(c(\x'))$. Then, as $\sigma'(c(\x))\, a_i(\x) = \frac{d f_\theta(\x) }{dw_i} = \frac{d f_\theta(\x') }{dw_i} = \sigma'(c(\x'))\, a_i(\x')$, we can deduce $a_i(\x) = a_i(\x')$ for all $i$ provided $\sigma'(c(\x)) \neq 0$. In that case, from these identical activities one can rebuild identical outputs. Otherwise, $\sigma'(c(\x)) = \sigma'(c(\x')) = 0 $, which is not possible with strictly monotonous activation functions, such as tanh or sigmoid. For ReLU, $\sigma'(c(\x)) = 0 \implies \sigma(c(\x)) = 0$ and thus $f_\theta(\x) = f_\theta(\x') = 0$.
The same reasoning holds for other activation functions with only one flat piece (such as the ReLU negative part), \ie for which the set $\sigma(\sigma'^{-1}(\{0\}))$ is a singleton.
\end{proof}

\subsection{Proof of Corollary \ref{alphasim}}
\begin{corollary} 
Under the same assumptions, for any inputs $\x$ and $\x'$,
$$\begin{array}{crcl}
& k^C_\theta(\x.\x') = 1  & \implies & \nabla_{\!\theta} f_\theta(\x) = \nabla_{\!\theta} f_\theta(\x') \,, \vspace{1mm} \\
\mathrm{hence} & k^C_\theta(\x.\x') = 1  & \implies & f_\theta(\x) = f_\theta(\x') \,. \\
\end{array}$$
\end{corollary}

\begin{proof}$\;\;$
$k^C_\theta(\x.\x') = 1$ means
$\frac{ \nabla_{\!\theta} f_\theta(\x)  }{ \| \nabla_{\!\theta} f_\theta(\x) \|}  \cdot \frac{ \nabla_{\!\theta} f_\theta(\x')   }{ \| \nabla_{\!\theta} f_\theta(\x') \|} = 1$, which implies
$\exists\, \alpha \in \R^*, \; \nabla_{\!\theta} f_\theta(\x)  = \alpha\,  \nabla_{\!\theta} f_\theta(\x')$. We need to show that $\alpha = 1$. Under the assumptions of Theorem \ref{basicnet}, following its proof:
\begin{itemize}
\item either the last layer is linear, the output is of the form $f_\theta(\x) = \sum_i w_i a_i(\x) + b$, and then $\nabla_{\!b} f_\theta(\x)  = \alpha  \nabla_{\!b} f_\theta(\x')$ while $\frac{d f_\theta(\x) }{db} = 1$ and $\frac{d f_\theta(\x') }{db} =1$, hence $\alpha = 1$;
\item either the output is of the form $f_\theta(\x) = \sigma(c(\x))$ with $c(\x) = \sum_i w_i a_i(\x) + b$, and then $\sigma'(c(\x)) = \nabla_{\!b} f_\theta(\x)  = \alpha  \nabla_{\!b} f_\theta(\x') = \alpha \,\sigma'(c(\x'))$, while, for any $i$,
$\sigma'(c(\x))\, a_i(\x) = \frac{d f_\theta(\x) }{dw_i} = \alpha \frac{d f_\theta(\x') }{dw_i} = \alpha \,\sigma'(c(\x'))\, a_i(\x')$. Thus, supposing $\sigma'(c(\x)) \neq 0$, we obtain $a_i(\x) = a_i(\x')\; \forall i$, and thus we can rebuild from the activities $c(\x) = c(\x')$, from which $\sigma'(c(\x)) = \sigma'(c(\x'))$ and thus $\alpha = 1$. Otherwise, $\sigma'(c(\x)) = \sigma'(c(\x')) = 0$ and the two full gradients $\nabla_{\!\theta} f_\theta(\x)$ and $\nabla_{\!\theta} f_\theta(\x')$ are 0 and thus equal.
\end{itemize}
The conditions for $\;k^C_\theta(\x.\x') = 1  \implies \nabla_{\!\theta} f_\theta(\x) = \nabla_{\!\theta} f_\theta(\x')\;$ to hold are actually much weaker: it is sufficient that in the whole network architecture there exists \emph{one} useful neuron (in the sense of the next paragraph) of that type (so called \emph{linear} but actually affine).

\end{proof}

\subsection{Proof of Theorem \ref{basicnet2}}

\begin{theorem} 
  For any real-valued neural network $f_\theta$ without parameter sharing,
  if $\nabla_{\!\theta} f_\theta(\x) = \nabla_{\!\theta} f_\theta(\x')$ for two inputs $\x, \x'$,
  then all useful activities computed when processing $\x$ are equal to the ones obtained when processing $\x'$.
\end{theorem}

We name \emph{useful} activities all activities whose variation would have an impact on the output, \ie all the ones satisfying $\frac{d f_\theta(\x) }{da_i} \neq 0$. This condition is typically not satisfied when the activity is multiplied by 0, \ie $w_i = 0$, or when it is negative and followed by a ReLU, or when all its contributions to the output annihilate together (\eg, a sum of two neurons with opposite weights: $f_\theta(\x) = \sigma( a_i(\x) ) - \sigma( a_i(\x) )$).

\begin{proof}
Let $a_i(\x)$ be a useful activity (for $\x$).
It is fed to at least one useful neuron, whose pre-activation output is of the form $c(\x) = \sum_i w_i a_i(\x) + b$. Then $\frac{d f_\theta(\x) }{db} = \frac{d f_\theta(\x) }{dc} \neq 0$ (the output of the neuron is useful), and
$\frac{d f_\theta(\x) }{dw_i} = \frac{d f_\theta(\x) }{db} a_i(\x)$.
From the gradient equality, $a_i(\x) = \frac{d f_\theta(\x) }{dw_i} / \frac{d f_\theta(\x) }{db} = \frac{d f_\theta(\x') }{dw_i} / \frac{d f_\theta(\x') }{db}  = a_i(\x')$.
\end{proof}

\section{Higher output dimension}
\label{sec:high2}

We expand here all the mathematical aspects of the homonymous section of the article.

\subsection{Derivation}

Let us now study the case where $f_\theta(\x)$ is a vector in $\R^d$ with $d > 1$.

The optimal parameter change $\delta \theta$ to push $f_\theta(\x)$ in a direction $\vv$ (with a force $\epsi$) is less straightforward to obtain.
First, one can define as many gradients as output coordinates: $\nabla_{\!\theta} f_\theta^i(\x)$, for $i \in \llbracket 1, d \rrbracket$.

This family of gradients can be shown to be linearly independent, unless the architecture of the network is specifically built not to. If for instance each output coordinate has its own bias parameter, \ie writes in the form $f_\theta^i(\x) = b_i + g_\theta(\x)$ or $\sigma( b_i + g_\theta(\x) )$ with a strictly monotonous activation function $\sigma$, then the derivative \wrt $b_i$ will be 1 (or $\sigma'$) only in the $i$-th gradient and 0 in the other ones. Thus the $j$-th gradient contains in particular the subvector $(\frac{df^j}{db_i})_i = (\delta_{i=j})_i$, and the gradients are consequently independent. In the case where all coordinates depend on all biases, but not identically, as with a softmax, the argument stays true.

Any parameter variation $\delta \theta \in \R^p$ can then be uniquely decomposed as:
$$\delta\theta = \sum_{i=1}^d \alpha_i \nabla f_\theta^i(\x) \;+\;\gamma $$
where $\alpha_i \in \R$ and where $\gamma \in \R^p$ is orthogonal to all coordinate gradients. This parameter variation induces an output variation:
$$ f_{\theta + \delta \theta} (\x) - f_\theta(\x) = \nabla_{\!\theta} f_\theta(\x) \;  \delta \theta + O(\|\delta\theta\|^2)$$
$$= \left( \sum_i \alpha_i \nabla_{\!\theta} f^i_\theta(\x) \cdot \nabla f_\theta^j(\x)\right)_j + 0 + O(\|\delta\theta\|^2)$$
$$= C \alpha + O(\|\alpha\|^2)$$
where $C$ is the correlation matrix of the gradients: $C_{ij} = \nabla_{\!\theta} f^i_\theta(\x) \cdot \nabla f_\theta^j(\x)$.
It turns out that $C$ is invertible:
$$C \alpha = 0 \implies \alpha C \alpha = 0 \implies \alpha \nabla_{\!\theta} f_\theta(\x)\, \nabla_{\!\theta} f_\theta(\x) \,\alpha = 0$$
$$\implies \|\nabla_{\!\theta} f_\theta(\x)\, \alpha\|^2 = 0 \implies \sum_i \alpha_i \nabla f_\theta^i(\x) = 0$$
$\implies \alpha = 0$
as the $\nabla_{\!\theta} f^i_\theta(\x)$ are linearly independent. Thus, for a desired output move in the direction $\vv$ with amplitude $\epsi$, \ie $f_{\theta + \delta \theta} (\x) - f_\theta(\x) = \epsi \vv$, one can compute the associated linear combination $\alpha = \epsi\, C^{-1}\vv$ and thus the smallest associated parameter change $\delta\theta = \sum_i \alpha_i \nabla f_\theta^i(\x)$.

The output variation induced at any other point $\x'$ by this parameter change is then:
$$f_{\theta + \delta \theta} (\x') - f_\theta(\x') = \left( \nabla_{\!\theta} f^i_\theta(\x') \cdot \delta \theta \right)_i + O(\|\delta\theta\|^2)$$
$$ = \left( \sum_j \alpha_j \nabla_{\!\theta} f^i_\theta(\x') \cdot \nabla_{\!\theta} f^j_\theta(\x) \right)_i + O(\|\delta\theta\|^2).$$
\begin{equation}
  = \epsi \, K_\theta(\x',\x)\,  C_\theta(\x)^{-1}\, \vv \, +\, O(\epsi^2)
\end{equation}
where the $d \times d$ kernel matrix $K_\theta(\x,\x')$ is defined by $K^{ij}_\theta(\x,\x') =  \nabla_{\!\theta} f^i_\theta(\x) \cdot \nabla_{\!\theta} f^j_\theta(\x')$, and where the matrix $C_\theta(\x) = K_\theta(\x,\x)$ is the previously defined self-correlation  matrix $C$. Its role is equivalent of the normalization by $\|\nabla_{\!\theta} f_\theta(\x)\|^2$ in the 1D case, in plus of decorrelating the gradients.

The interpretation of (\ref{eq:multidim}) is that if one moves the output for point $\x$ by $\vv$, then the output for point $\x'$ will be moved also, by $M \vv$, with $M = K_\theta(\x,\x')\,  K_\theta(\x,\x)^{-1}$.
Note that these matrices $M$ or $K$ are only $d \times d$ where $d$ is the output dimension. They are thus generally small and easy to manipulate or inverse.

\subsection{Normalized cross-correlation matrix}

The normalized version of the kernel (\ref{eq:multidim}) is:
\begin{equation}
  K_\theta^C(\x,\x') \;=\;  C_\theta(\x)^{-1/2}\; K_\theta(\x,\x')\;  C_\theta(\x')^{-1/2}
\end{equation}
which is symmetric in the sense that $K_\theta^C(\x',\x) = K_\theta^C(\x,\x')^T$.

A matrix $K_\theta^C(\x,\x')$ with small coefficients means that $\x$ and $\x'$ are relatively independent, from a neural network point of view (moves at $\x$ won't be transferred to $\x'$). On the opposite,
the highest possible dependency is $K_\theta^C(\x,\x) = \Id$.

To study properties of this similarity measure, note that $K_\theta^C(\x,\x') = (G^N_\x)^T\, G^N_{\x'}$ with $G^N_\x = G_\x (G_\x^T G_\x)^{-1/2}$, where $G_\x = \nabla_{\!\theta} f(\x)$ : it is the product of normalized, decorrelated versions of the gradient. Indeed, at any point $\x$, the normalized gradient matrix $G^N_\x$ satisfies:
$(G^N_\x)^T\, G^N_{\x} = K_\theta^C(\x,\x) =  K_\theta(\x,\x)^{-1/2} K_\theta(\x,\x) K_\theta(\x,\x)^{-1/2}  = \Id$ and consequently $G^N_\x$ can be seen as an orthonormal family of vectors $G^{N,i}_\x$.

The $L^2$ (Frobenius) norm of the ortho-normalized gradient $G^N_\x$ is thus:
$$\big\|G^N_\x\big\|^2_F = \Tr((G^N_\x)^T\, G^N_{\x}) = \Tr(\Id) = d \;.$$

At point $\x'$, $G^N_{\x'}$ is also an orthonormal family, but possibly arranged differently or generating a different subspace of $\R^p$. If $G^N_{\x}$ and $G^N_{\x'}$ generate the same subspace, then their product $(G^N_\x)^T\, G^N_{\x'}$ is an orthogonal matrix $Q$ (change of basis) and its $L^2$ (Frobenius) norm is then $\big\|Q\big\|^2_F = \Tr(Q^T Q) = \Tr(\Id) = d$. Otherwise, $(G^N_\x)^T\, G^N_{\x'}$ can be seen as a projection from one subspace to another one, each vector $G^{N,j}_{\x'}$ is projected onto the ortho-normal family $(G^{N,i}_\x)_i$, and as a projection decreases the Euclidean norm, $\sum_i \left( G^{N,i}_\x \cdot G^{N,j}_{\x'} \right)^2 \leqslant \big\|G^{N,j}_{\x'}\big\|^2 = 1$. Thus:
$$\big\|K_\theta^C(\x,\x')\big\|_F = \sqrt{\sum_{ij} \left( G^{N,i}_\x \cdot G^{N,j}_{\x'} \right)^2} \leqslant \sqrt{d} \, .$$
Moreover, any coefficient of the kernel matrix satisfies:
$$\left| K_\theta^{C, ij}(\x,\x') \right| = \left| G^{N,i}_\x \cdot G^{N,j}_{\x'} \right| \leqslant \big\|G^{N,i}_\x\big\|_2 \, \big\|G^{N,j}_{\x'}\big\|_2 = 1$$
as each vector $G^{N,i}_\x$ is unit-norm.
This implies in particular that the trace is bounded:
$$-d \;\leqslant\; \Tr(K_\theta^C(\x,\x')) \;\leqslant d.$$

To sum up, the similarity matrix $K_\theta^C(\x,\x')$ satisfies the following properties:
\begin{itemize}  \setlength\itemsep{0em}  \setlength{\parskip}{1pt} 
\item its coefficients are bounded, in $[-1,1]$
\item its trace is at most $d$
\item its (Frobenius) norm is at most $\sqrt{d}$
\item self-similarity is identity: $\forall \x, \,\;K_\theta^C(\x,\x) = \Id$
\item the kernel is symmetric, in the sense that  $K_\theta^C(\x',\x) = K_\theta^C(\x,\x')^T$.
\end{itemize}

\subsection{Similarity in a single value}


Note that when the trace is close to its maximal value $d$, the diagonal coefficients are close to 1, and their contribution to the Frobenius norm squared is close to $d$. Therefore, all non-diagonal coefficients are close to 0, and the matrix is close to $\Id$. And reciprocally, a matrix close to $\Id$ has a trace close to $d$.
Thus, two related ways to quantify similarity in a single real value in $[-1,1]$ appear:
\begin{itemize}  \setlength\itemsep{1pt}  \setlength{\parskip}{1pt} 
\item the distance to the identity  $D = \big\| K_\theta^C(\x,\x') - \Id\big\|_F$, which can be turned into a similarity as $1 - \frac{1}{\sqrt{d}} D$ or $1 - \frac{1}{2d} D^2$, since $D \in [0, 2\sqrt{d}]$
\item the normalized trace: $\frac{1}{d} \,\Tr\, K^C_\theta(\x,\x')$, which is also the alignment with the identity: $\frac{1}{d} K_\theta^C(\x,\x') \cdot_F \Id$, where $\cdot_F$ denotes the Frobenius inner product (\ie coefficient by coefficient).
\end{itemize}
The link between these two quantities can be made explicit by developing:
$$\big\| K_\theta^C(\x,\x') - \Id \big\|^2_F
= \big\| K_\theta^C(\x,\x') \big\|^2_F - 2 \Tr(K_\theta^C(\x,\x')) + d$$
which rewrites as:
$$\left(1 - \frac{D^2}{2d} \right) = \frac{\Tr(K_\theta^C(\x,\x')) }{d}   + \frac{1}{2}\left( 1 - \frac{\big\| K_\theta^C(\x,\x') \big\|^2_F}{d} \right).$$
The last term lies in $[0,1]$ and measures the mismatch between the vector subspaces generated by the two families of gradients $\left(\nabla_{\!\theta} f^i(\x)\right)_i$ and $\left(\nabla_{\!\theta} f^i(\x')\right)_i$. It is 1 when $f_\theta(\x)$ and $f_\theta(\x')$ can be moved independently, and 0 when they move jointly (though not necessarily in the same direction).

As our two similarity measures $1 - \frac{D^2}{2d}$ and $\frac{1}{d}\Tr(K_\theta^C(\x,\x'))$ have same optimum ($\Id$) and are closely related, in the sequel we will focus on the second one and define:
\begin{equation}
  k_\theta^C(\x,\x') \;=\;
  \frac{1}{d} \,\Tr\, K^C_\theta(\x,\x') \;.
\end{equation}

\subsection{Metrics on output: rotation-invariance}

Similarity in $\R^d$, to compare $\vv$ and $\vv' = M\vv$, might be richer than just checking whether the vectors are equal or close in $L^2$ norm.

For instance, one could quotient the output space by the group of rotations, in order to express a known or desired equivariance of the network to rotations. If the output is the predicted motion of some object described in the input, one could wish indeed that if the input object is rotated by an angle $\phi$, then the output should be rotated as well with the same angle.

In that case, given two inputs $\x$ and $\x'$ and associated output variations $\vv$ and $\vv'$, without knowing the rotation angle if applicable, one could consider all possible rotated versions $R_\phi \vv' = R_\phi M \vv$, where $R_\phi$ is the rotation matrix with angle $\phi$, and pick the best angle $\phi$ that maximizes the alignment $\vv \cdot R_\phi M \vv$, \ie such that $R_\phi M$ is the closest to the $d \times d$ identity matrix. This can be computed easily in closed form, for instance in the 2-dimensional case as follows.

The $2 \times 2$ matrix of interest (Eq.~$\ref{eq:multidimkern}$) can be written as
the product of two $p \times 2$ matrices of the form $G (G^T G)^{-1/2}$, where $G$ is the matrix containing the gradient of all coordinates. Rotating the coordinates of $G$ amounts to considering $G R_\phi (R_\phi^TG^T GR_\phi)^{-1/2} = G (G^T G)^{-1/2} R_\phi$ instead. Thus the effect of rotation is just right-multiplying our $2 \times 2$ matrix $M$ of interest (Eq.~$\ref{eq:multidimkern}$) by $R_\phi$.
We are thus interested into getting $M R_\phi$ as close as possible to the $2 \times 2$ identity. For our trace-based similarity kernel (Eq.~\ref{eq:multidimkernsum}), this amounts to maximizing $\Tr(MR_\phi) = \cos(\phi)(M_{11}+M_{22}) + \sin(\phi)(M_{12}-M_{21})$ \wrt $\phi$, whose optimal value is:
\begin{align*}
  k^{C, \rot}_\theta(\x,\x') & = \frac{1}{2} \sqrt{(M_{11}+M_{22})^2 + (M_{12}-M_{21})^2}\\
  & =  \frac{1}{2}  \sqrt{ \big\|M\big\|_F^2 + 2 \det M}
  \end{align*}
where $M = K_\theta^C(\x,\x')$. This quantity is indeed rotation-invariant, as the Frobenius norm and the determinant do not change upon rotations. Note that one could also consider instead the subspace match $\frac{1}{d}\big\|M\big\|_F^2 $. The main difference between the two is that the first one penalizes mirror symmetries (through $\det M$) while the second one does not.

Note that other metrics are possible in the output space. For instance, the loss metric quantifies the norm of a move $\vv$ by its impact on the loss $\left.\frac{dL(y)}{dy}\right|_{f_\theta(\x)}(\vv)$. It has a particular meaning though, and is relevant only if well designed and not noisy, as seen in the remote sensing image registration example. Note also that in such a case the associated similarity would not be intrinsic anymore to the neural network as it depends on the loss.


\section{Estimating density}
\label{sec:estim2}

\subsection{Toy problem}

The toy problem used in the paper to test the various estimators for neighbor count estimation consists of predicting a one dimensional function, namely a sinusoid (such as in Fig.\ref{fig:toy_problem_2d} (a)). We can easily change the difficulty of the problem by using different values of frequency. The neural network would perform this mapping: $y = \sin(2 \pi f x), x \in [0, 1]$.

A problem arises however when estimating the number of neighbors because the input space has 2 boundaries at $x=0$ and $x=1$, leading to fewer neighbors when $x$ approaches either of those boundaries. To avoid this problem, we transform the input space to a 2D circle. Namely, the task is now $y = sin(2\pi f \alpha(x)), x \in \{(\cos(2\pi\alpha), \sin(2\pi\alpha)), \alpha \in [0, 1]\}$, with the input space having no boundaries.

The dataset is generated with n=2048 input points. The network used is fully-connected and has 5 hidden layers of 64 neurons trained with the Adam optimizer for 80 epochs with a base learning rate of $1e^{-4}$. An experiment consist of training the network on a dataset generated with a specific frequency f. Each experiment was repeated 5 times, in order to take the median of every result to limit the variance due to the neural network stochastic training.

We can see in Fig.\ref{fig:toy_problem_2d} (b) the proposed soft estimate $k_\theta^C$ for each input point (projected to 1D). As expected we observe that the number of neighbors drops when the curvature is high: the objective changes quickly and the network adjusts to better distinguish inputs in places of higher curvature.

\begin{figure}
	\centering
	\begin{subfigure}[b]{0.45\textwidth}
		\centering
		\includegraphics[width=\linewidth]{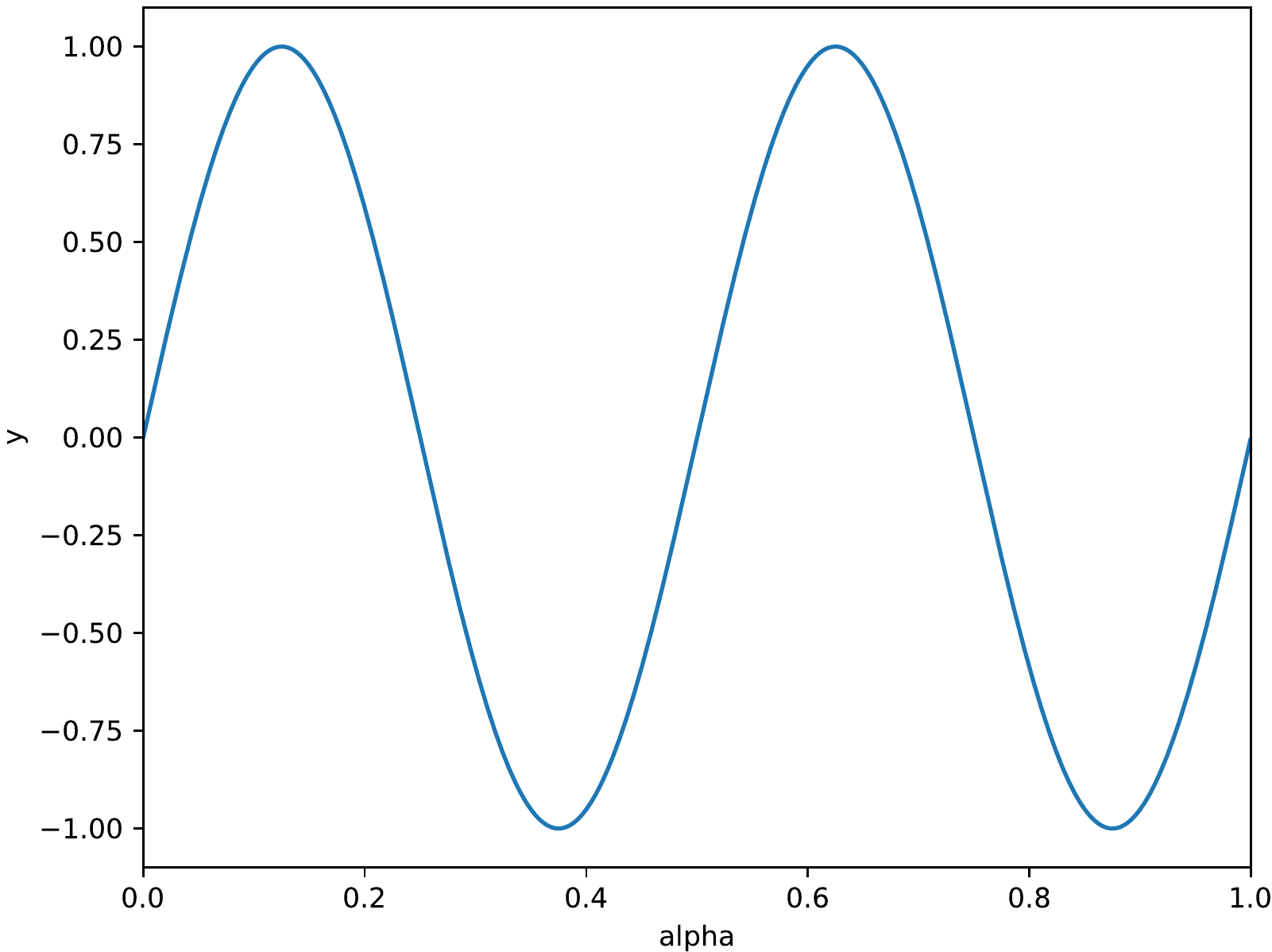}
		\caption{Function to predict.}
	\end{subfigure}
	\begin{subfigure}[b]{0.45\textwidth}
		\centering
		\includegraphics[width=\linewidth]{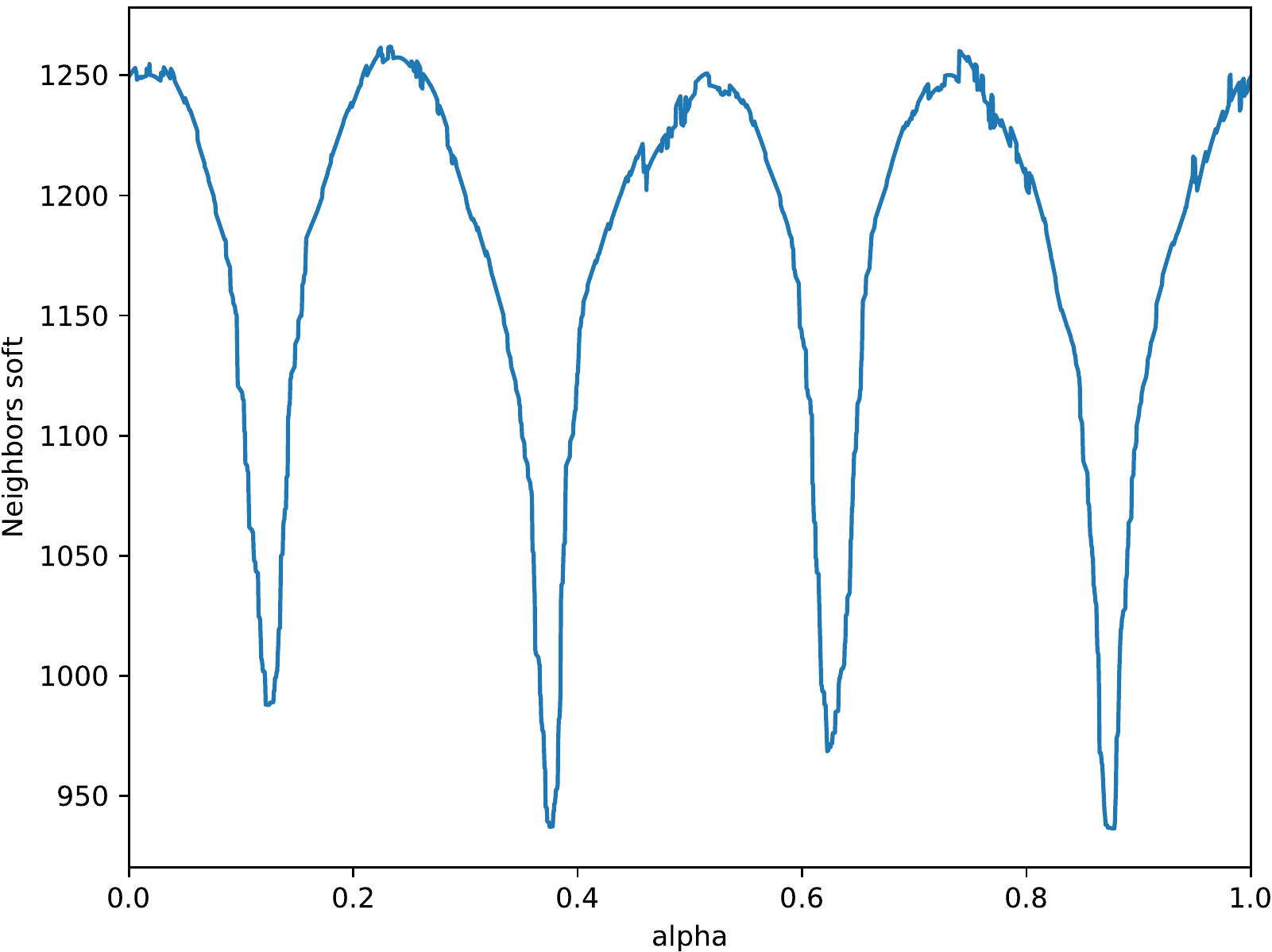}
		\caption{Neighbors soft estimate.}
	\end{subfigure}
	\caption{Toy problem with the frequency f = 2.}
	\label{fig:toy_problem_2d}
\end{figure}

\begin{figure}
	\centering
	\includegraphics[width=0.4\linewidth,trim={100 25 40 60},clip]{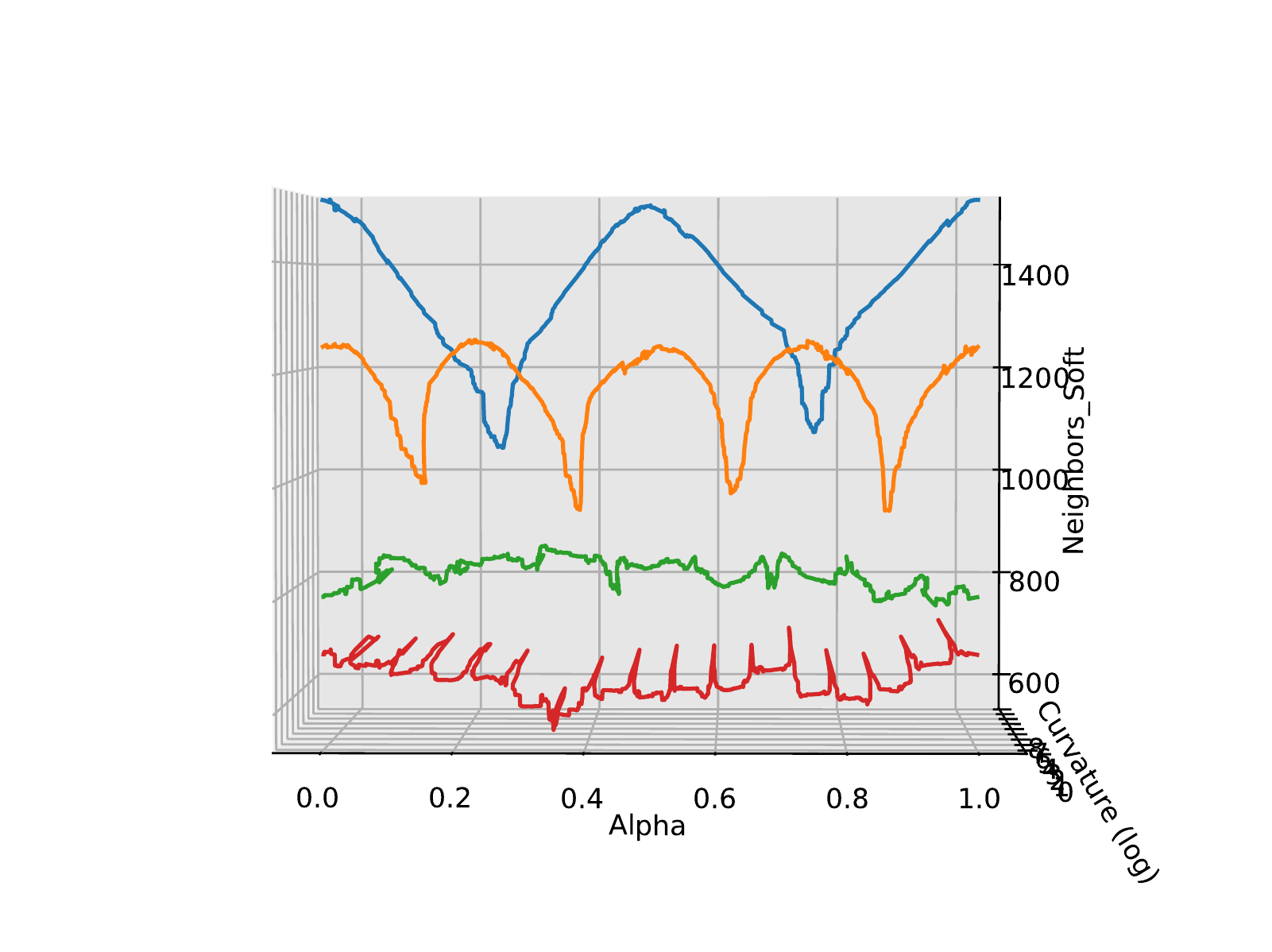}
	\includegraphics[width=0.5\linewidth,trim={50 25 0 0},clip]{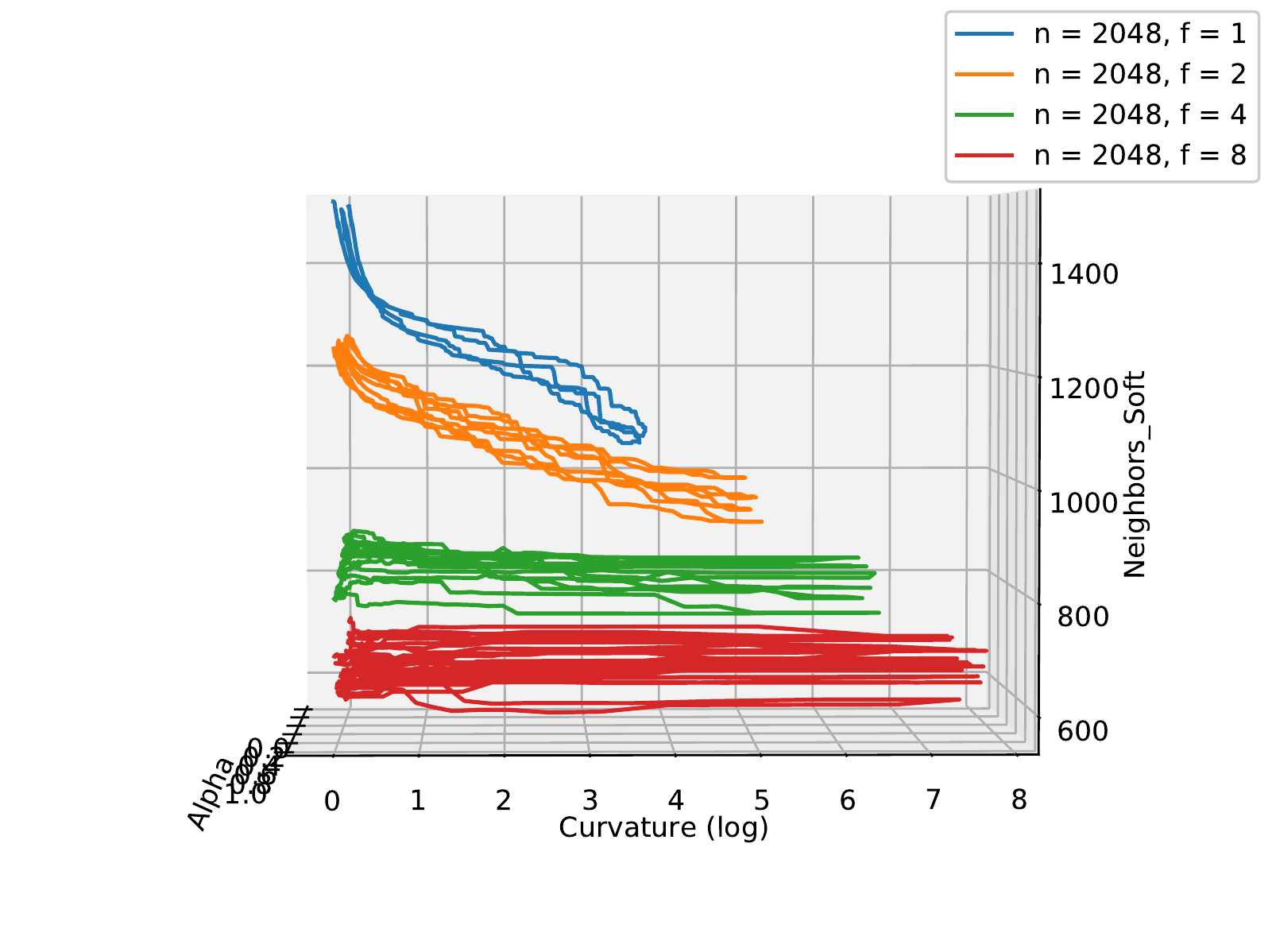}
	\includegraphics[width=0.5\linewidth,trim={50 0 0 60},clip]{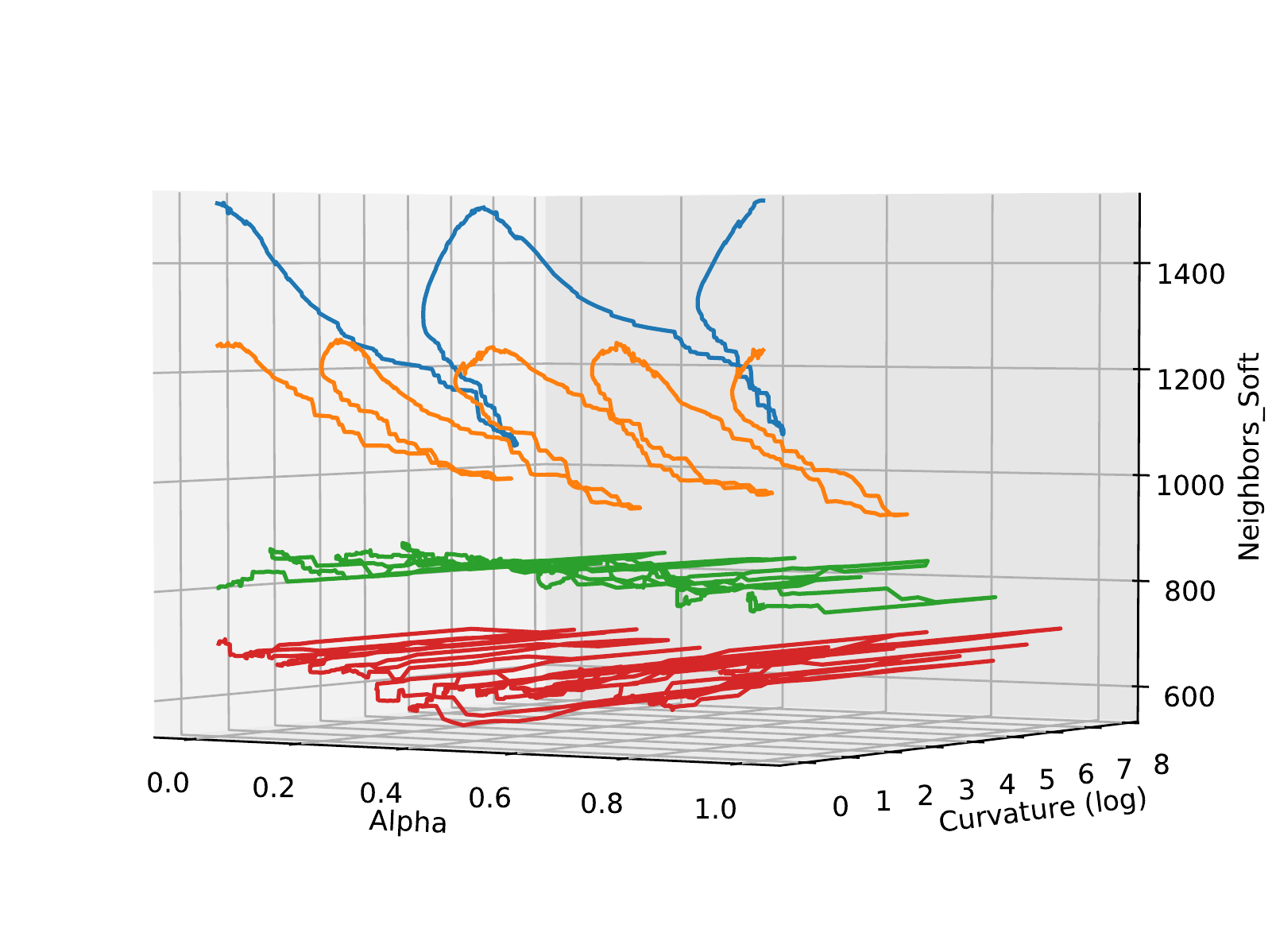}
	\caption{3D plot of neighbors soft with varying frequency. Script and data to plot interactively in attached files. Run the bash script "main\_plot\_exps.paper.sh" to reproduce this exact figure. Alternatively use "main\_plot\_exps.py" with arguments of your choosing to plot different values (run "python main\_plot\_exps.py -h" to see possible arguments).}
	\label{fig:toy_problem_3d}
\end{figure}

\subsection{Other possible uses}

\paragraph{Density homogeneity as an optimization criterion}
The estimations above are meant to be done post-training.
This said, one could control density
explicitly, by computing the number of neighbors for all points, and asking it to be in a reasonable range, or in a reasonable proportion $q$ of the dataset size $\mathcal{D}$, by adding \eg to the loss
$\sum_i \left( \frac{ N_S(\x_i) }{ \mathcal{D} }
- q  \right)^2$.
Online learning could also make use of such tools, to sample first lowly-populated areas, where uncertainty is higher.


\section{Enforcing similarity}

We give here a few more details on the homonymous section of the paper.

\subsection{Complexity}

A gradient descent step on this quantity for a given pair $(\x,\x')$ (in a mini-batch approach, \eg) requires the computation of the gradient $\nabla_\theta k^C_\theta(\x,\x') = \nabla_\theta \left( \nabla_\theta f_\theta(\x) \cdot \nabla_\theta f_\theta(\x') \right)$.
While a naive approach would require the computation of a second derivative, \ie a matrix of size $p \times p$ where $p$ is the number of parameters,
it is actually possible to compute
$\nabla_\theta k^C_\theta(\x,\x') = \nabla_\theta \sum_i \frac{d f_\theta(\x)}{d\theta_i}  \frac{d f_\theta(\x')}{d\theta_i}$ in linear time $O(p)$, taking advantage of the serial structure of the computational graph.
The framework enabling such computations is already available on common deep learning platforms, initially intended for the computation of $\nabla_\x \nabla_\theta f_\theta(\x)$ for some variations on GANs.

\subsection{Group invariance}
\label{sec:group2}

Dataset augmentation is a standard machine learning technique; when augmenting the dataset by a group transformation of the input (\eg, translation, rotation...) or by small intensity noise, new samples are artificially created, to augment the dataset size and hope for invariance to such transformations. One can ask the network to consider orbits of samples as similar with the technique above.

Furthermore, if the group infinitesimal elements are expressible as differential operators $e_k$, one could require directly, for all $\x$, invariance in the tangent plane in the directions of these differential operators:
$$\|  \partial_\x \nabla_\theta f(\x) \cdot e_k(\x) \|^2$$
which is the limit of
$\frac{1}{\epsi^2}\| \nabla_\theta f_\theta(\x) - \nabla_\theta f_\theta\left(\x + \epsi e_k(\x)\right) \|^2$ when $\epsi \to 0$.
For instance, in the case of image translations, the operator is $e: \x \mapsto \nabla_x \x(x)$ where $x$ denotes spatial coordinates in the image $\x$, as $\x(x+\tau) = \x(x) + \tau \cdot \nabla_x \x(x) + O(\tau^2)$.
This is however not recommended, as representing a translation with such a spatially-local operator does not take into account the spatially-irregular nature of image intensities.

Note that to the opposite of standard robustification techniques considering regularizers such as $\sum_{\x} \| \nabla_\theta f_\theta(\x)\|^2$, we ask not gradients to be always small, but to be smooth, and in certain directions only.

\subsection{Dynamics of learning: Experimentation details}

The results in figure 6 show the average and standard deviation over 60 runs for each curve. The x-axis is the number of batches to the network is trained on (with a batch size of 16). The y-axis is the accuracy metric on the whole validation set. The network architecture is made of 2 convolutions layers (with a kernel size of 5), 2 linear layers and uses PReLU non-linearities. We used Adam with a learning rate of 1e-3 and no weight decay.

We tested other architecures on MNIST: one with residual blocks, one deeper (8 convolutions) and one with tanh non-linearities. Similar results were observed on all cases. Additional tests were performed on CIFAR10 with a VGG architecture and only negligible benefits were observed.


\section{Noisy Map Alignment Analysis}
\label{sec:denoise2}

The task here it to align maps in the form of a list of polygons with remote sensing images while using only the available noisy annotations. We analyze the model developed in a previous work \cite{anonymous}. Specifically, the model is trained in a multiple-rounds training scheme to iteratively align the available noisy annotations, which provides a better ground truth used to train a better model in the next round. An open question is why multiple rounds are needed in this noisy supervision setting, and why not all the noise can be removed in a single training step.

More specifically, the model is made out of 4 neural networks. Each is trained on a different resolution (in terms of ground pixel size) and are applied in a multi-resolution pyramidal manner. In all our experiments we only analyzed the networks trained for a ground pixel size of 4 time smaller than the reference ground pixel size which is $0.3m$. We used the already-trained networks for each round, of which there are 3.

The network was trained with small patches of (image, misaligned map) pairs from images of the Inria dataset \cite{maggiori2017dataset} and the Bradbury dataset \cite{bradbury_buildings_roads_height_dataset}. Ideally we would want to compute the similarities of every possible pairs of inputs, with a small patch size of $124$ px. However, given that a typical image of the training dataset is $1250 \times 1250$ px (after rescaling) and there are a few hundred of them (328 from the Inria dataset, only counting images where OSM annotations \cite{osm}), this would result in 32800 patches. The resulting amount of similarities to compute would be around half a billion. As the network has a few million of parameters and the output is 2D, each computation of similarity takes around $0.5$s. To make any computation feasible, we first sample 10 patches per image from the 328 of the Inria dataset. Those patches are chosen at random, as long as there is at least one building lying fully in the patch. As some images have rather sparse buildings, some images give less than 10 patches. We thus obtain 3045 patches representing the dataset. The amount of similarities to compute would be close to 5 million. To study all patches globally, we can use the soft neighbors estimator $k_\theta^C$ which has a linear complexity and allows us to compute the amount of neighbors for all 3045 patches in under an hour. However it is also interesting to go in deeper detail and compute similarities for some input pairs. We thus furthermore reduce the amount of pairs by estimating all similarities only for a very small number of patches, for example 10. This results in a $10\times3045$ similarity matrix.

\subsection{Soft estimate on a sampling of the training dataset}

In this section we present the results of computing the soft neighbors estimator $k_\theta^C$ on the 3045 sampled patches of inputs. We obtain results for the 3 networks of the 3 rounds of the noisy-supervision multi-rounds training scheme. Fig.\ref{fig:overall_hist} shows a histogram of the soft neighbors estimations. It additionally representative input patches for each bin of the histogram. Those representative patches are chosen so that their neighbor count is closest to the right edge of that bin. We especially observe that inputs in round 2 have more neighbors than the other 2 rounds. This particularity of round 2 will be seen throughout the remaining results. It is the round that aligns the most the annotations (see the Fig.2 on accuracy cumulative distributions in the paper). Round 3 does not perform any more alignment, that might be the reason why its results are different from those of round 2.

\subsection{Similarities on pairs of input patches}

In this section are the results for the computation of similarities between pairs of input patches. In a first experiment, for every round we chose the 10 patches shown in Fig.\ref{fig:overall_hist}, and computed their similarities with all the other 3045 patches. In order to visualize this data, we computed the 10-nearest neighbors in terms of similarity for each of those patches, see Fig.\ref{fig:round_0_overall_hist_k_nearest}, \ref{fig:round_1_overall_hist_k_nearest}, \ref{fig:round_2_overall_hist_k_nearest}. We computed the histogram of similarities as well, see Fig.\ref{fig:overall_hist_individual_hist}.

In a second experiment, to better compare between rounds, we used another set of 10 patches, this time the same set for each round. Specifically, we sampled 10 patches from the bloomington22 image of the Inria dataset. As just before we computed the 10-nearest neighbors (Fig.\ref{fig:round_0_bloomington22_k_nearest}, \ref{fig:round_1_bloomington22_k_nearest}, \ref{fig:round_2_bloomington22_k_nearest}) and the histogram of similarities(Fig.\ref{fig:bloomington22_individual_hist}) for a visualization of those measures.

Generally speaking, inputs in round 2 have more neighbors and the 10-nearest ones are closer than in other rounds (see Fig.\ref{fig:round_0_overall_hist_k_nearest}, \ref{fig:round_1_overall_hist_k_nearest}, \ref{fig:round_2_overall_hist_k_nearest} and Fig.\ref{fig:round_0_bloomington22_k_nearest}, \ref{fig:round_1_bloomington22_k_nearest}, \ref{fig:round_2_bloomington22_k_nearest}). For each parch, its closest neighbors generally (for similarity > 0.8) look similar from a human point of view. For example patches with sparse houses and trees have the same kind of neighbors. The same can be said for patches with parking lots and big roads. Another group are patches that are almost empty of buildings, with a lot of low vegetation. Other patch nearest neighbors are more difficult to interpret.
In Fig.\ref{fig:overall_hist_individual_hist} and Fig.\ref{fig:bloomington22_individual_hist} we can see that for round 2, the spread of the similarities of the selected patches is smaller and the peak of the histogram are closer to the right, meaning all patches are closer than in other rounds. Additionally in Fig.\ref{fig:overall_hist_individual_hist} we can observe that the bottom patch has closer neighbors than the top patch, this is because the top patch corresponds to the left patch in \ref{fig:overall_hist} and the bottom one corresponds to the right patch in \ref{fig:overall_hist}.


\input{preuves_denoising.tex}


\subsection{Data augmentation as a label denoising technique}
Data augmentation can be seen as label denoising, as it multiplies the number of neighbors. Indeed, in the infinite sampling limit, where the dataset becomes a probability distribution over all possible images, adding a transformed copy $\x' = T_\phi\, \x$ of a given point $\x$ (\eg rotating it with an angle $\alpha_\phi$ and adding small noise $\epsi_\phi$) means adding $(\x', l(\x))$ to the dataset, where $l(\x)$ is the desired label for $\x$. But if $(\x', l(\x'))$ was already in the dataset, this amounts to enriching the possible labels for $\x'$. Supposing $T_\phi$ is an invertible transformation parameterized by $\phi$, full data augmentation (\ie for all possible $\phi$, applied on all points $\x$) enriches $\x'$ with all labels $l( T_\phi^{-1}(\x') )$. In case of i.i.d.~label noise, data augmentation will thus reduce this noise by a factor $\sqrt{\text{number of copies}}$.




\begin{figure}
	\centering
	\begin{subfigure}[b]{\textwidth}
		\includegraphics[width=\linewidth]{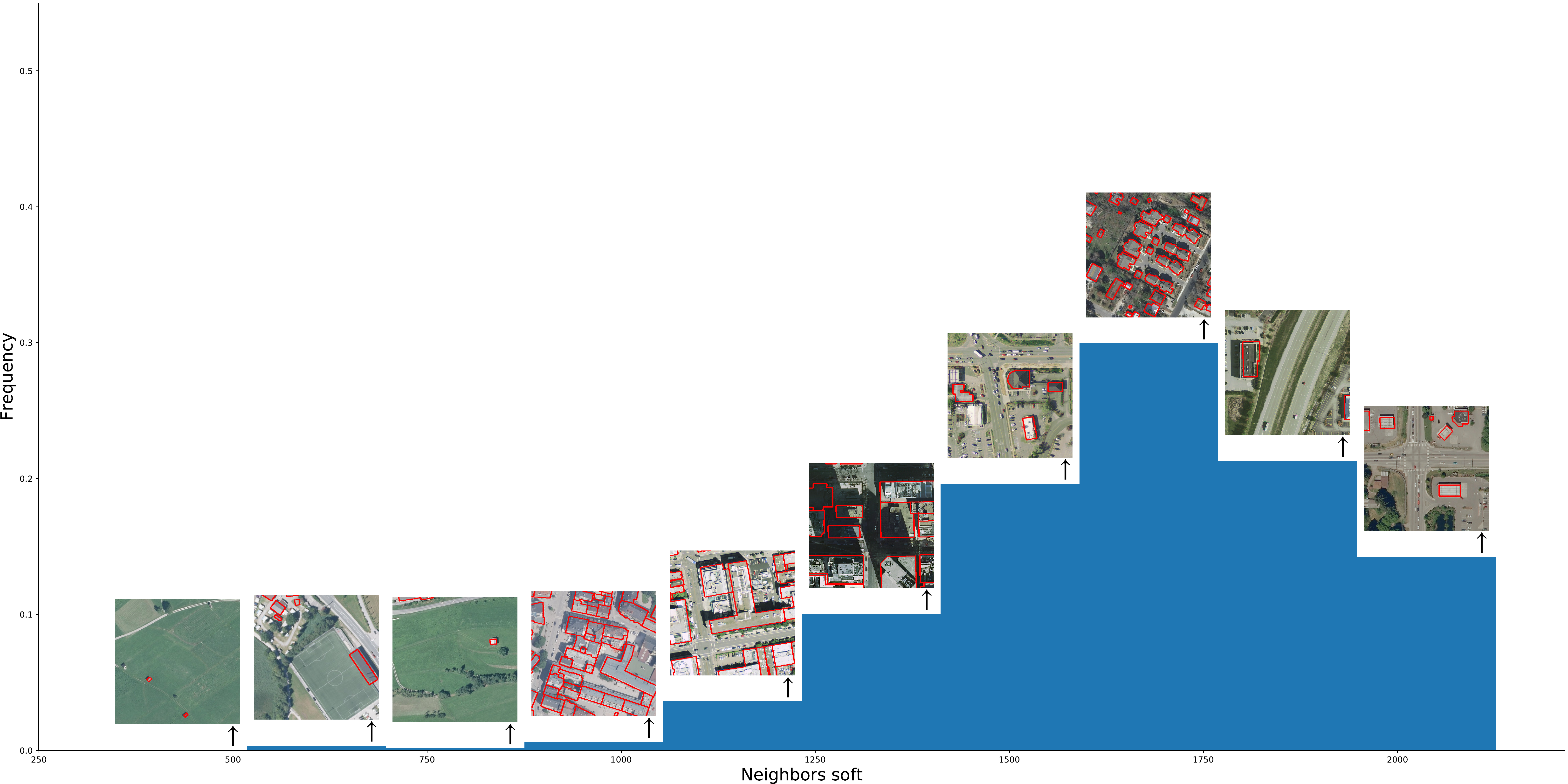}
		\caption{Round 1}
	\end{subfigure}
	\begin{subfigure}[b]{\textwidth}
		\includegraphics[width=\linewidth]{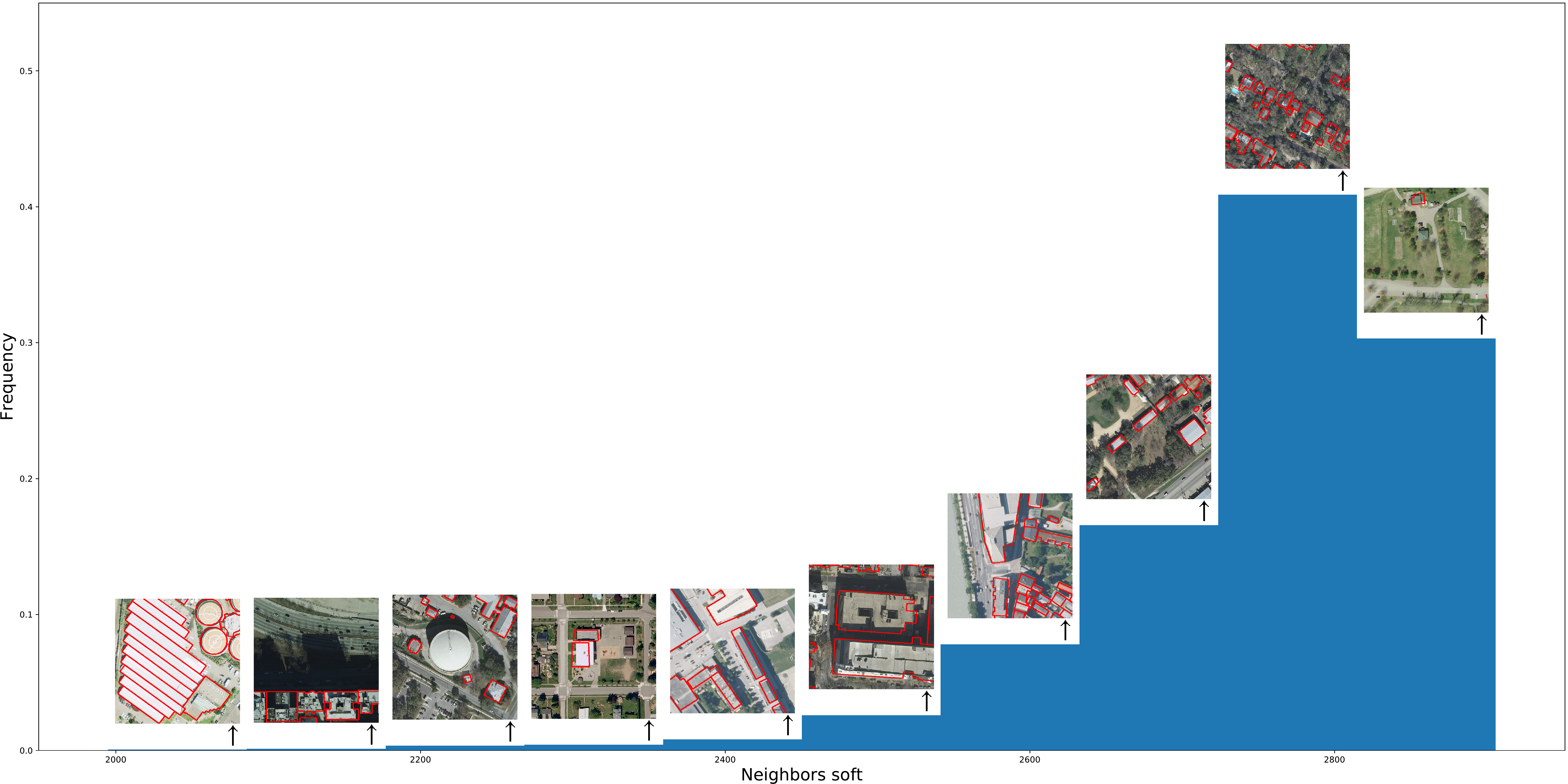}
		\caption{Round 2}
	\end{subfigure}
	\begin{subfigure}[b]{\textwidth}
		\includegraphics[width=\linewidth]{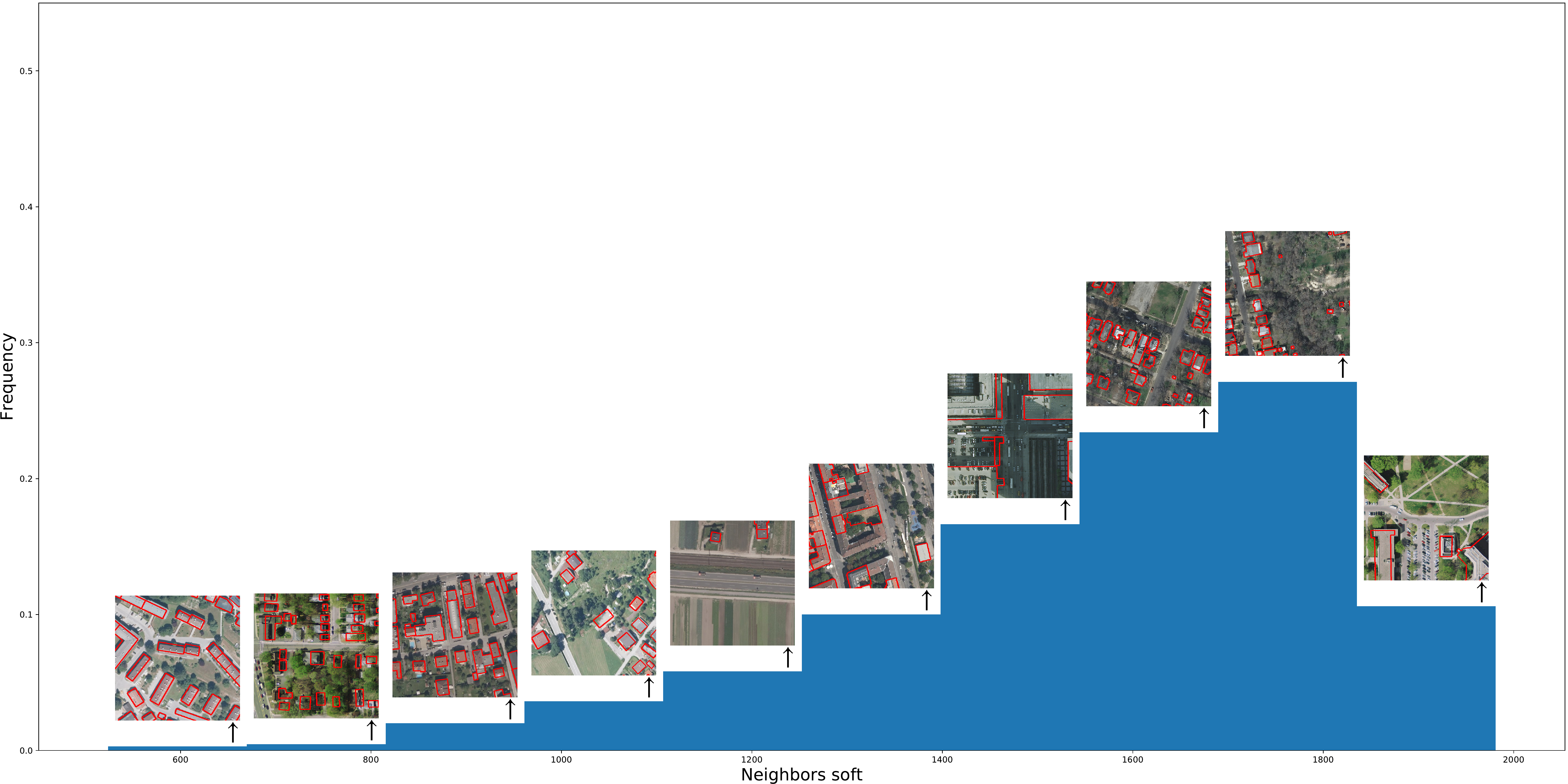}
		\caption{Round 3}
	\end{subfigure}
	\caption{Histogram of the soft estimate of neighbors on 3045 patches. Horizontal scale is different for each.}
	\label{fig:overall_hist}
\end{figure}


\begin{figure}
	\centering
	\includegraphics[width=\linewidth]{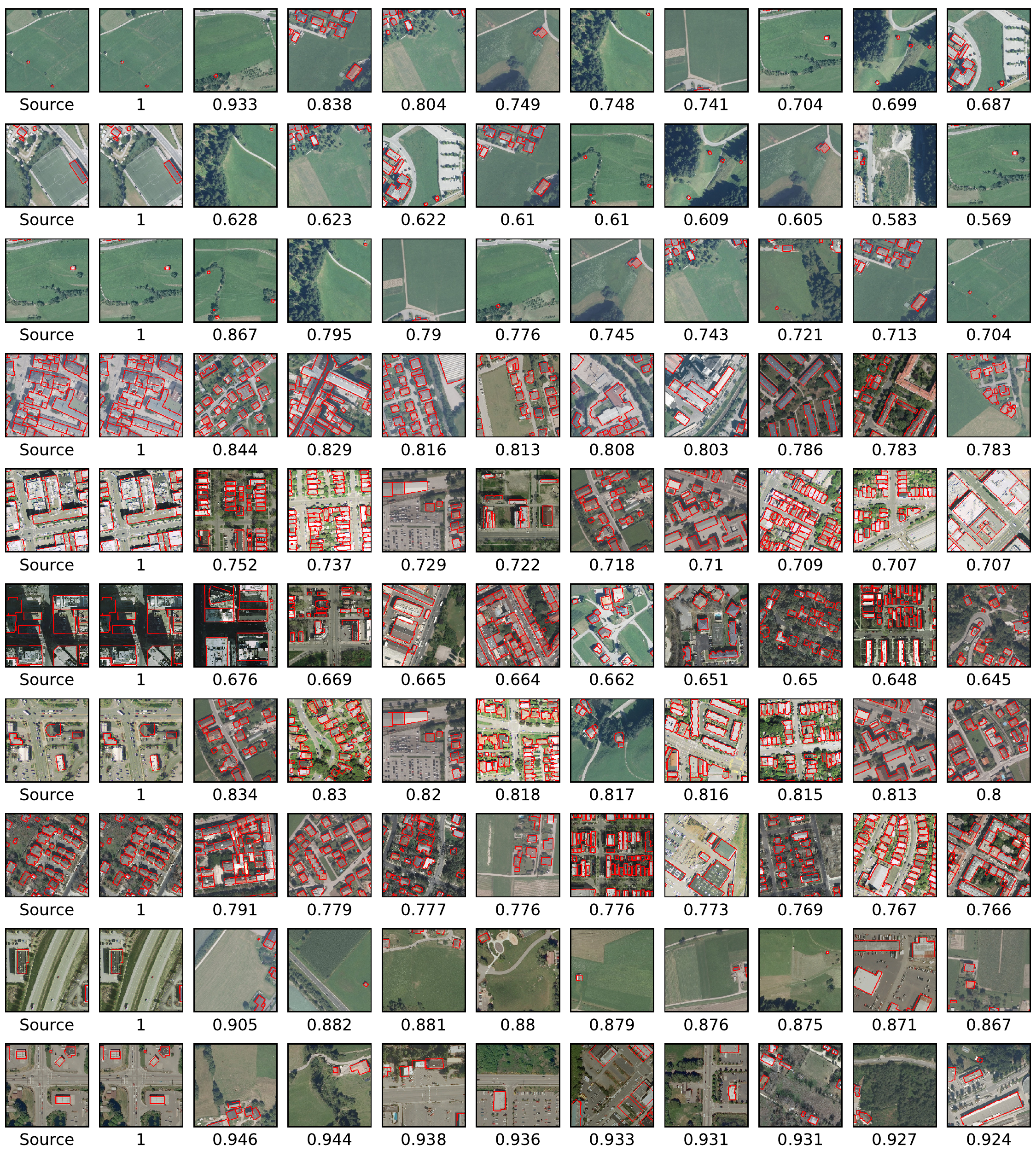}
	\caption{\textbf{Round 1}: k-nearest neighbors with k=10. The 10 patches selected correspond to the 10 patches of Fig.\ref{fig:overall_hist} for that round.}
	\label{fig:round_0_overall_hist_k_nearest}
\end{figure}
\begin{figure}
	\centering
	\includegraphics[width=\linewidth]{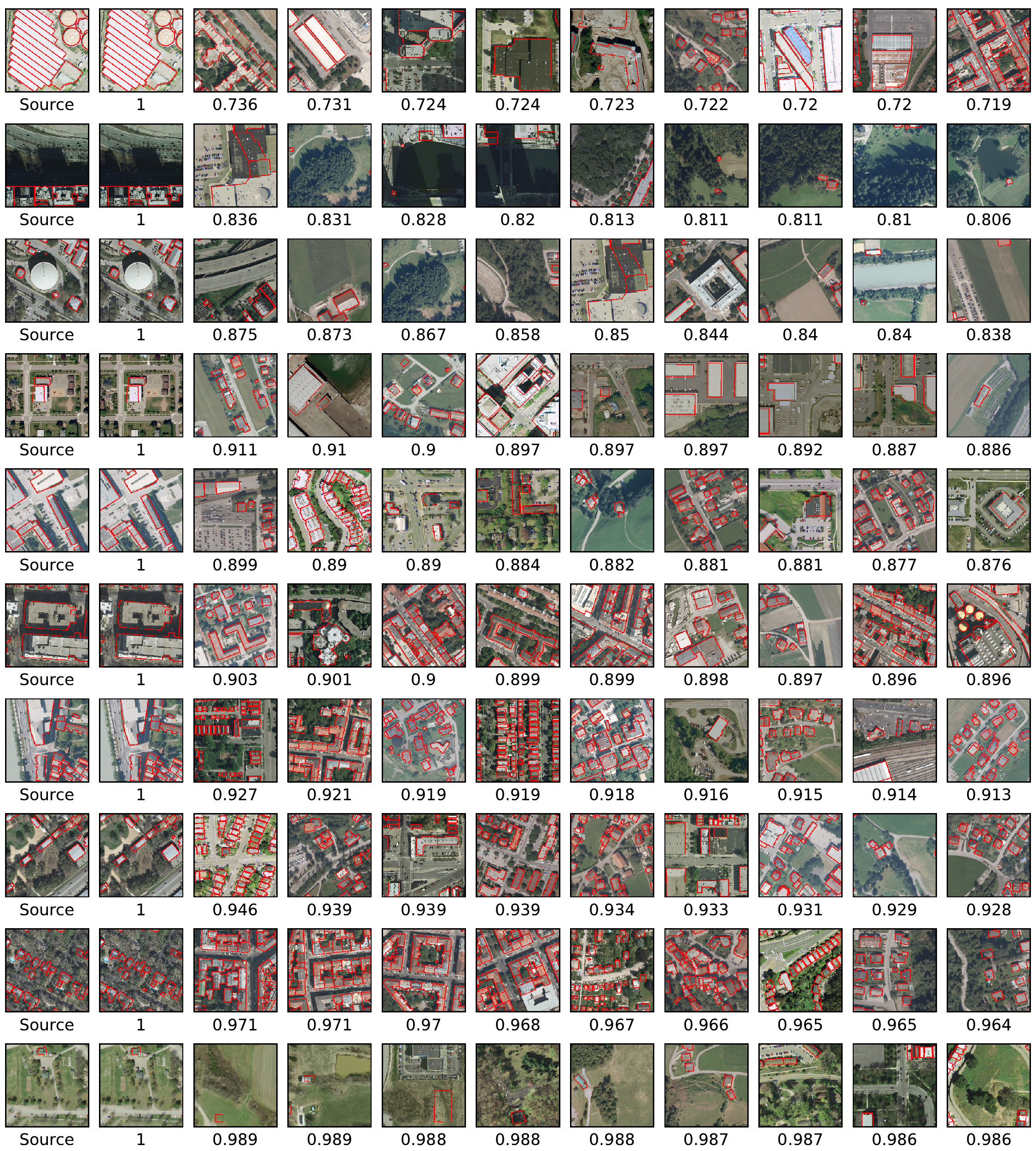}
	\caption{\textbf{Round 2}: k-nearest neighbors with k=10. The 10 patches selected correspond to the 10 patches of Fig.\ref{fig:overall_hist} for that round.}
	\label{fig:round_1_overall_hist_k_nearest}
\end{figure}
\begin{figure}
	\centering
	\includegraphics[width=\linewidth]{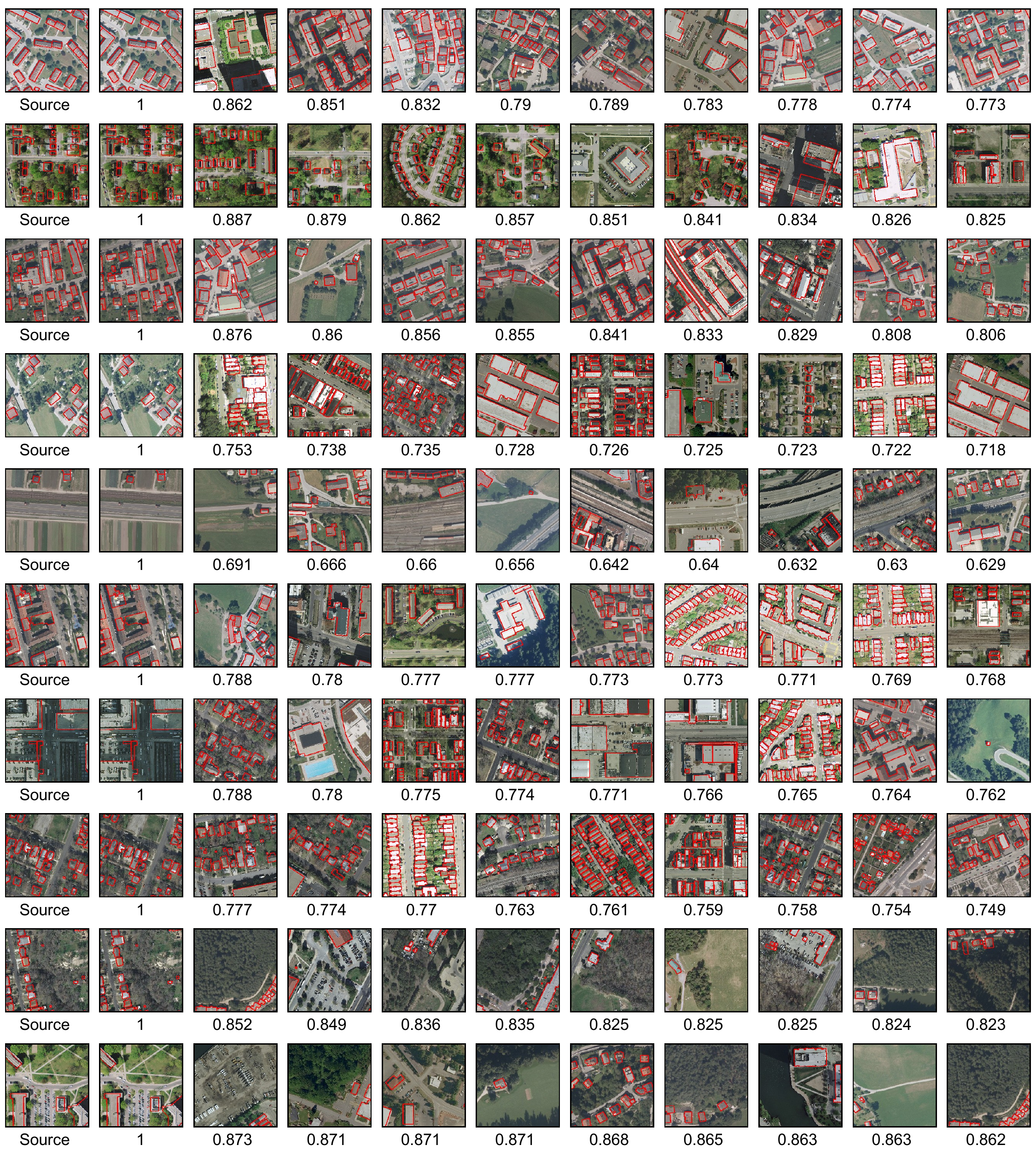}
	\caption{\textbf{Round 3}: k-nearest neighbors with k=10. The 10 patches selected correspond to the 10 patches of Fig.\ref{fig:overall_hist} for that round.}
	\label{fig:round_2_overall_hist_k_nearest}
\end{figure}

\begin{figure}
	\centering
	\begin{subfigure}[b]{0.3\textwidth}
		\centering
		\caption{Round 1}
		\includegraphics[width=0.7\linewidth]{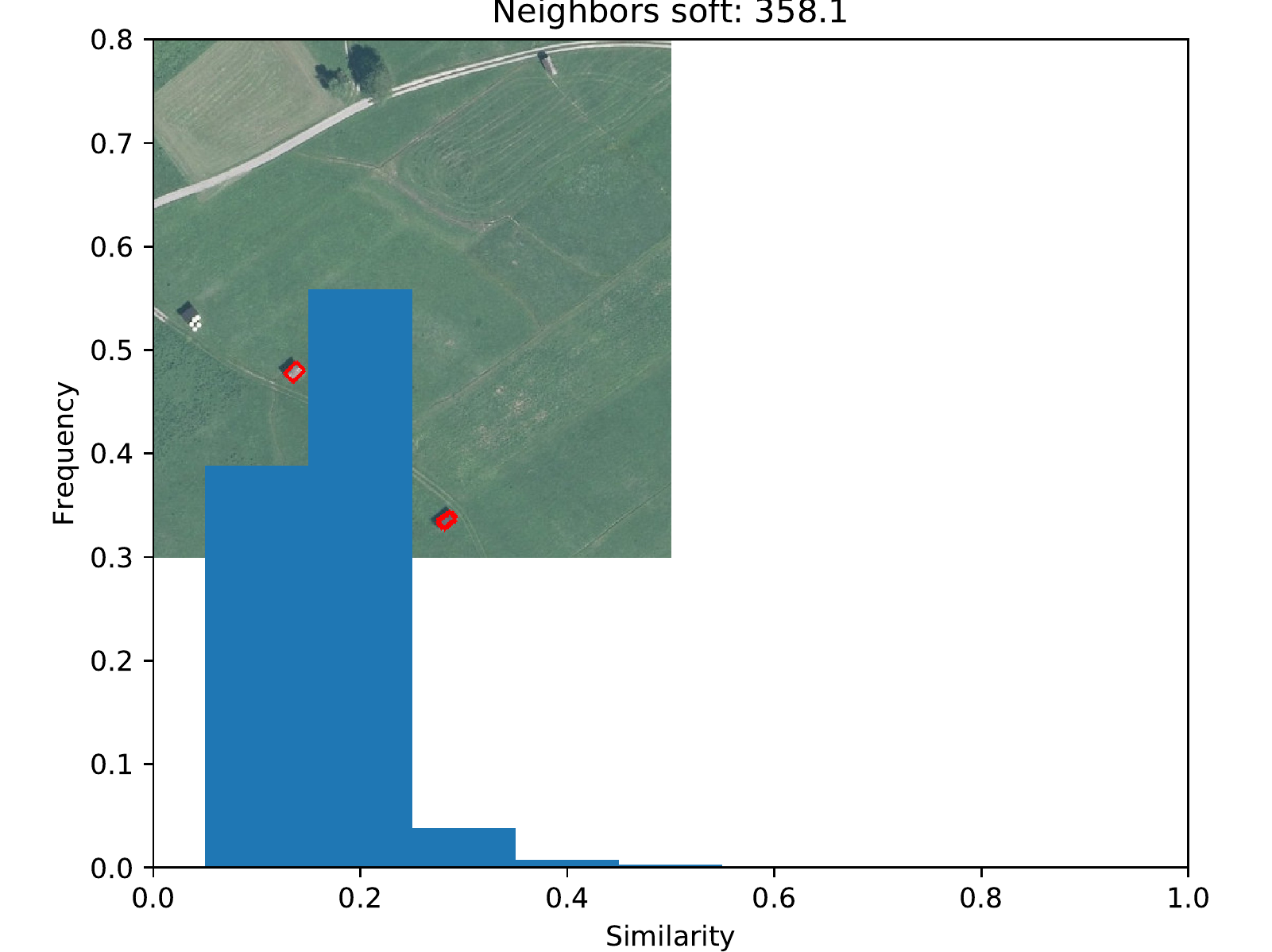}
		\includegraphics[width=0.7\linewidth]{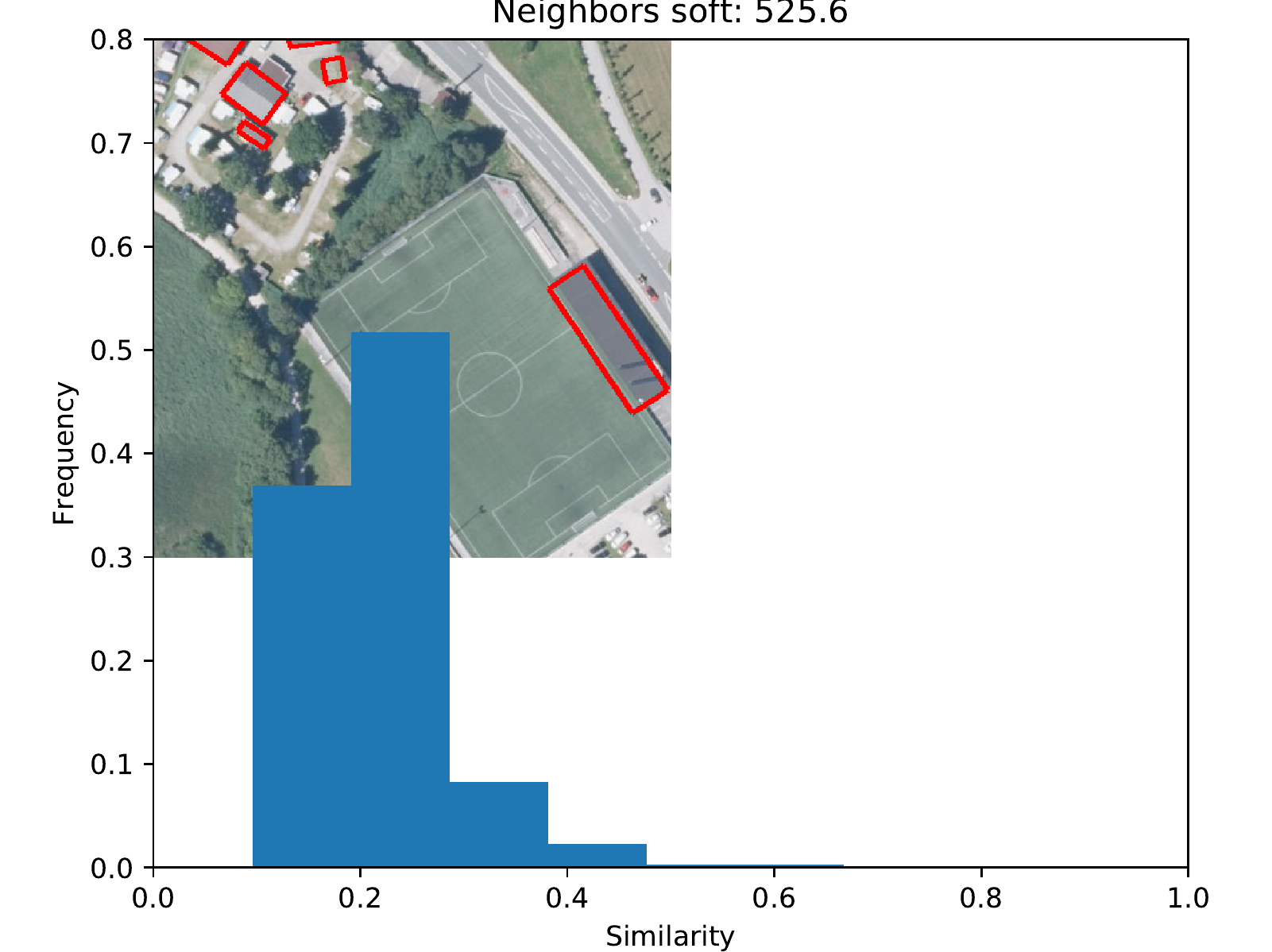}
		\includegraphics[width=0.7\linewidth]{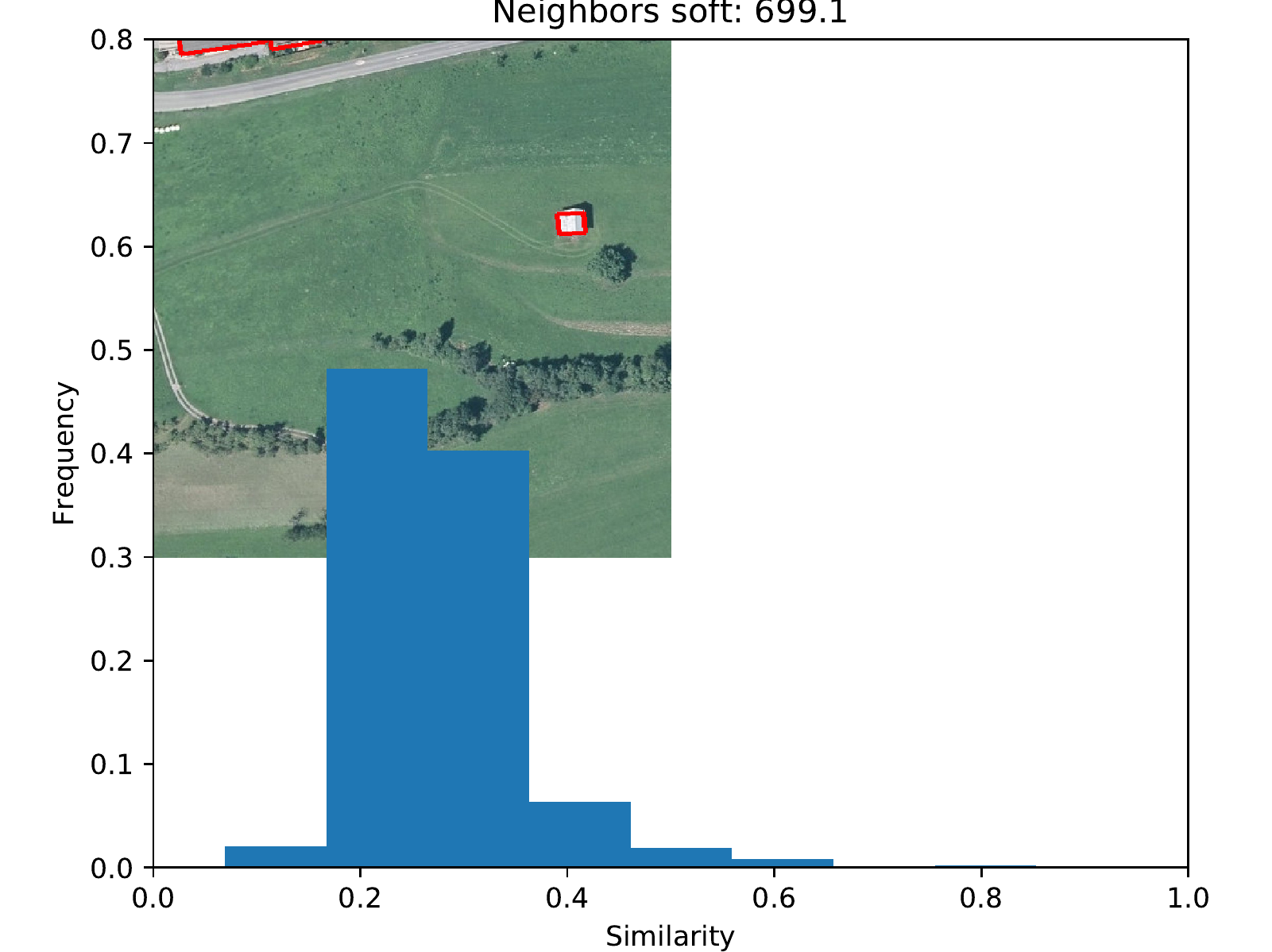}
		\includegraphics[width=0.7\linewidth]{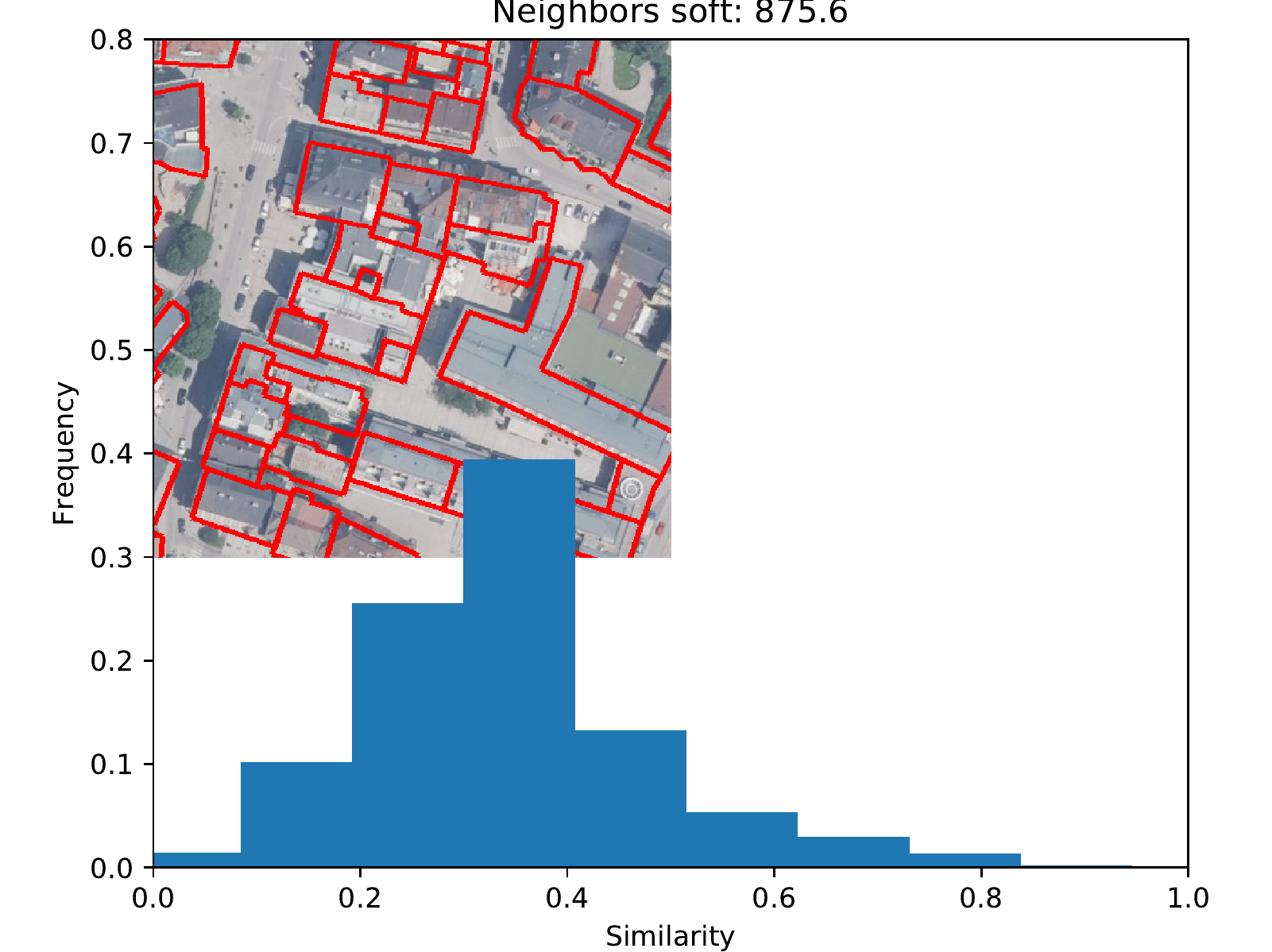}
		\includegraphics[width=0.7\linewidth]{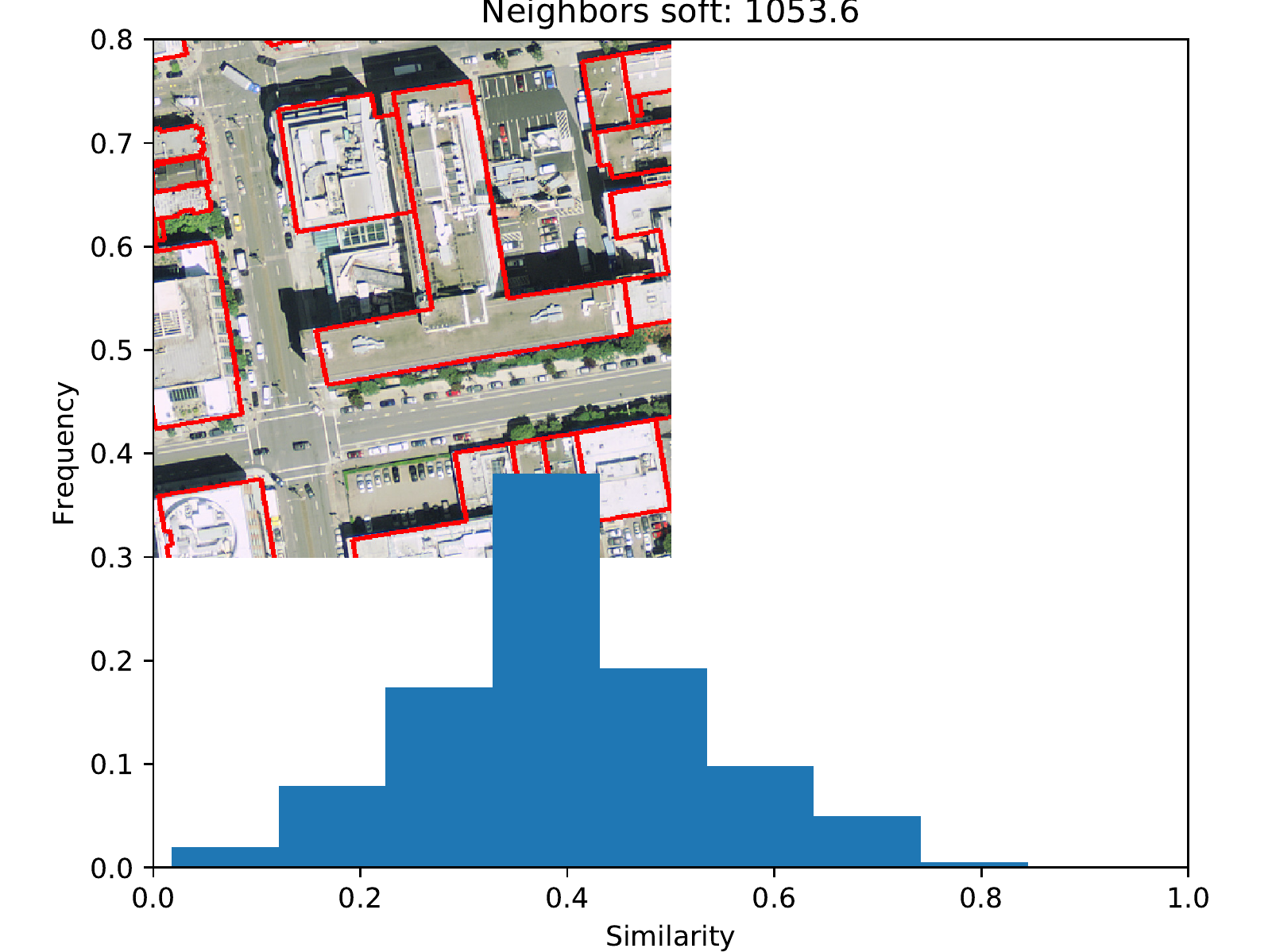}
		\includegraphics[width=0.7\linewidth]{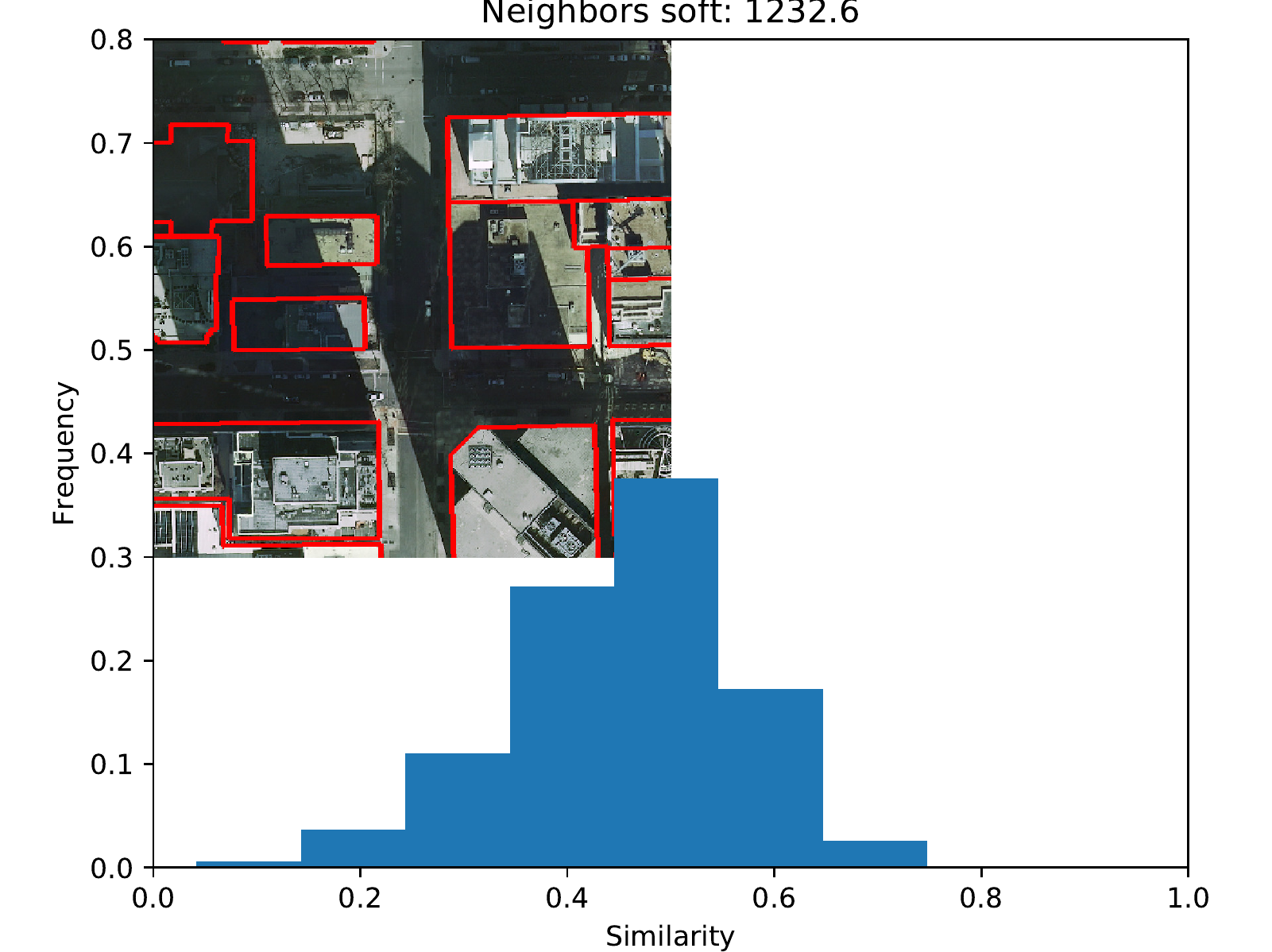}
		\includegraphics[width=0.7\linewidth]{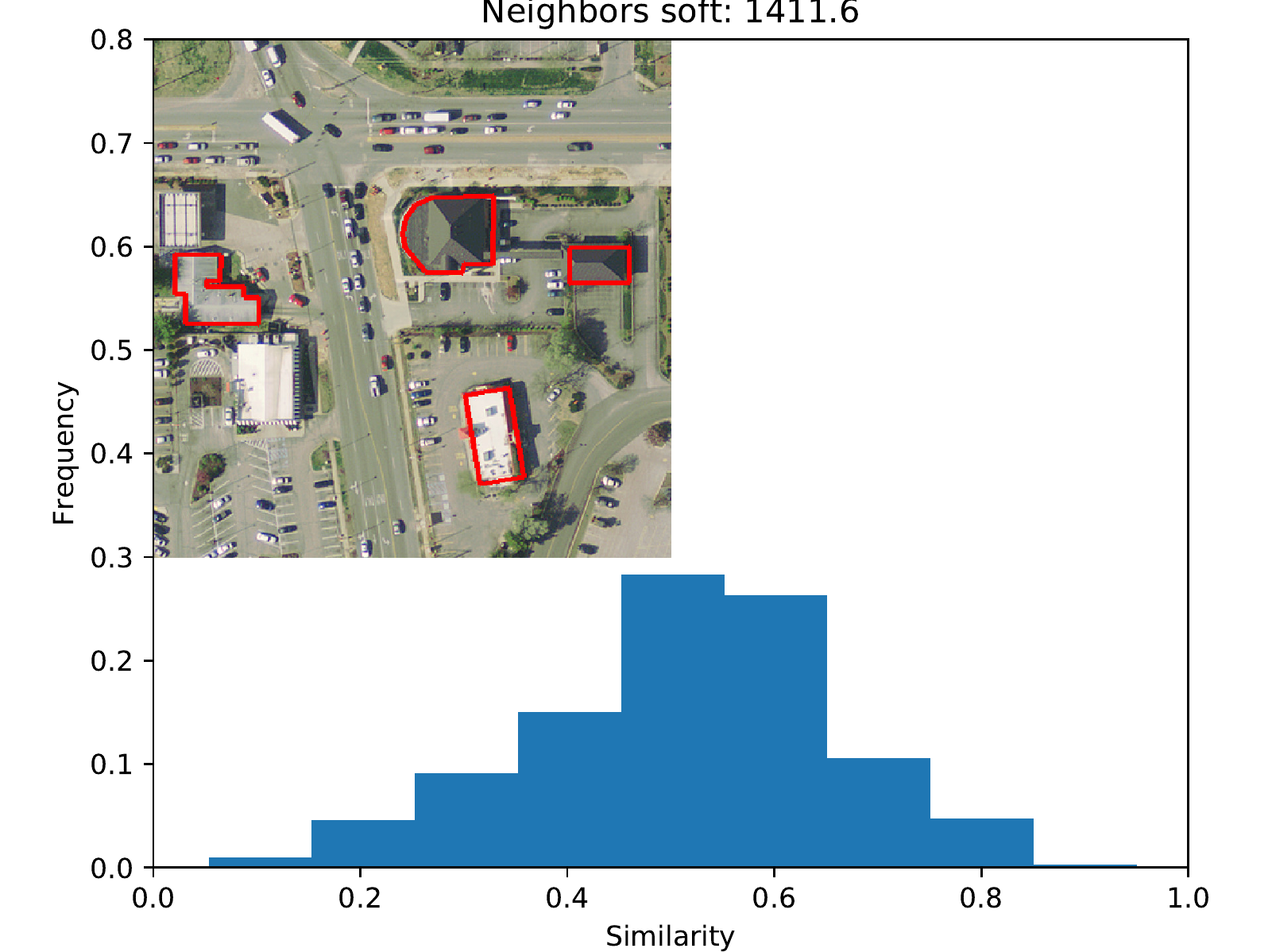}
		\includegraphics[width=0.7\linewidth]{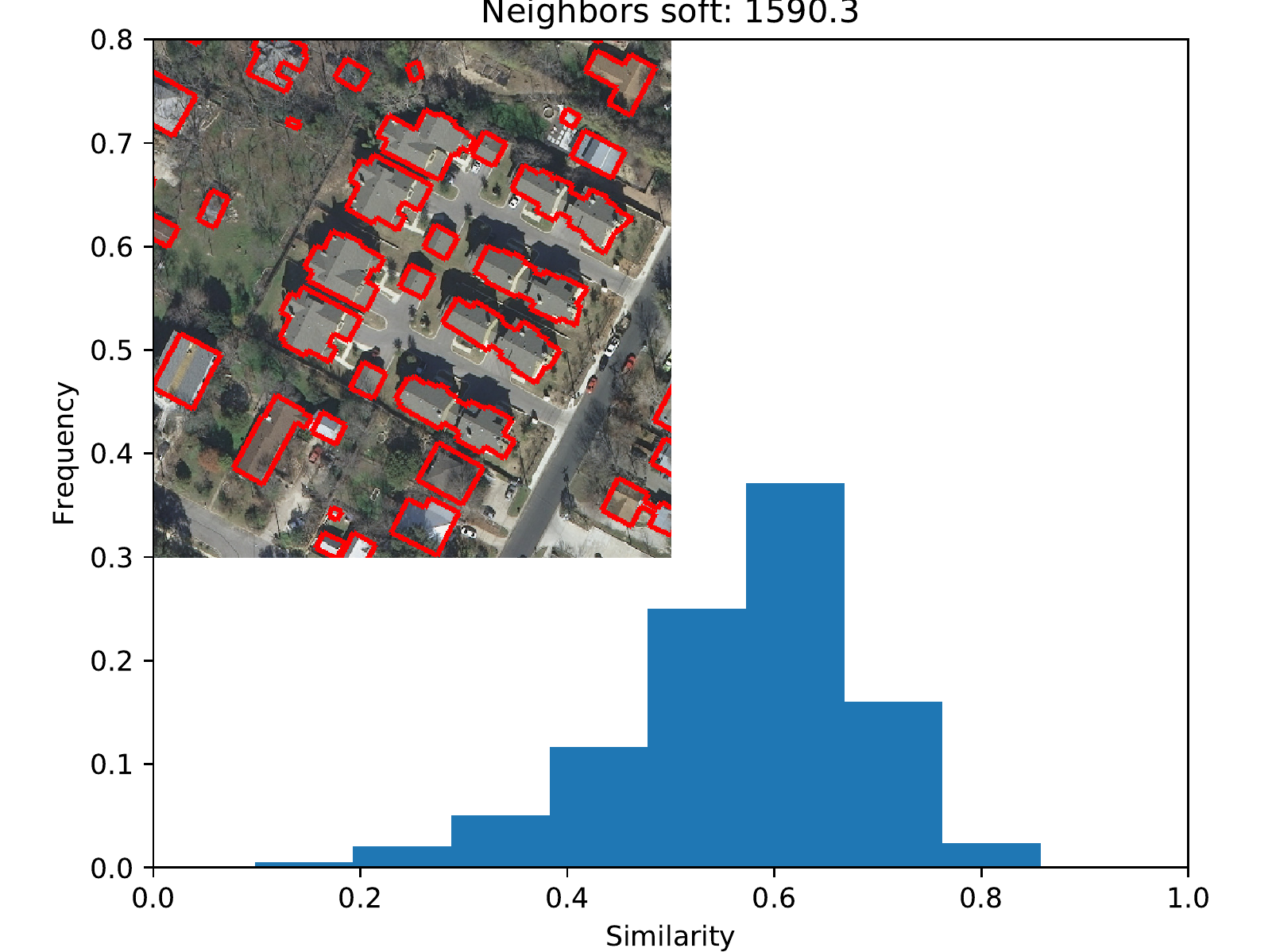}
		\includegraphics[width=0.7\linewidth]{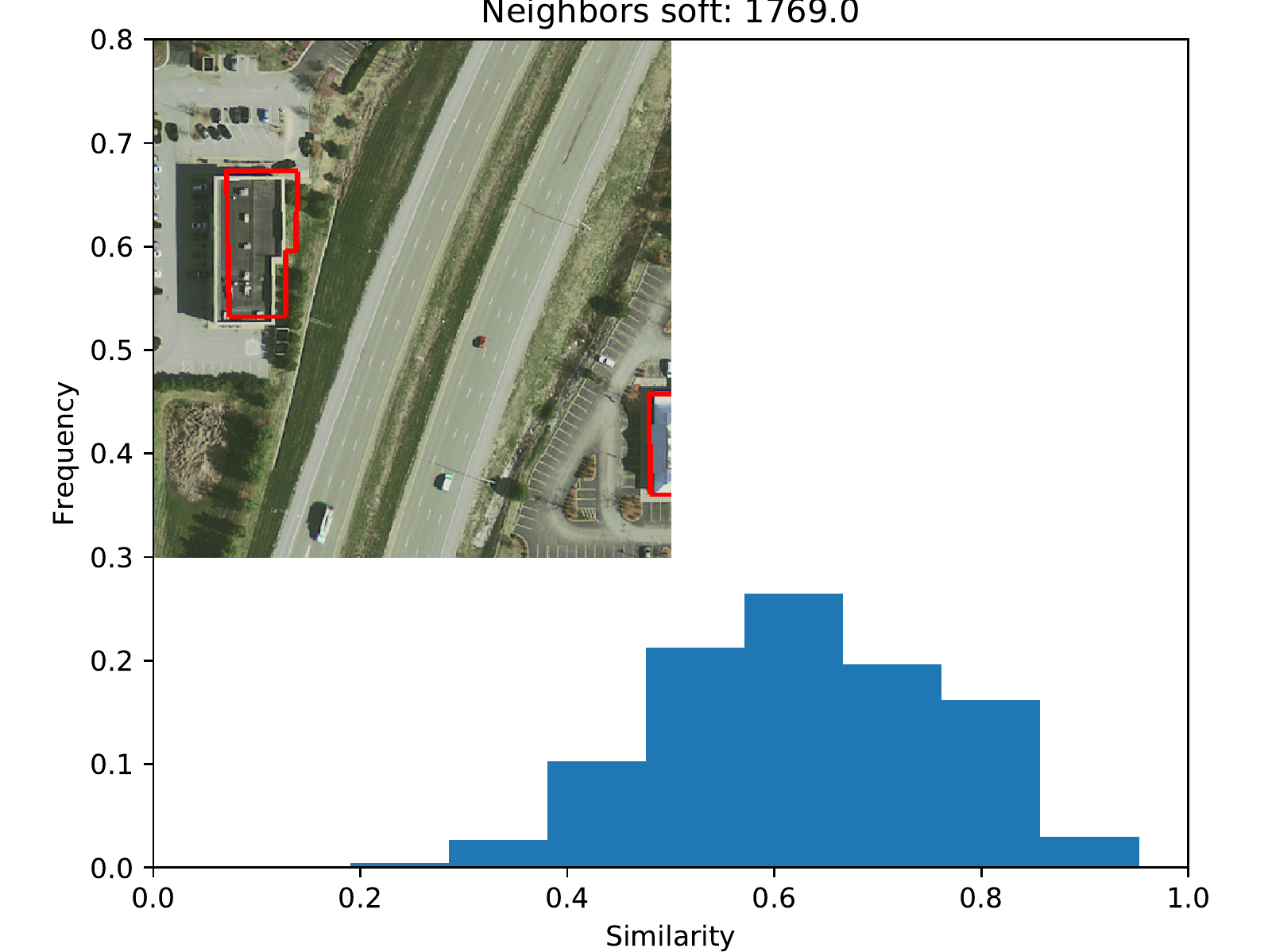}
		\includegraphics[width=0.7\linewidth]{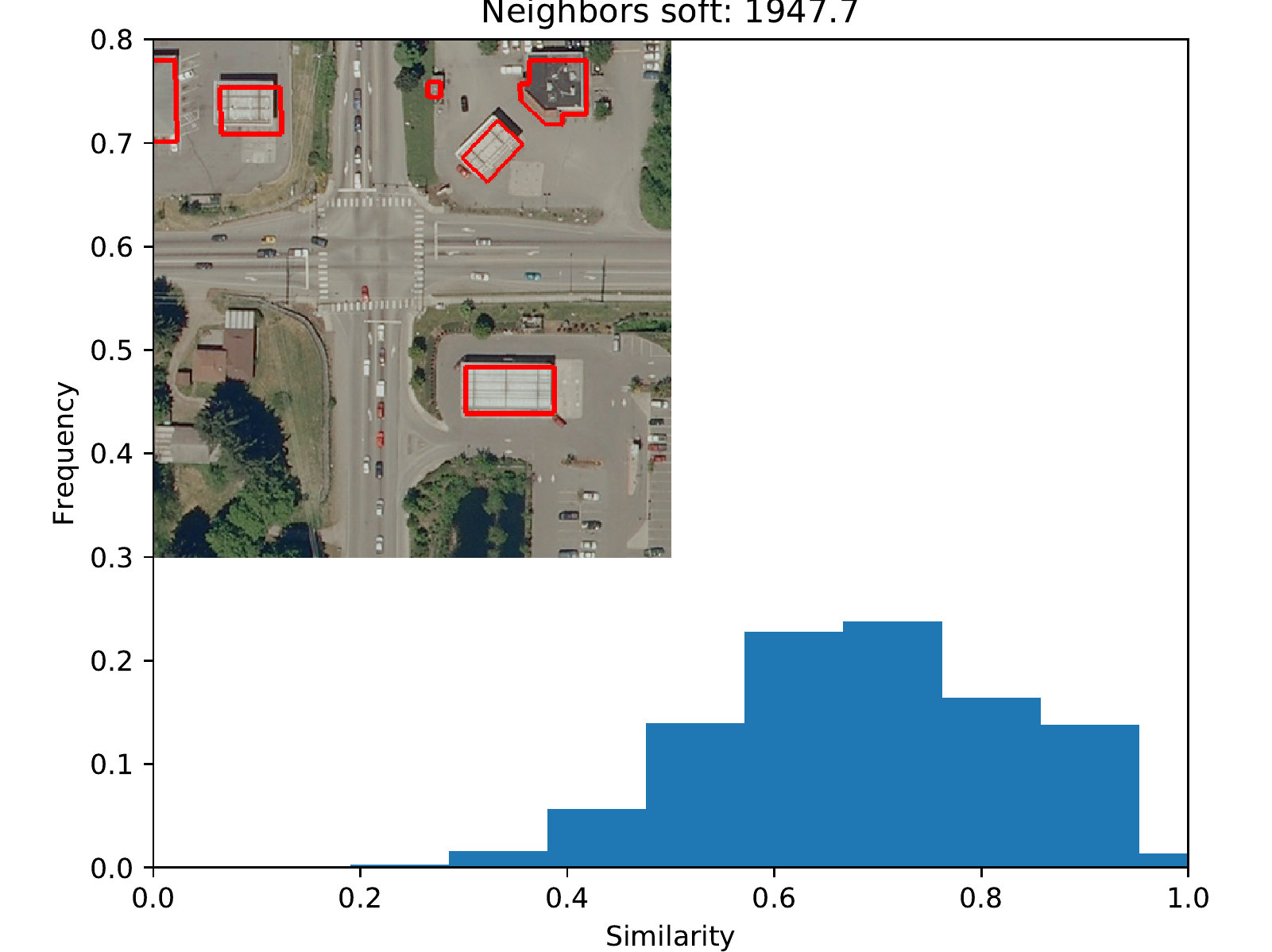}
	\end{subfigure}
	\begin{subfigure}[b]{0.3\textwidth}
		\centering
		\caption{Round 2}
		\includegraphics[width=0.7\linewidth]{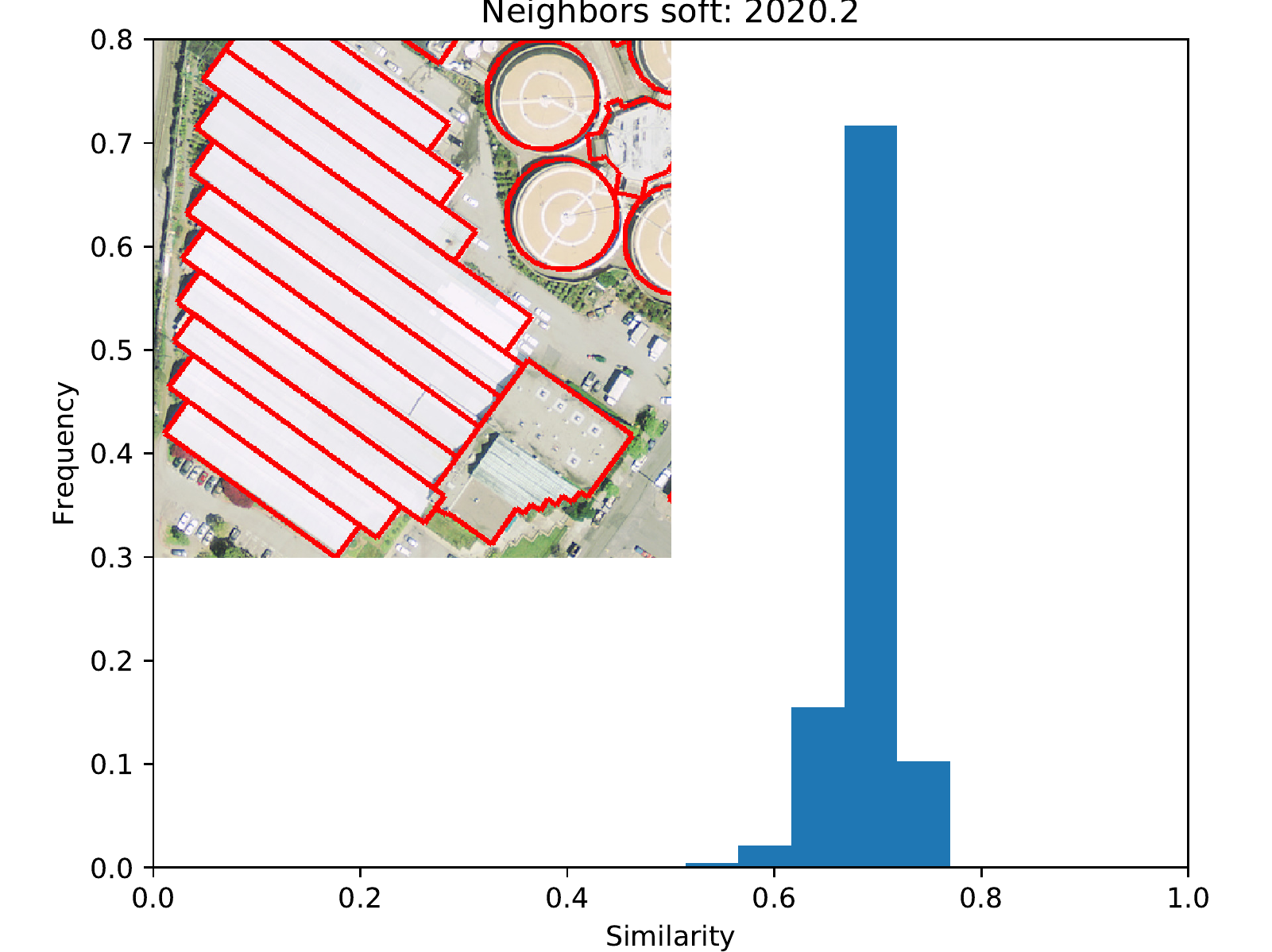}
		\includegraphics[width=0.7\linewidth]{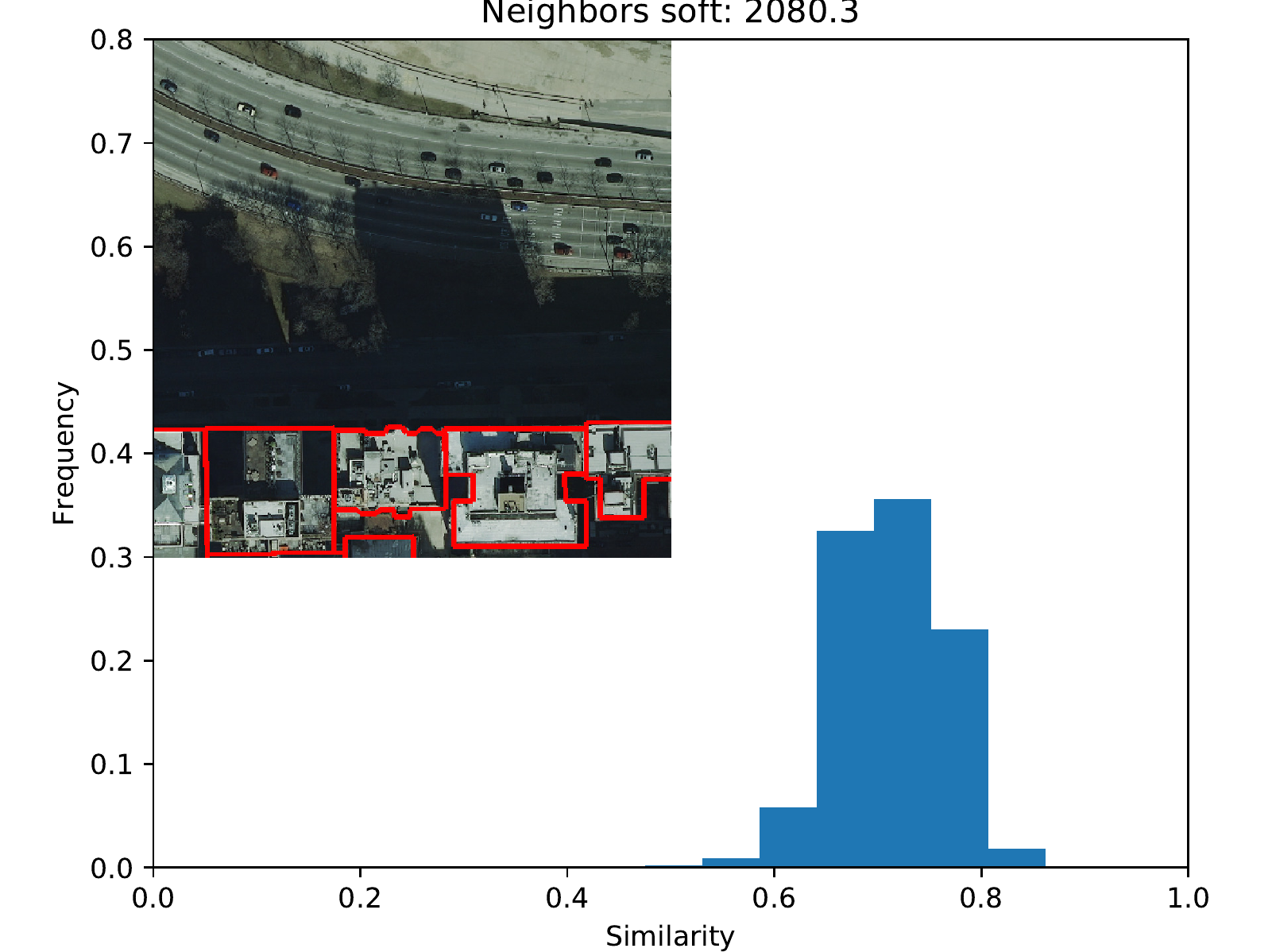}
		\includegraphics[width=0.7\linewidth]{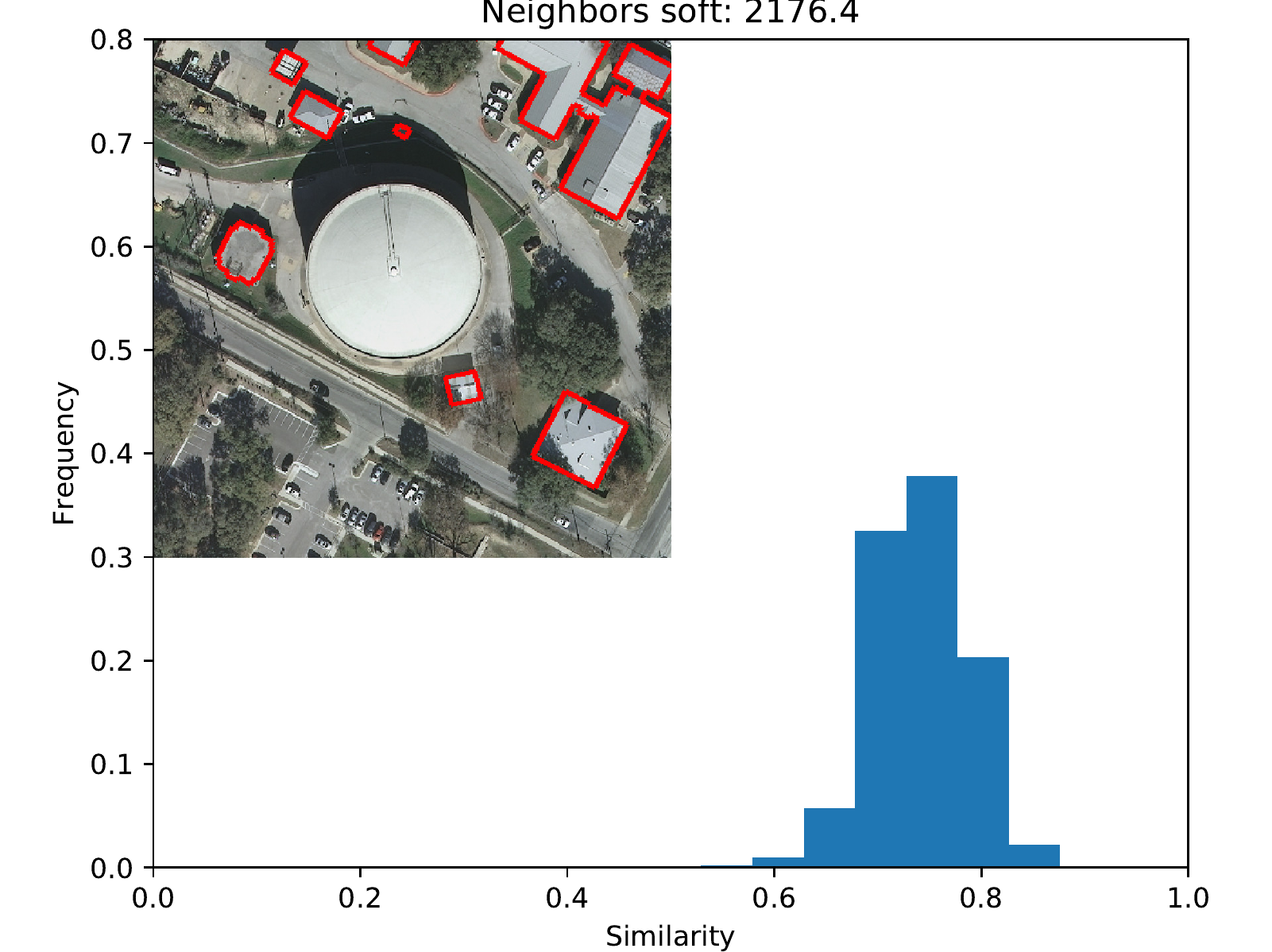}
		\includegraphics[width=0.7\linewidth]{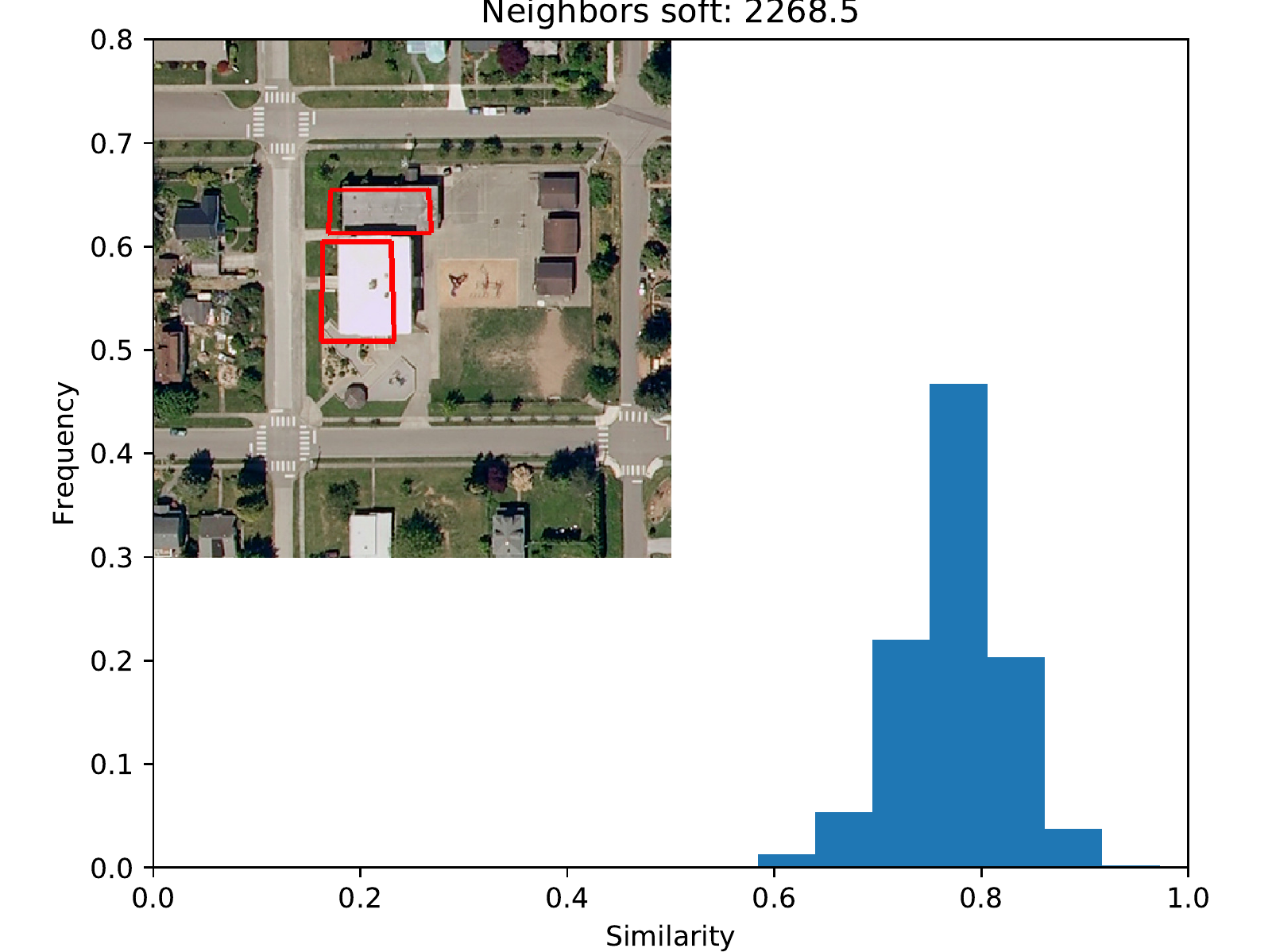}
		\includegraphics[width=0.7\linewidth]{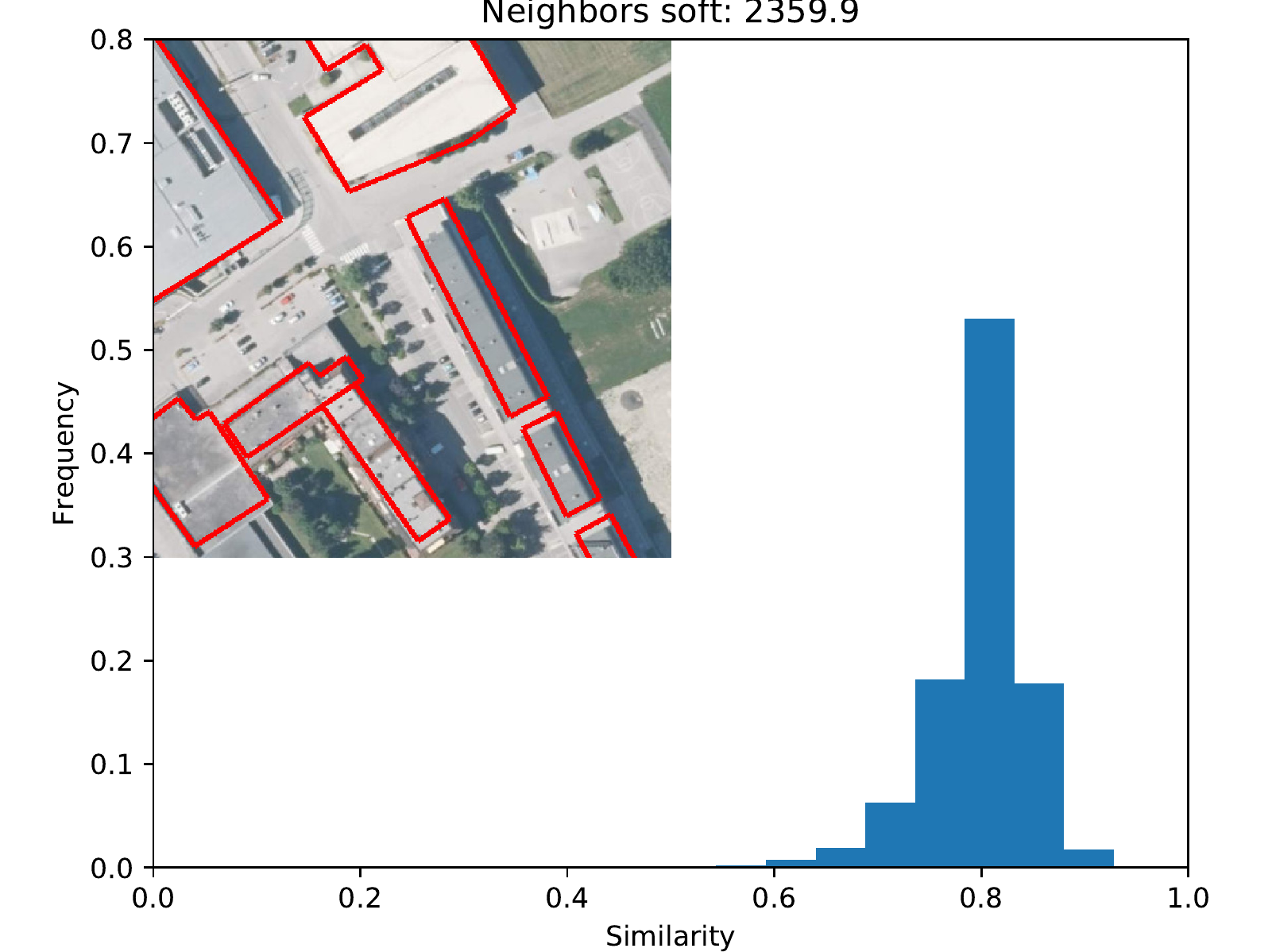}
		\includegraphics[width=0.7\linewidth]{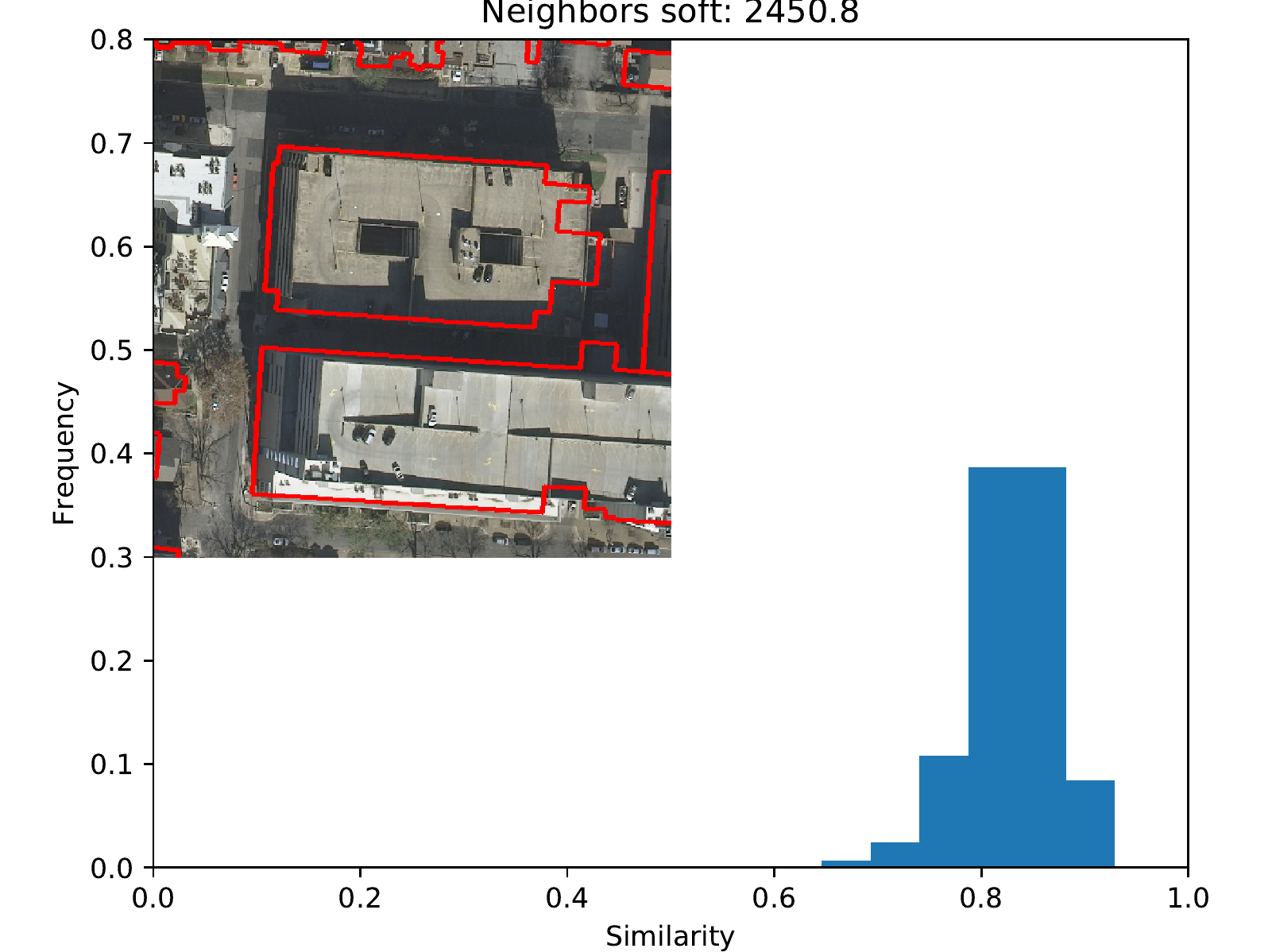}
		\includegraphics[width=0.7\linewidth]{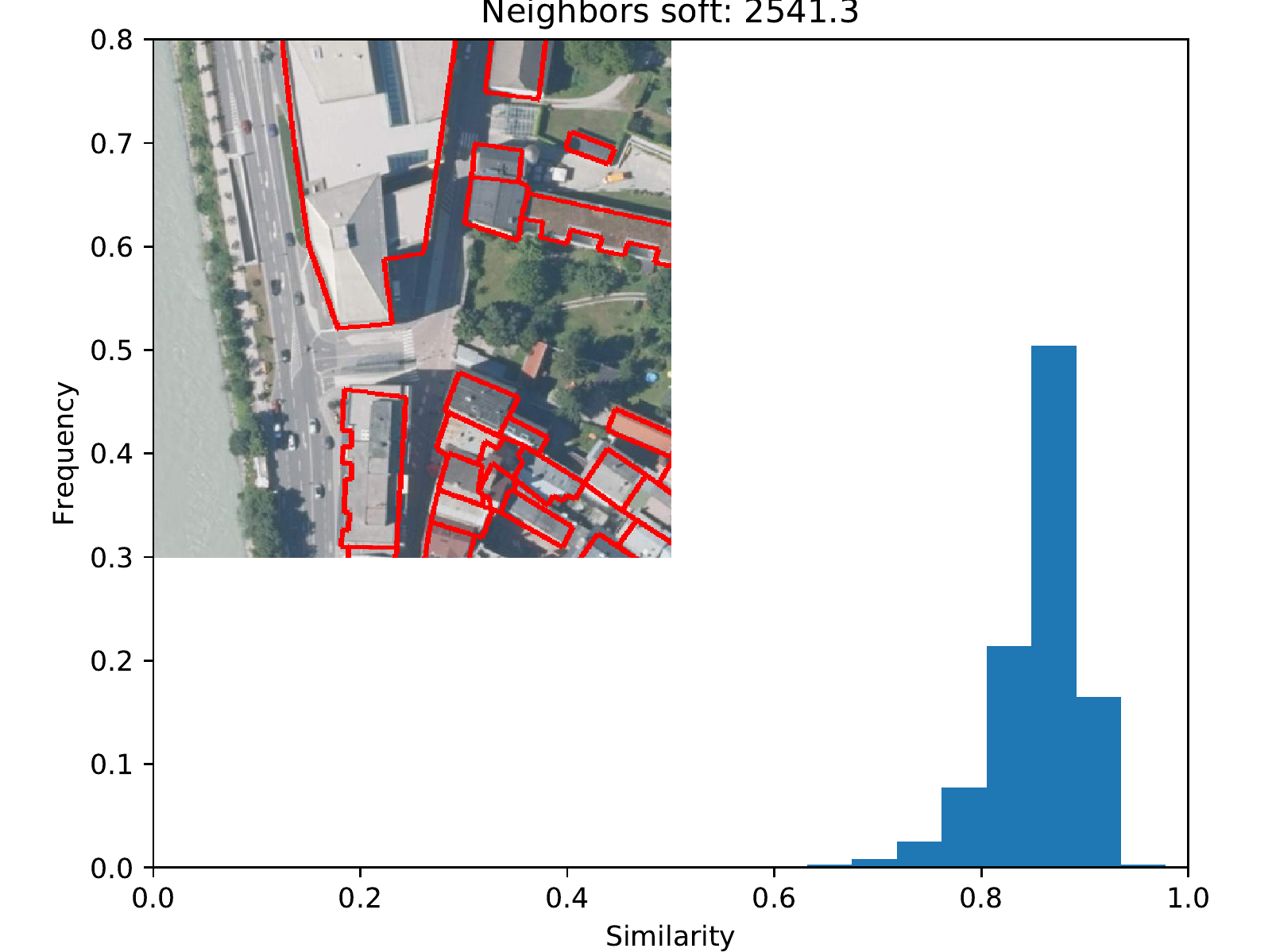}
		\includegraphics[width=0.7\linewidth]{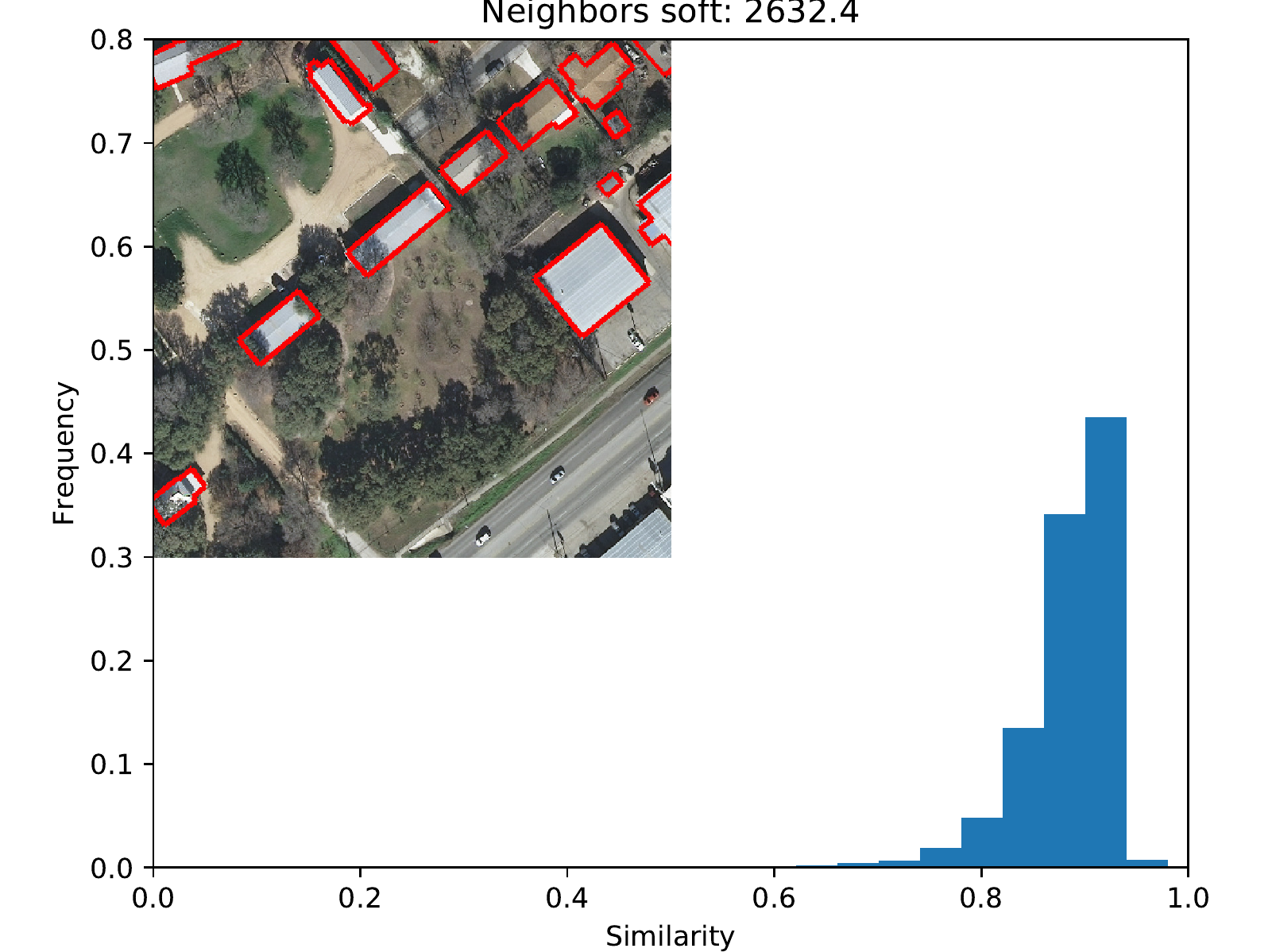}
		\includegraphics[width=0.7\linewidth]{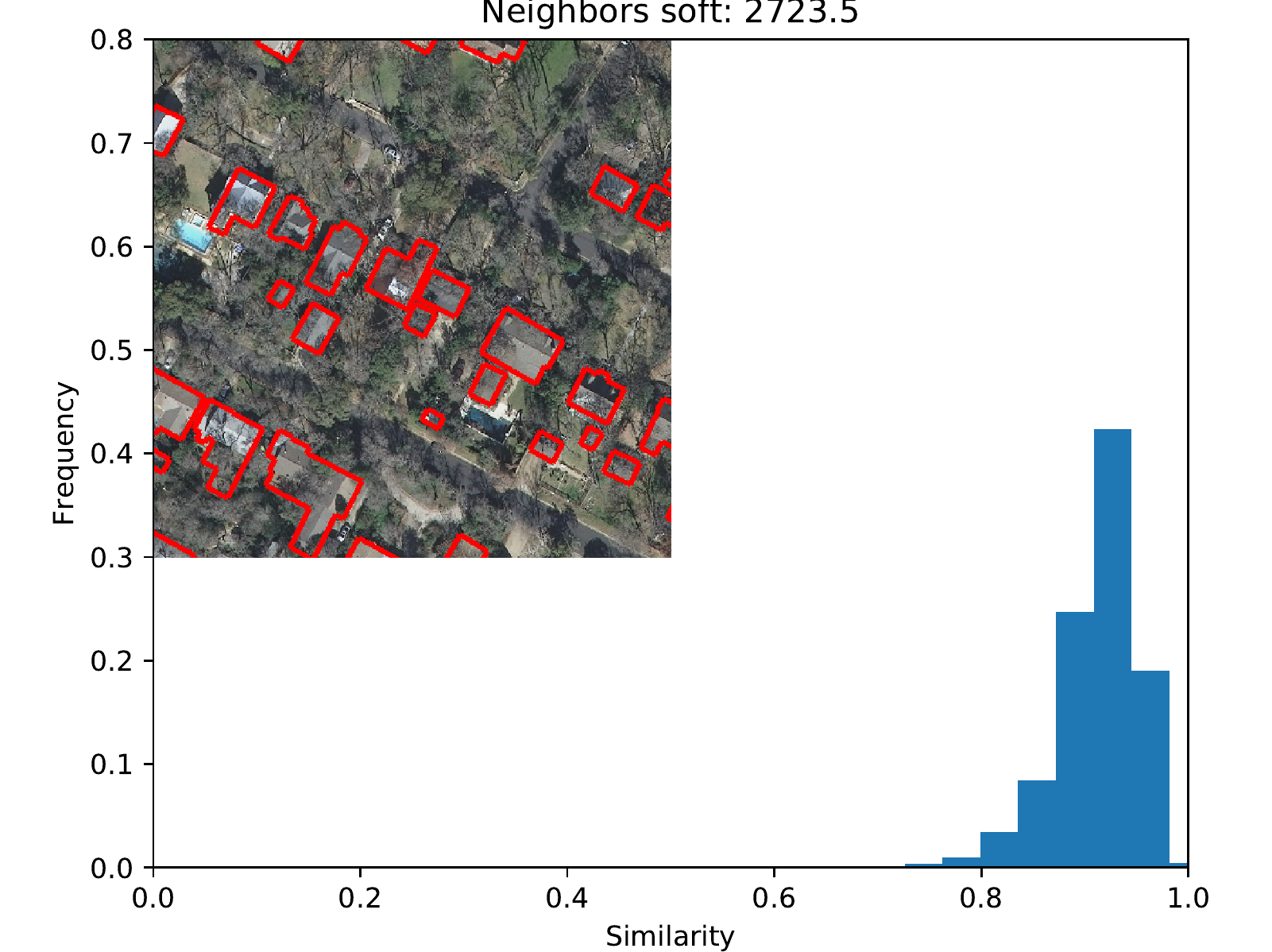}
		\includegraphics[width=0.7\linewidth]{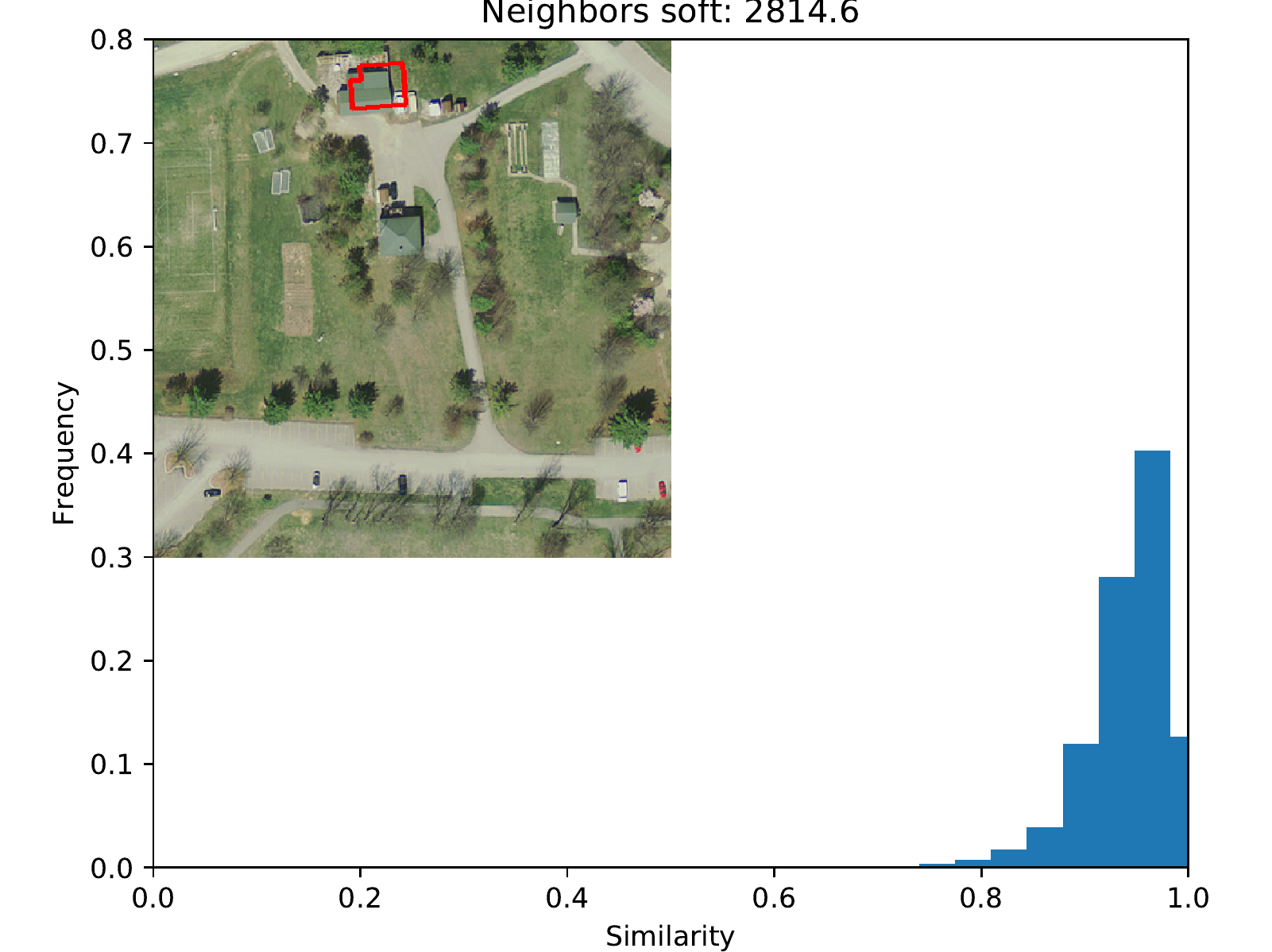}
	\end{subfigure}
	\begin{subfigure}[b]{0.3\textwidth}
		\centering
		\caption{Round 3}
		\includegraphics[width=0.7\linewidth]{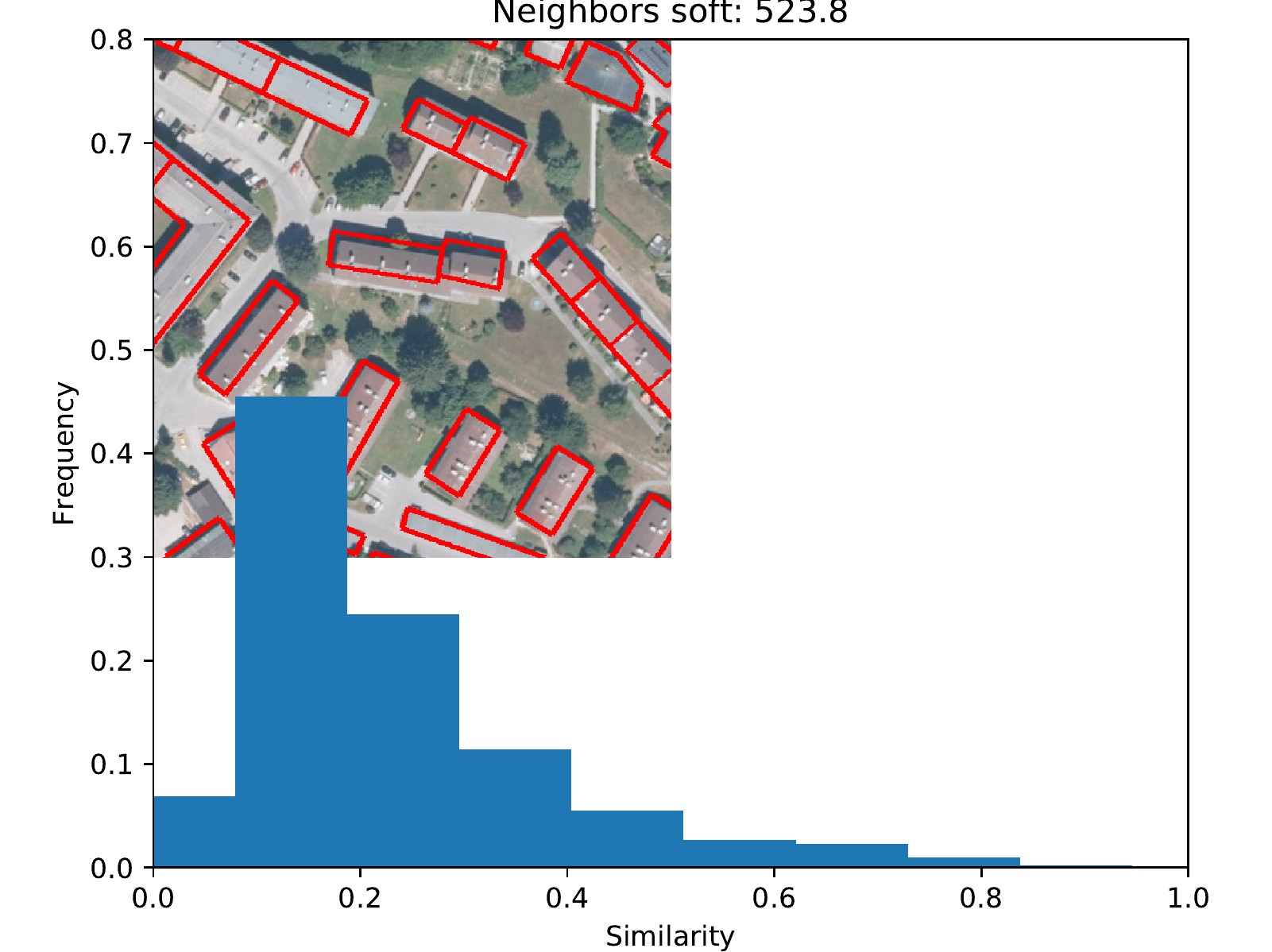}
		\includegraphics[width=0.7\linewidth]{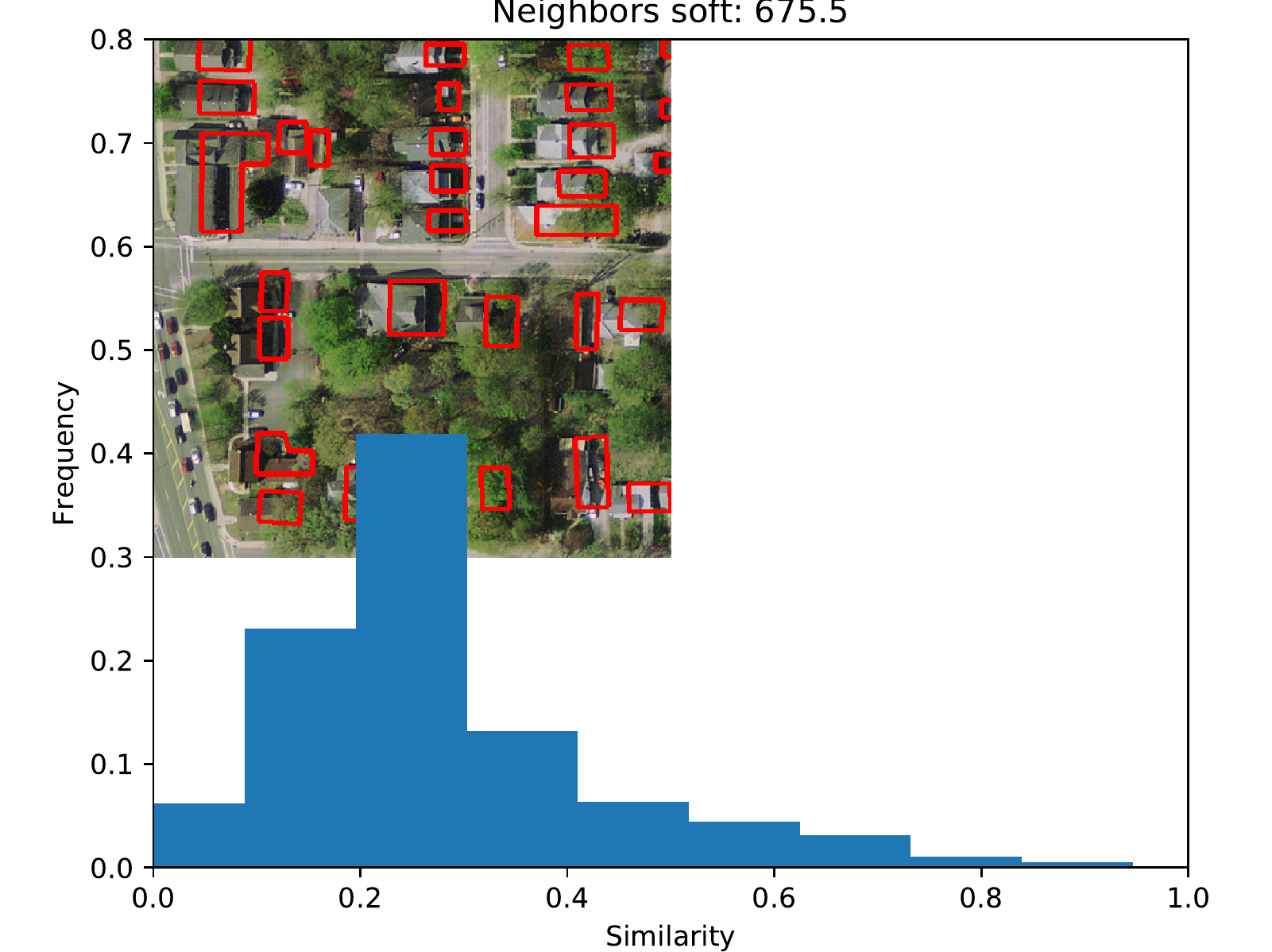}
		\includegraphics[width=0.7\linewidth]{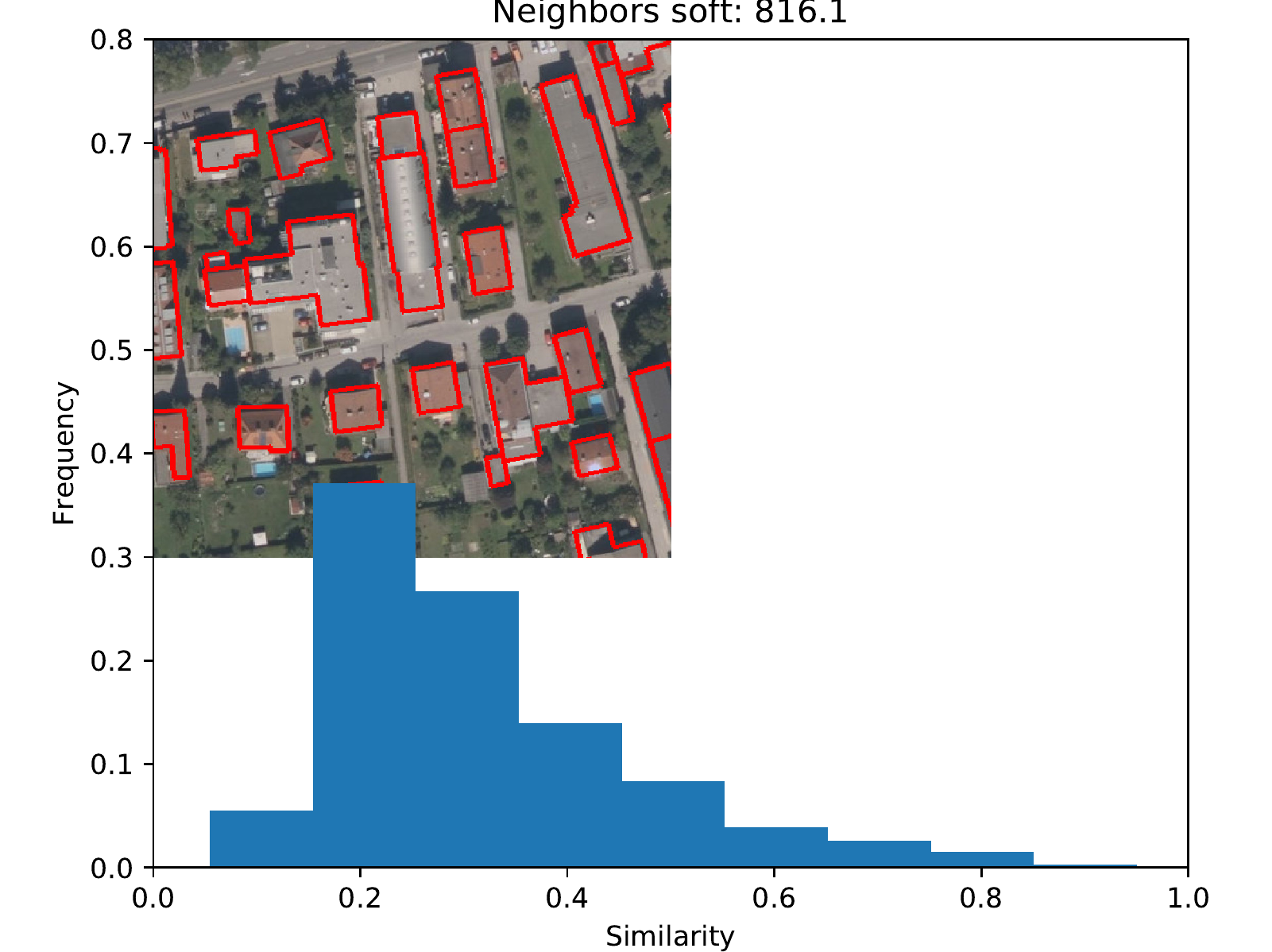}
		\includegraphics[width=0.7\linewidth]{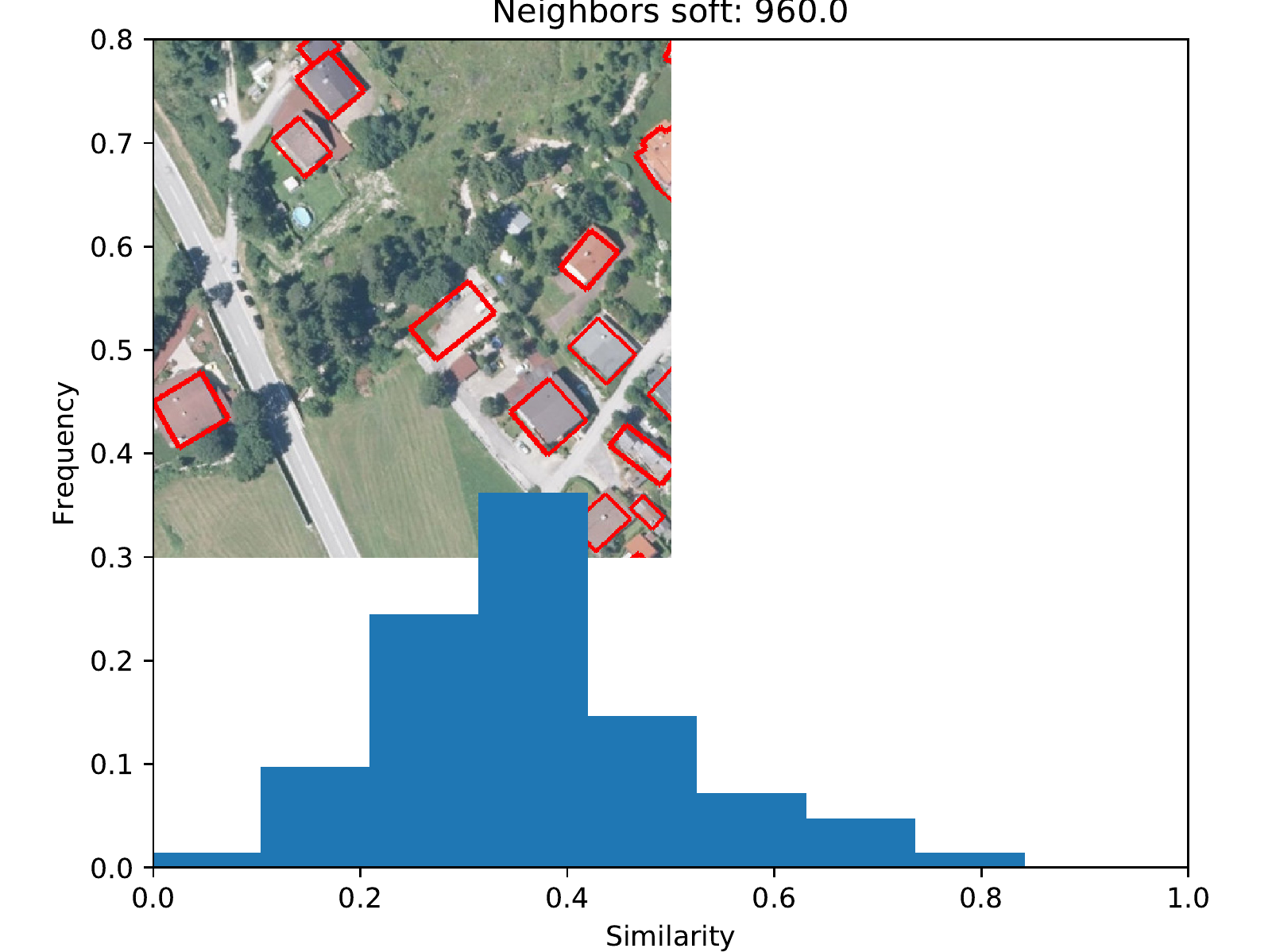}
		\includegraphics[width=0.7\linewidth]{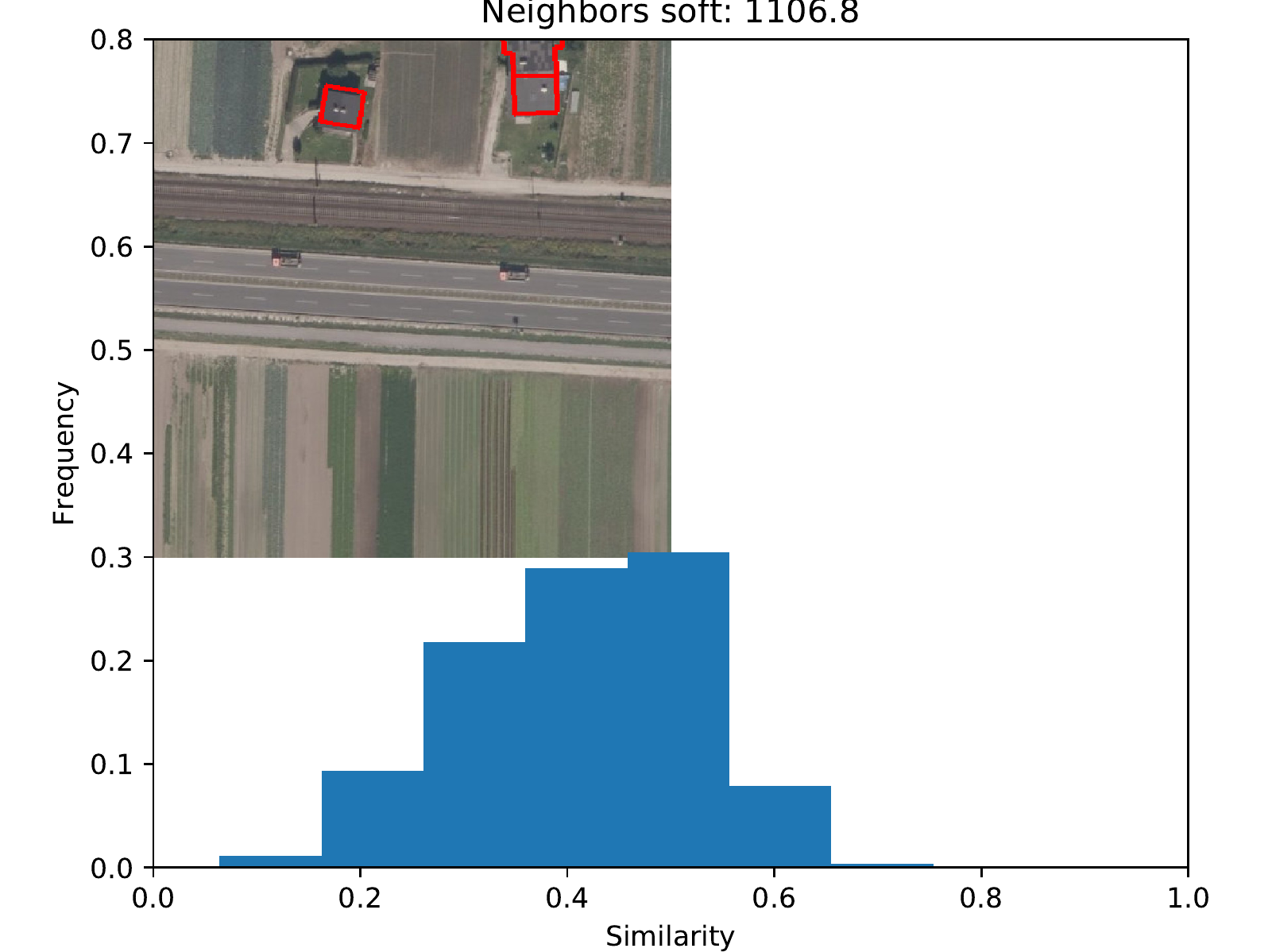}
		\includegraphics[width=0.7\linewidth]{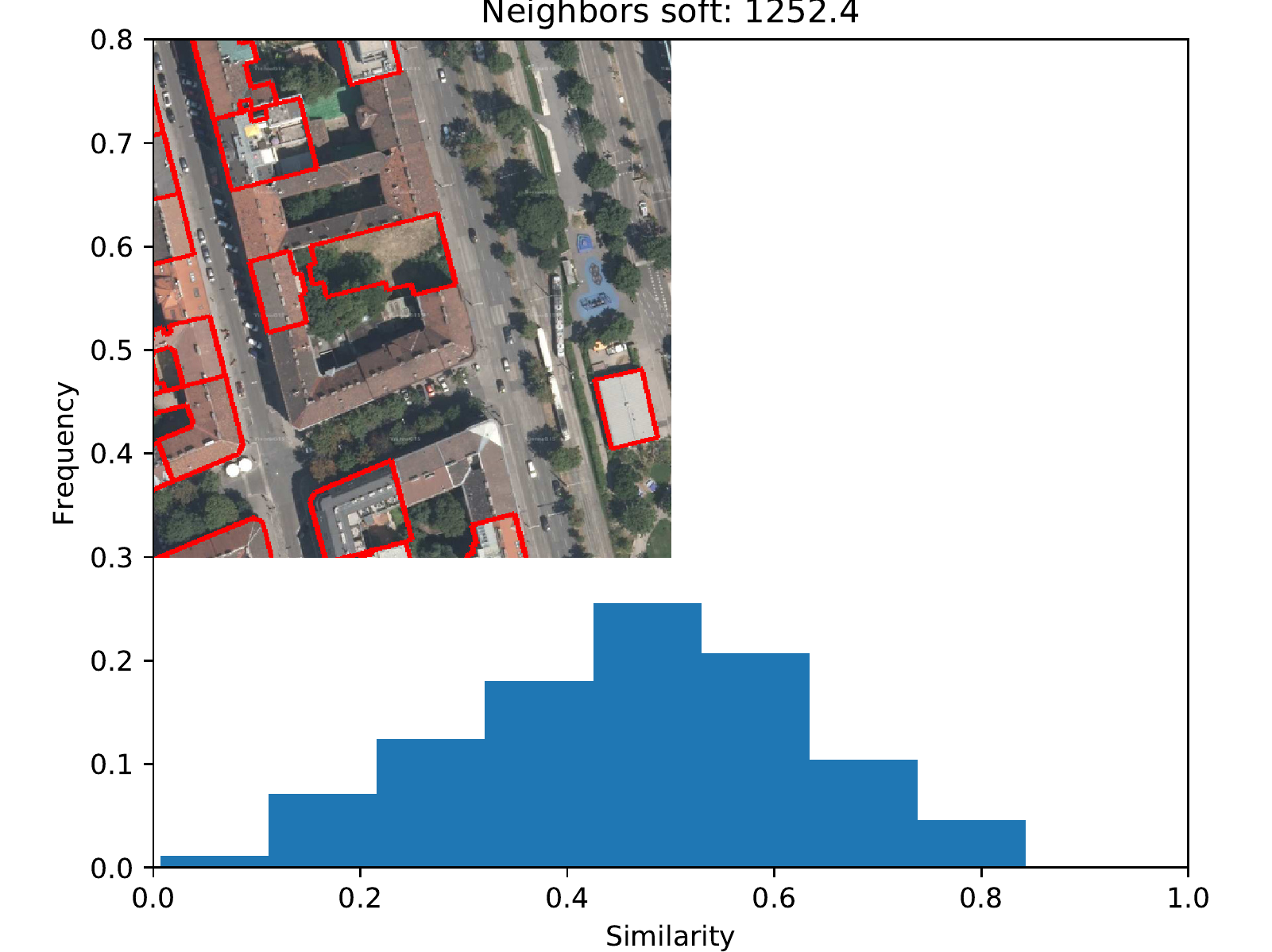}
		\includegraphics[width=0.7\linewidth]{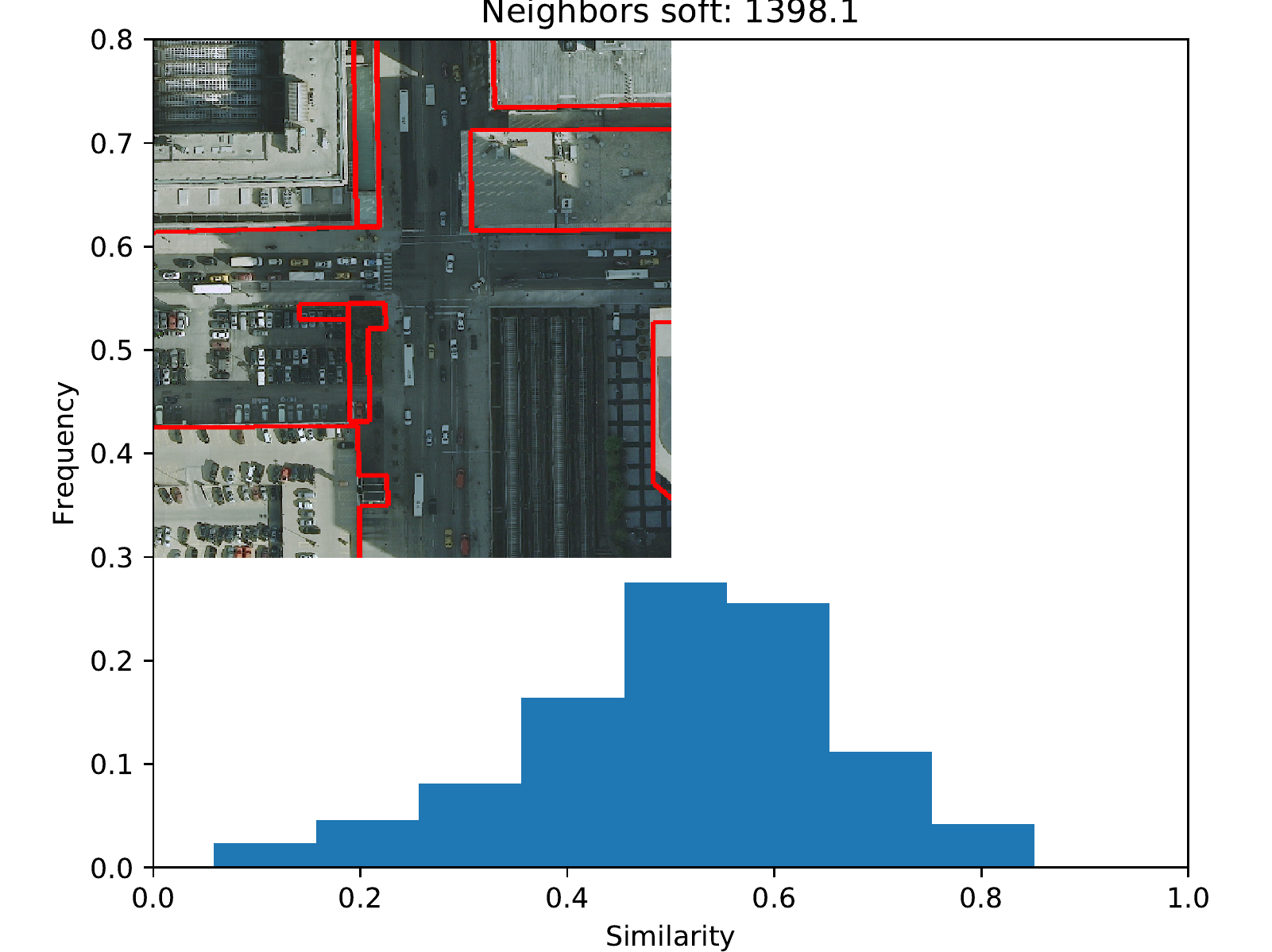}
		\includegraphics[width=0.7\linewidth]{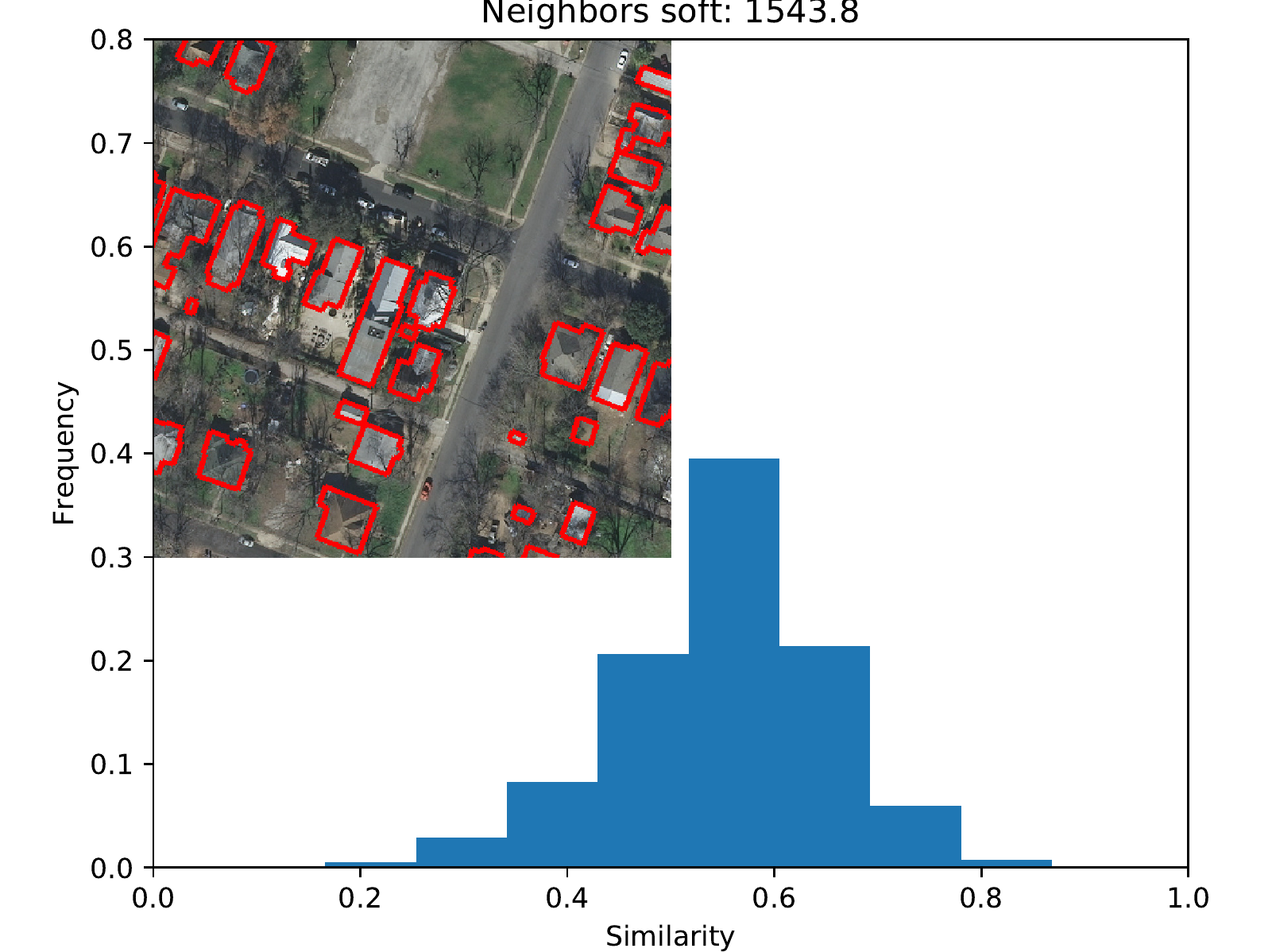}
		\includegraphics[width=0.7\linewidth]{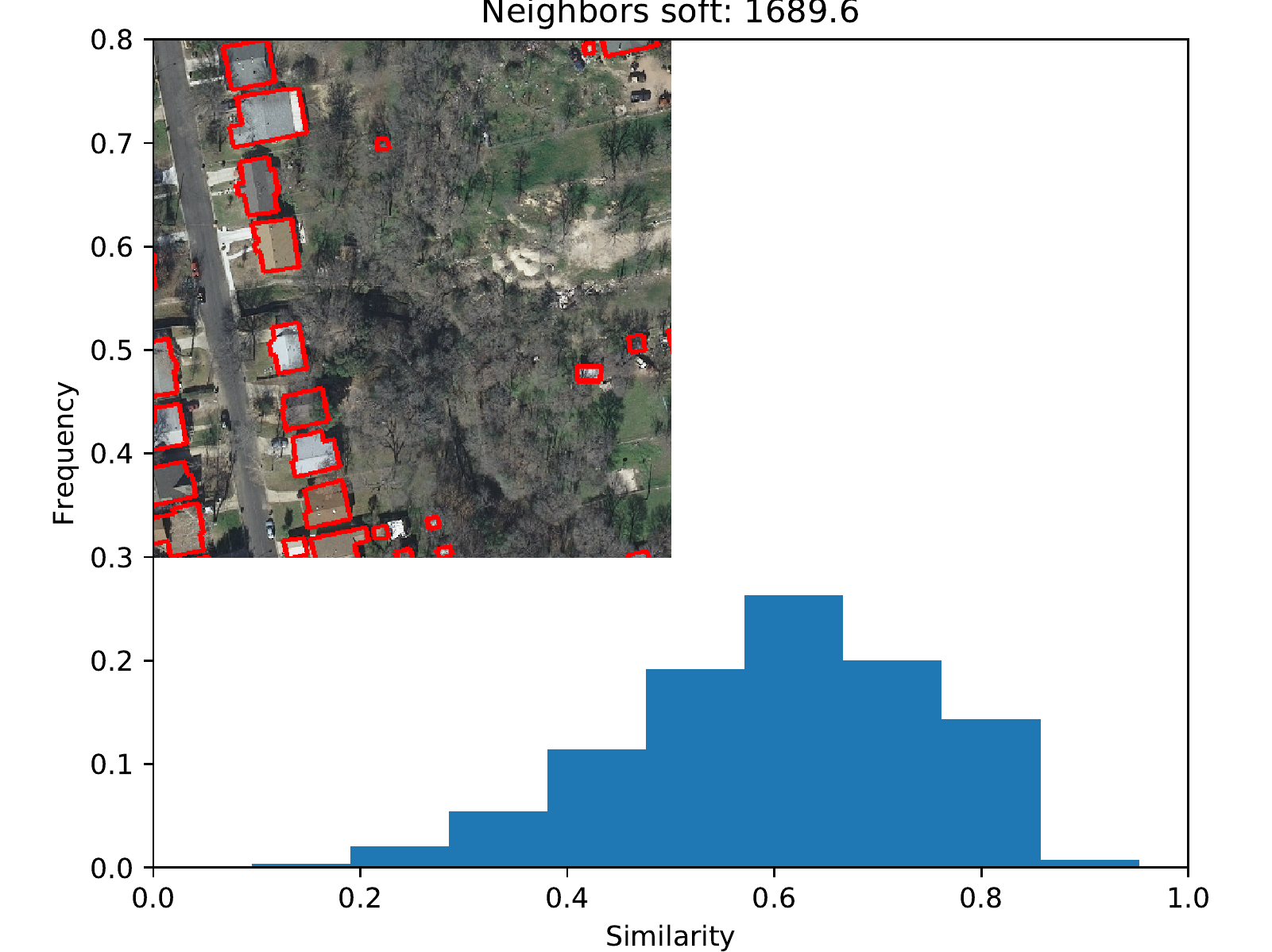}
		\includegraphics[width=0.7\linewidth]{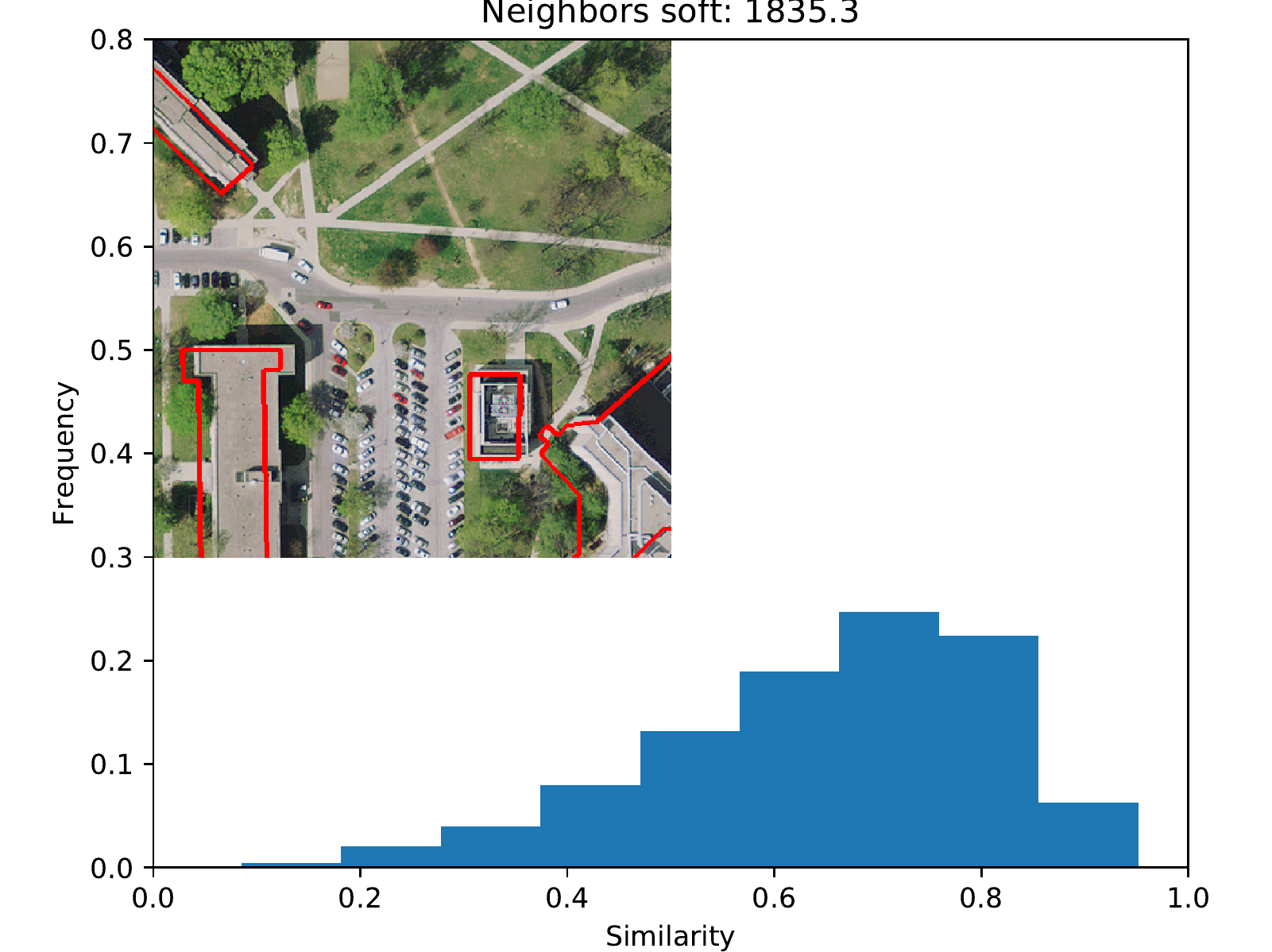}
	\end{subfigure}
	\caption{Histograms of similarities shown for the same 10 patches as Fig.\ref{fig:overall_hist} and Fig.\ref{fig:round_0_overall_hist_k_nearest}, \ref{fig:round_1_overall_hist_k_nearest}, \ref{fig:round_2_overall_hist_k_nearest}.}
	\label{fig:overall_hist_individual_hist}
\end{figure}

\begin{figure}
	\centering
	\includegraphics[width=\linewidth]{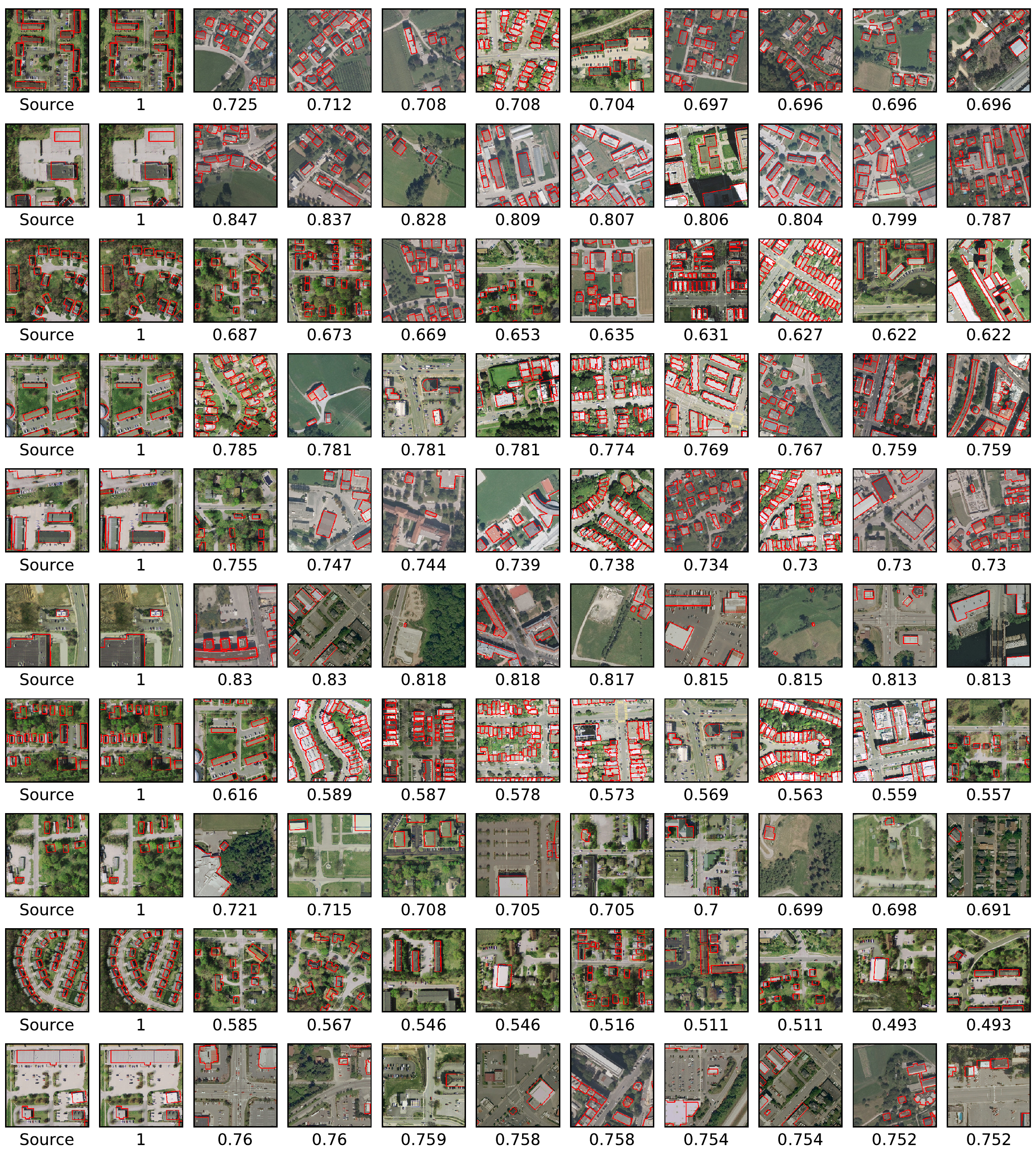}
	\caption{\textbf{Round 1}: k-nearest neighbors with k=10. The 10 patches are from from the bloomington22 image. Same patch selection across rounds.}
	\label{fig:round_0_bloomington22_k_nearest}
\end{figure}
\begin{figure}
	\centering
	\includegraphics[width=\linewidth]{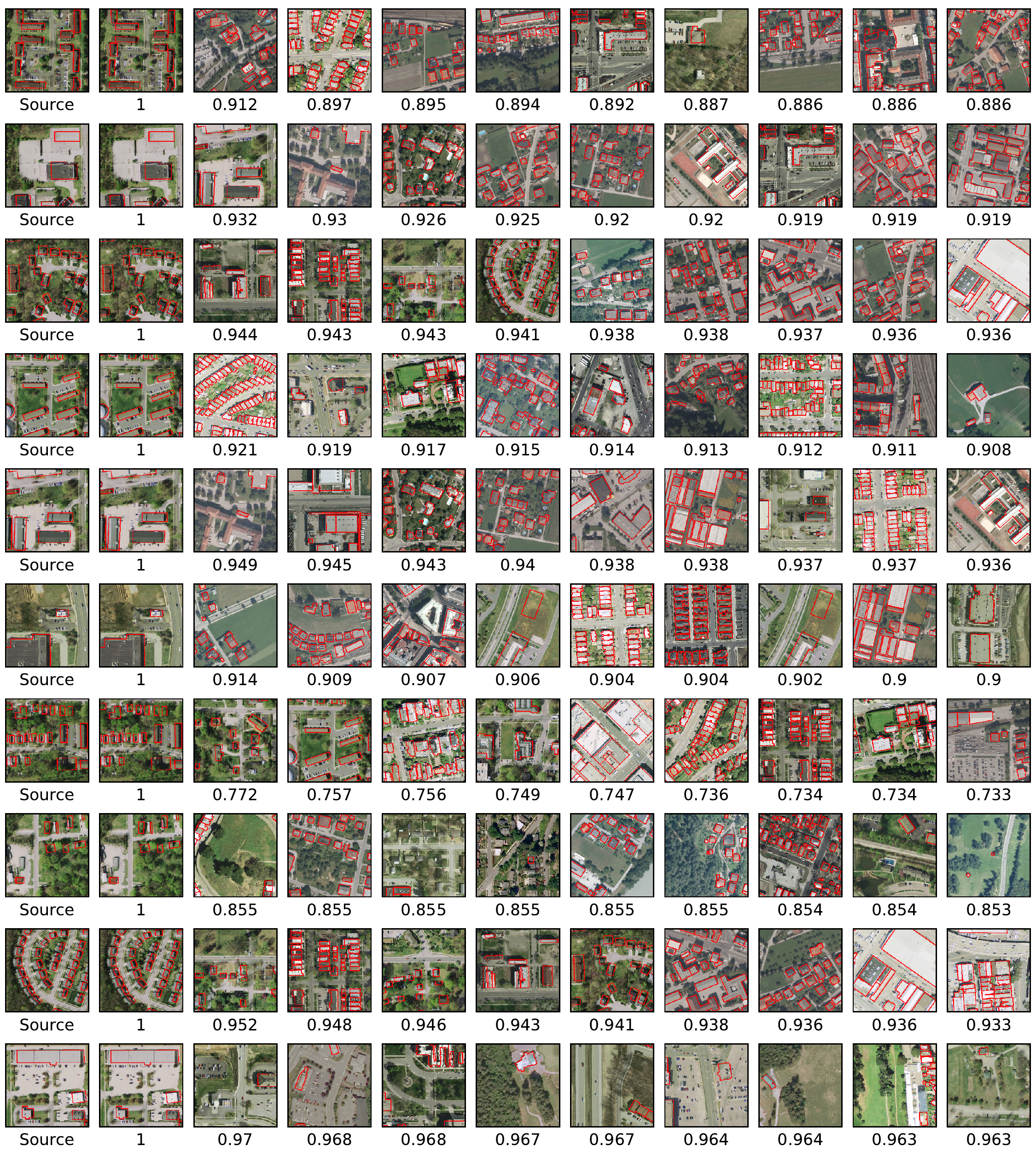}
	\caption{\textbf{Round 2}:k-nearest neighbors with k=10. The 10 patches are from from the bloomington22 image. Same patch selection across rounds.}
	\label{fig:round_1_bloomington22_k_nearest}
\end{figure}
\begin{figure}
	\centering
	\includegraphics[width=\linewidth]{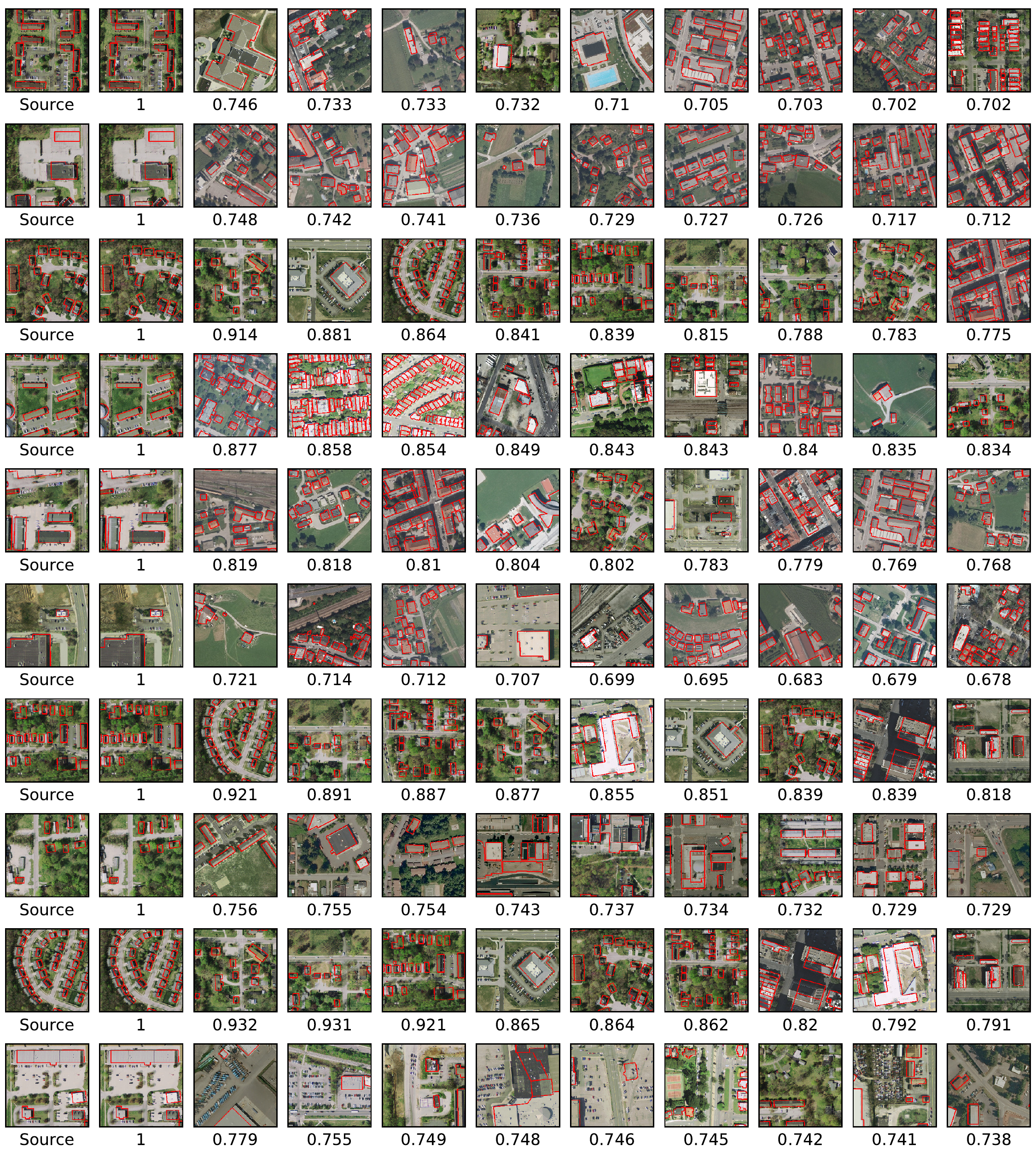}
	\caption{\textbf{Round 3}: k-nearest neighbors with k=10. The 10 patches are from from the bloomington22 image. Same patch selection across rounds.}
	\label{fig:round_2_bloomington22_k_nearest}
\end{figure}

\begin{figure}
	\centering
	\begin{subfigure}[b]{0.3\textwidth}
		\centering
		\caption{Round 1}
		\includegraphics[width=0.7\linewidth]{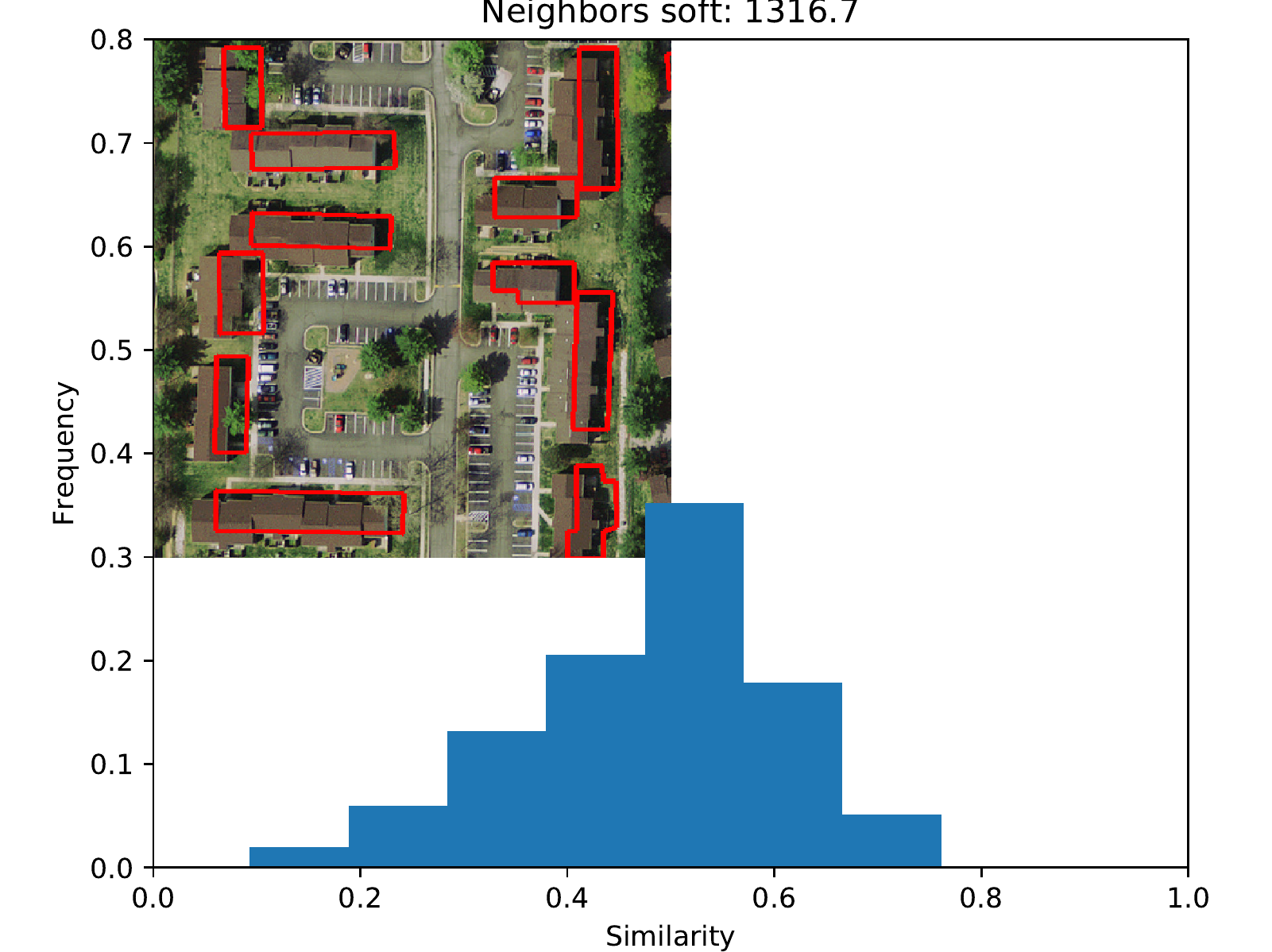}
		\includegraphics[width=0.7\linewidth]{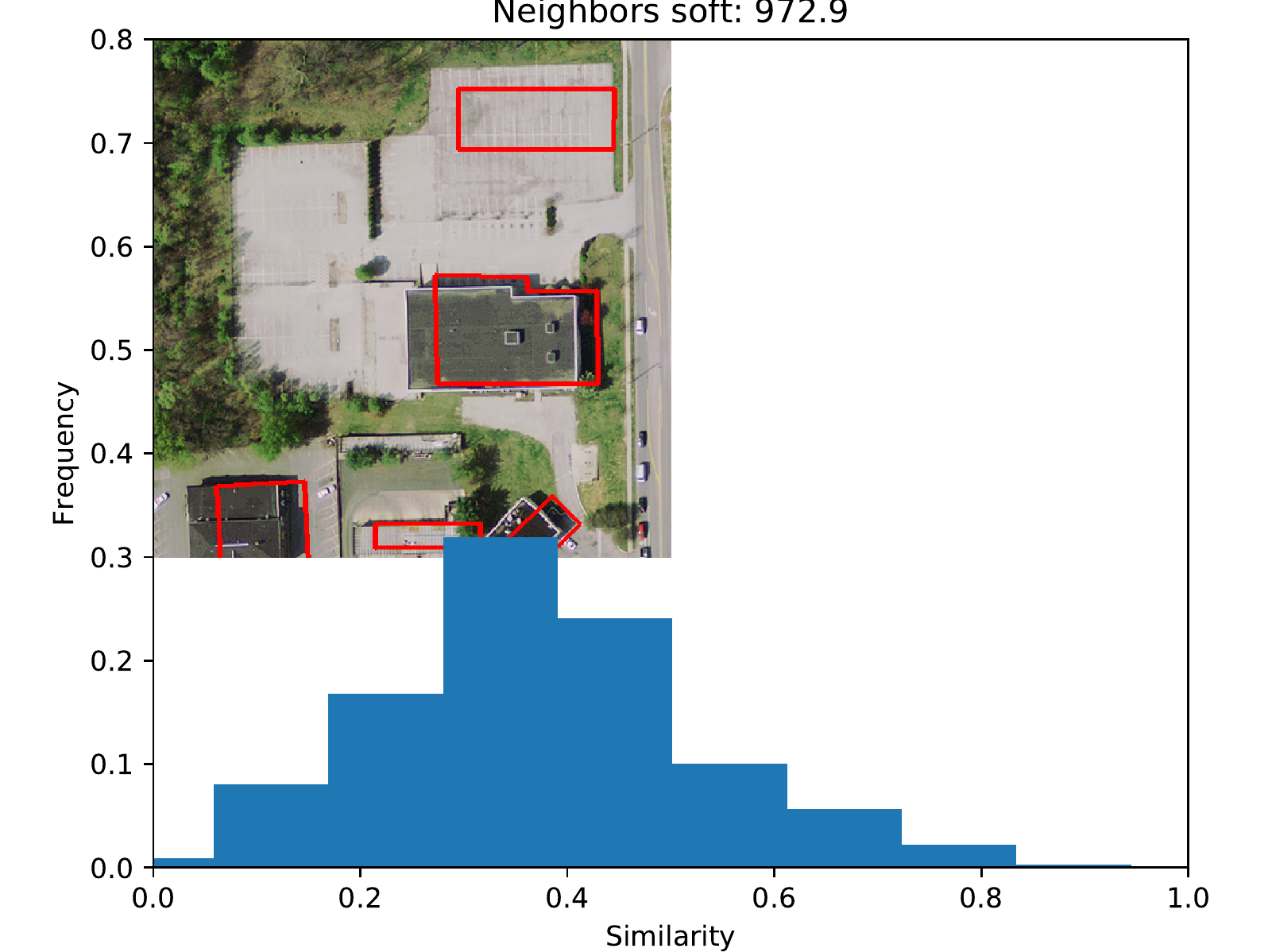}
		\includegraphics[width=0.7\linewidth]{netsimilarity_ds_fac_4_round_0_bloomington22_individual_hist_02-eps-converted-to.pdf}
		\includegraphics[width=0.7\linewidth]{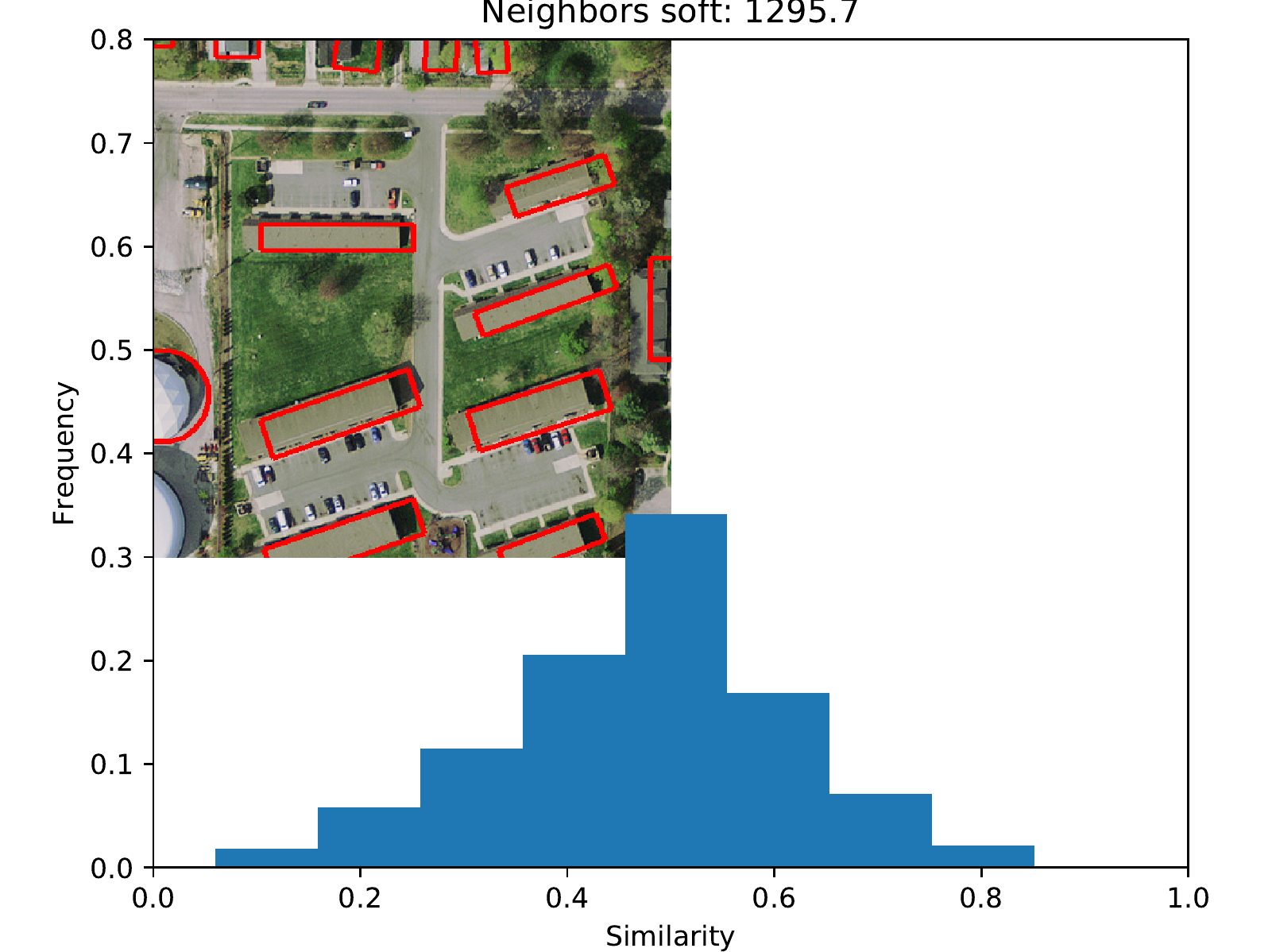}
		\includegraphics[width=0.7\linewidth]{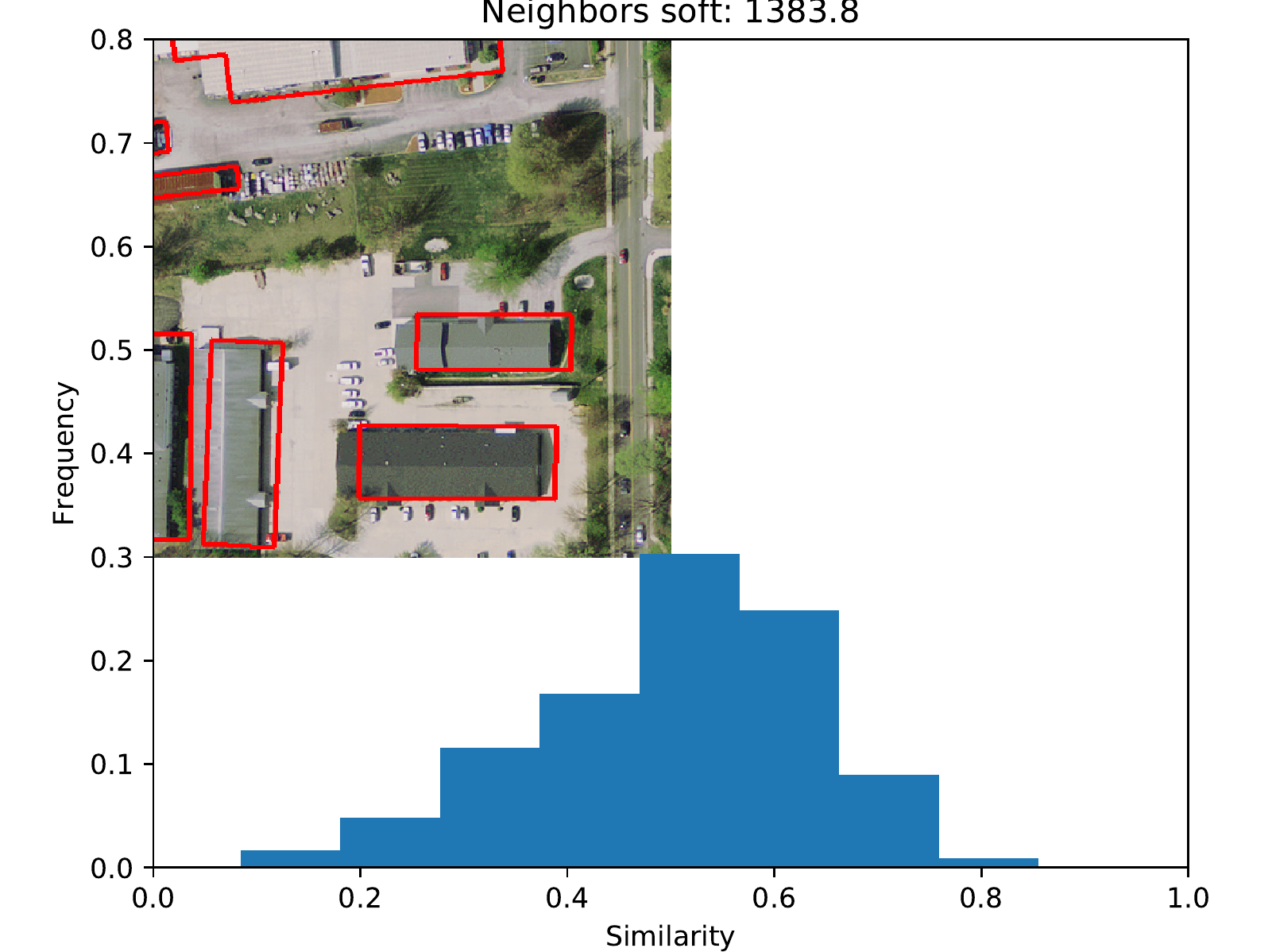}
		\includegraphics[width=0.7\linewidth]{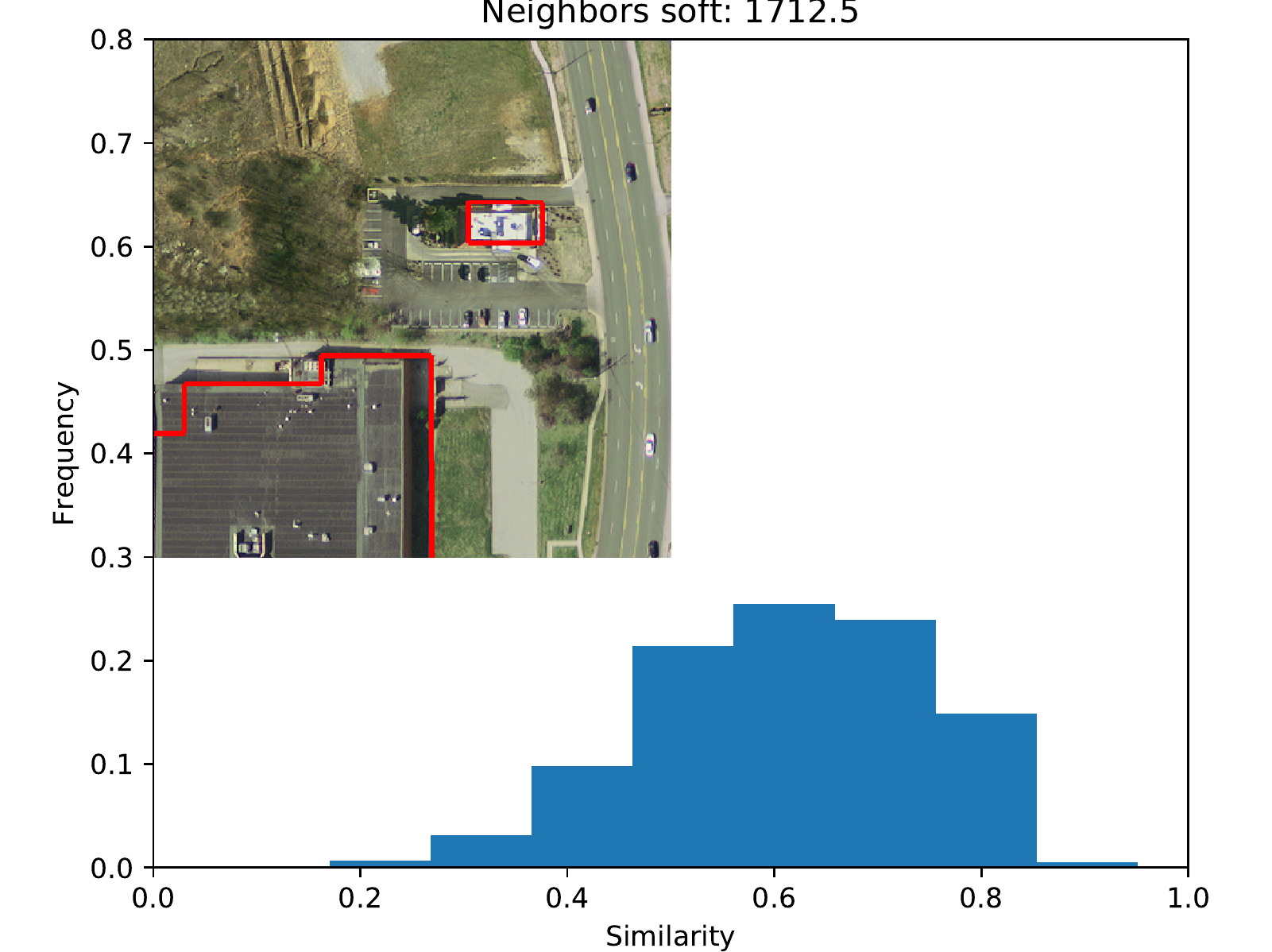}
		\includegraphics[width=0.7\linewidth]{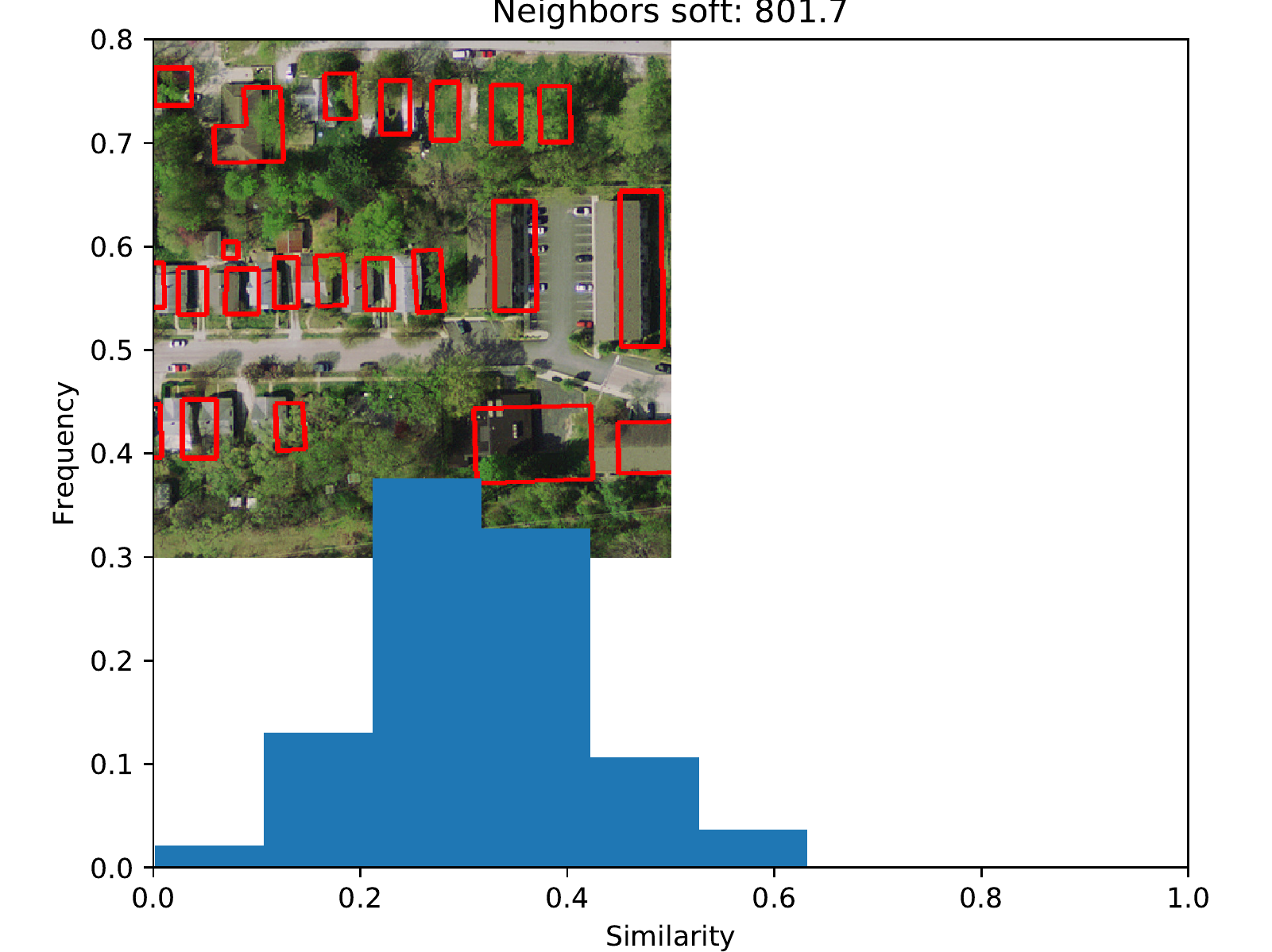}
		\includegraphics[width=0.7\linewidth]{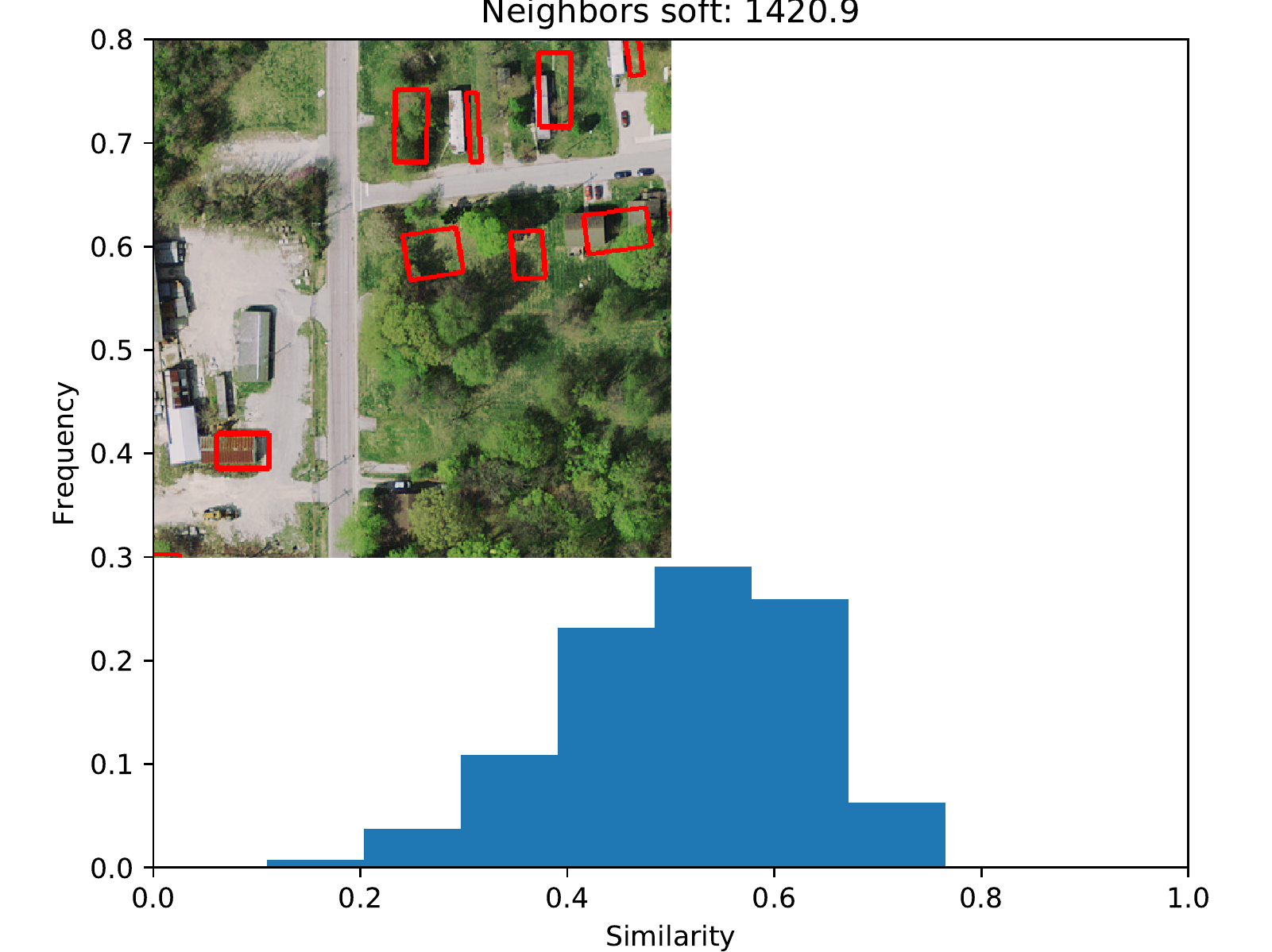}
		\includegraphics[width=0.7\linewidth]{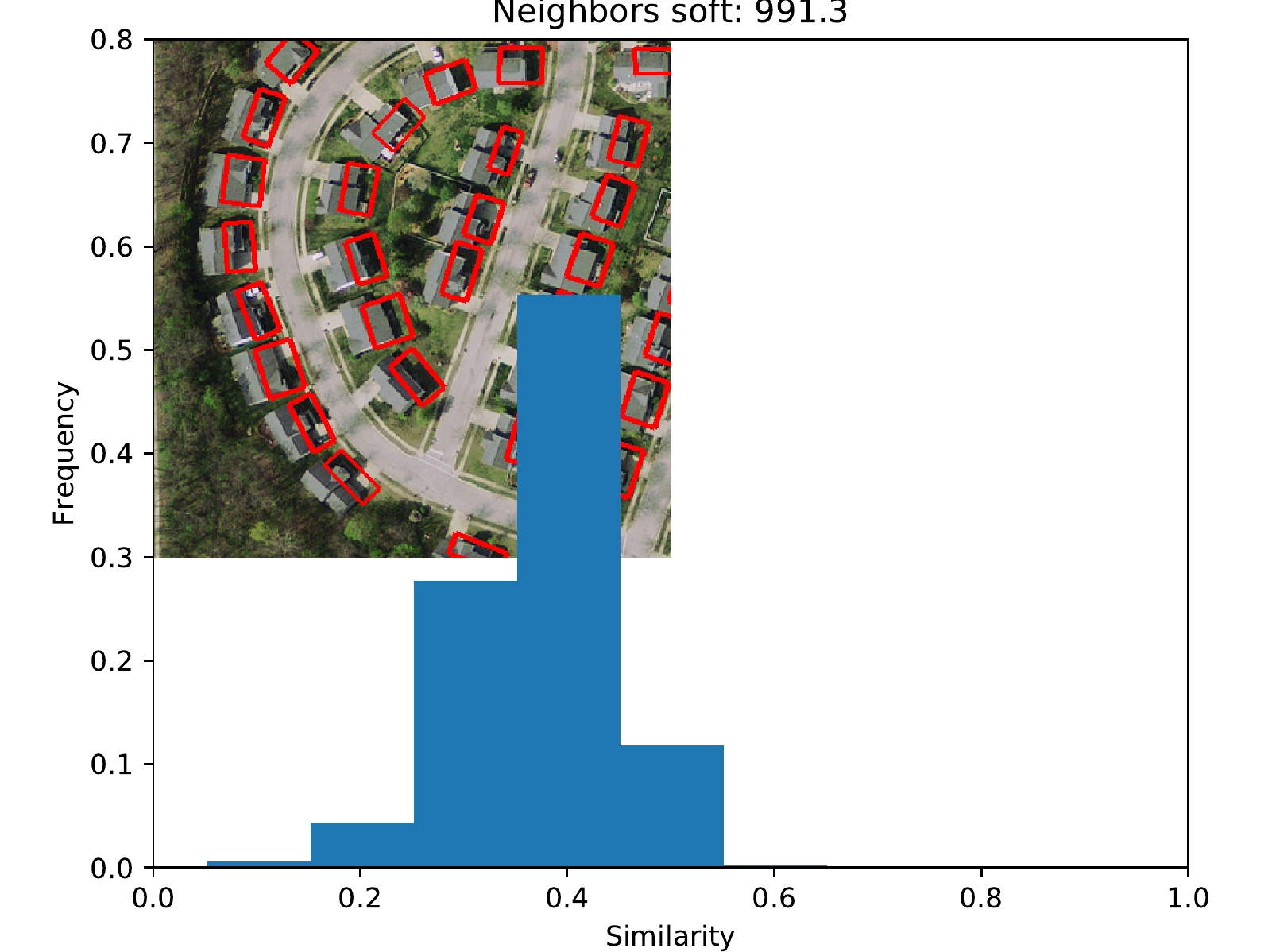}
		\includegraphics[width=0.7\linewidth]{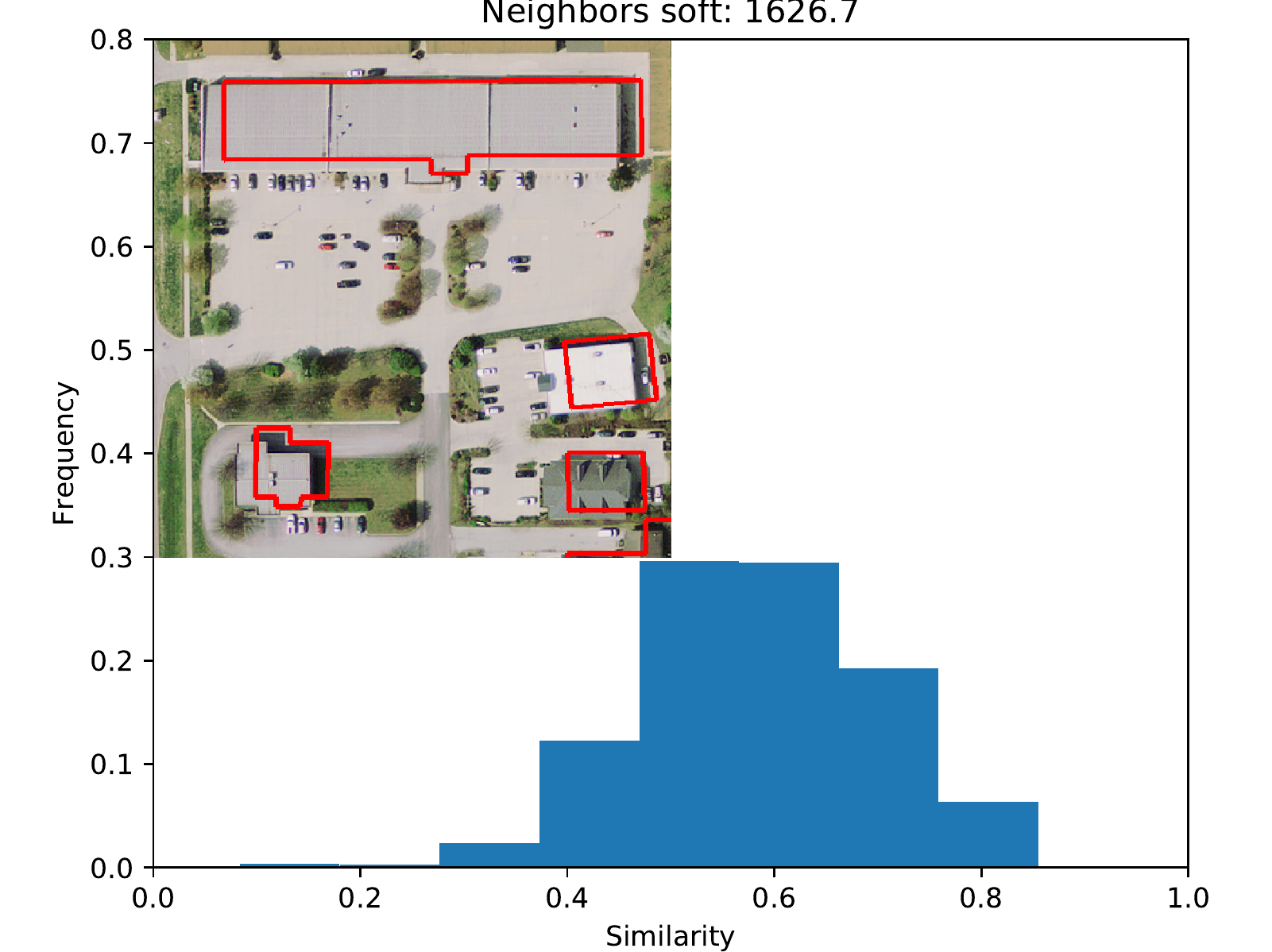}
	\end{subfigure}
	\begin{subfigure}[b]{0.3\textwidth}
		\centering
		\caption{Round 2}
		\includegraphics[width=0.7\linewidth]{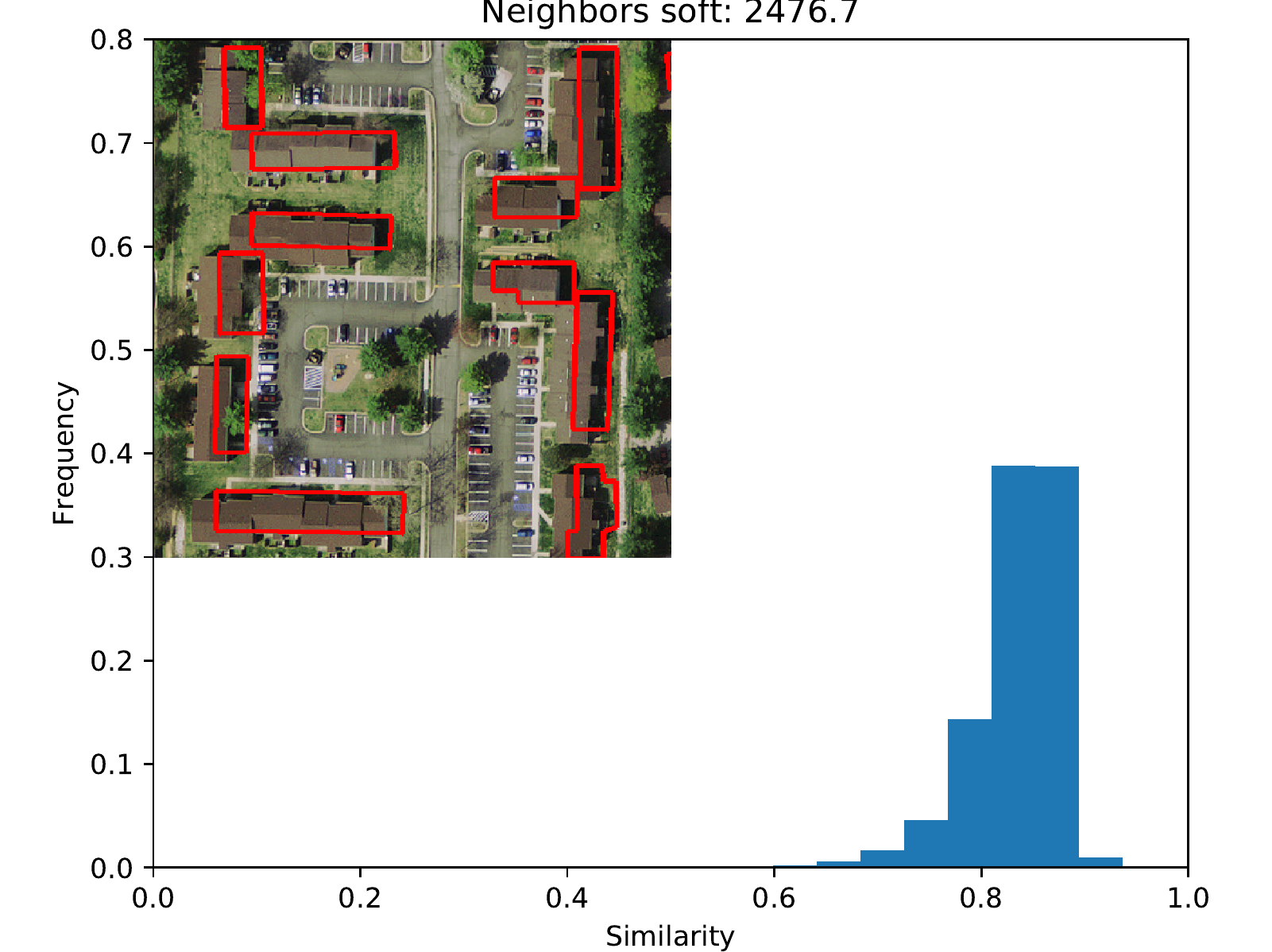}
		\includegraphics[width=0.7\linewidth]{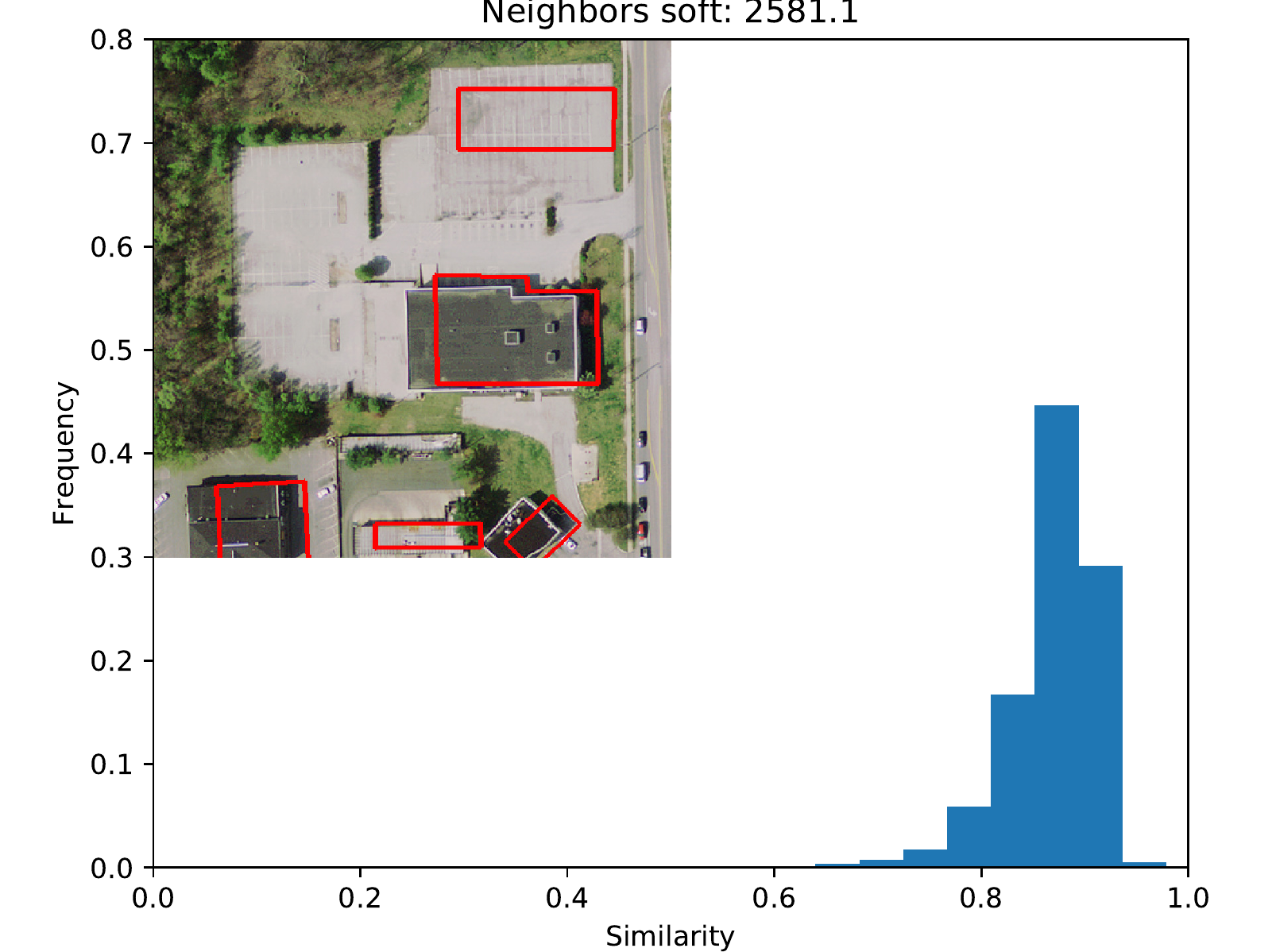}
		\includegraphics[width=0.7\linewidth]{netsimilarity_ds_fac_4_round_1_bloomington22_individual_hist_02-eps-converted-to.pdf}
		\includegraphics[width=0.7\linewidth]{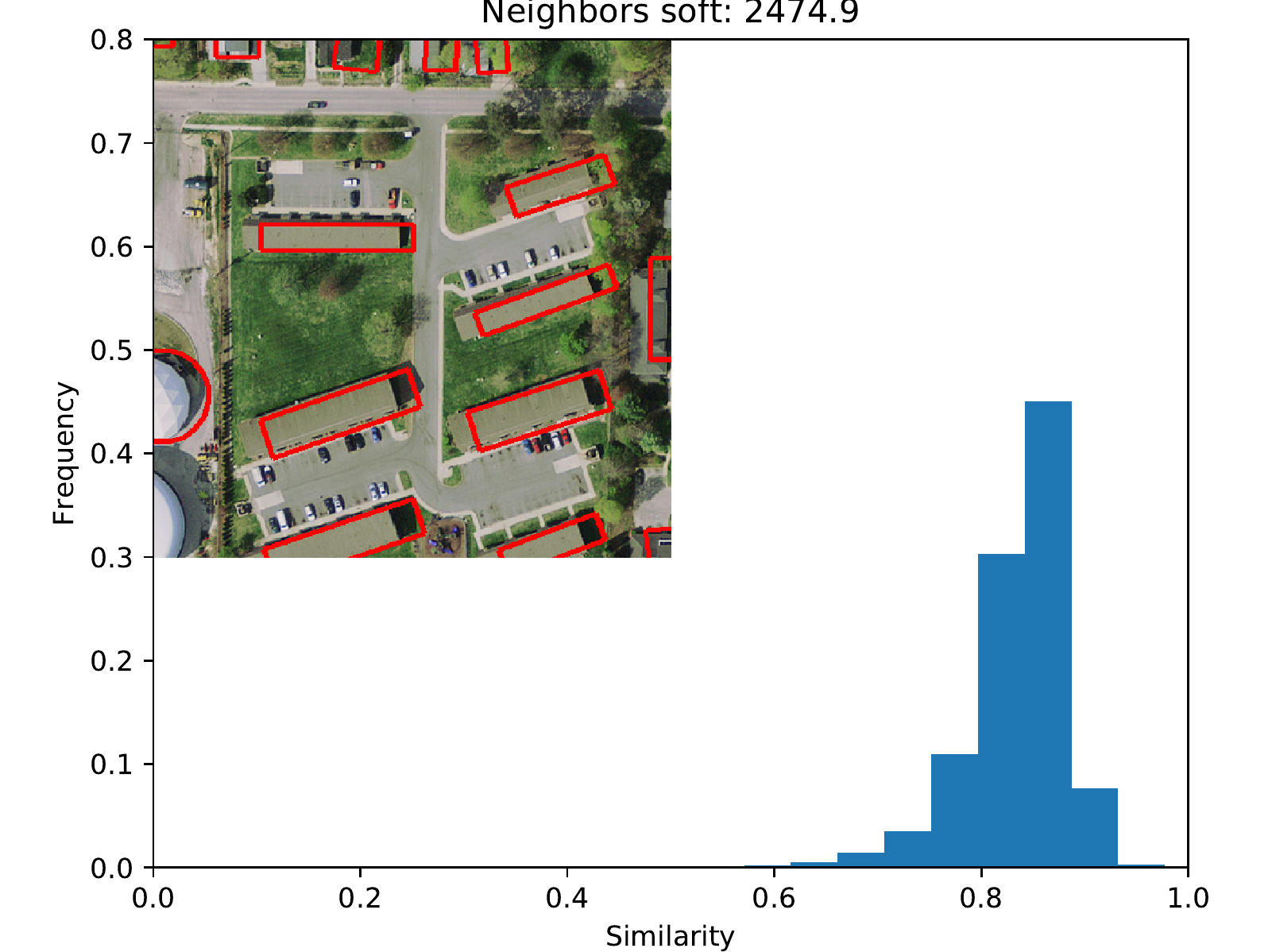}
		\includegraphics[width=0.7\linewidth]{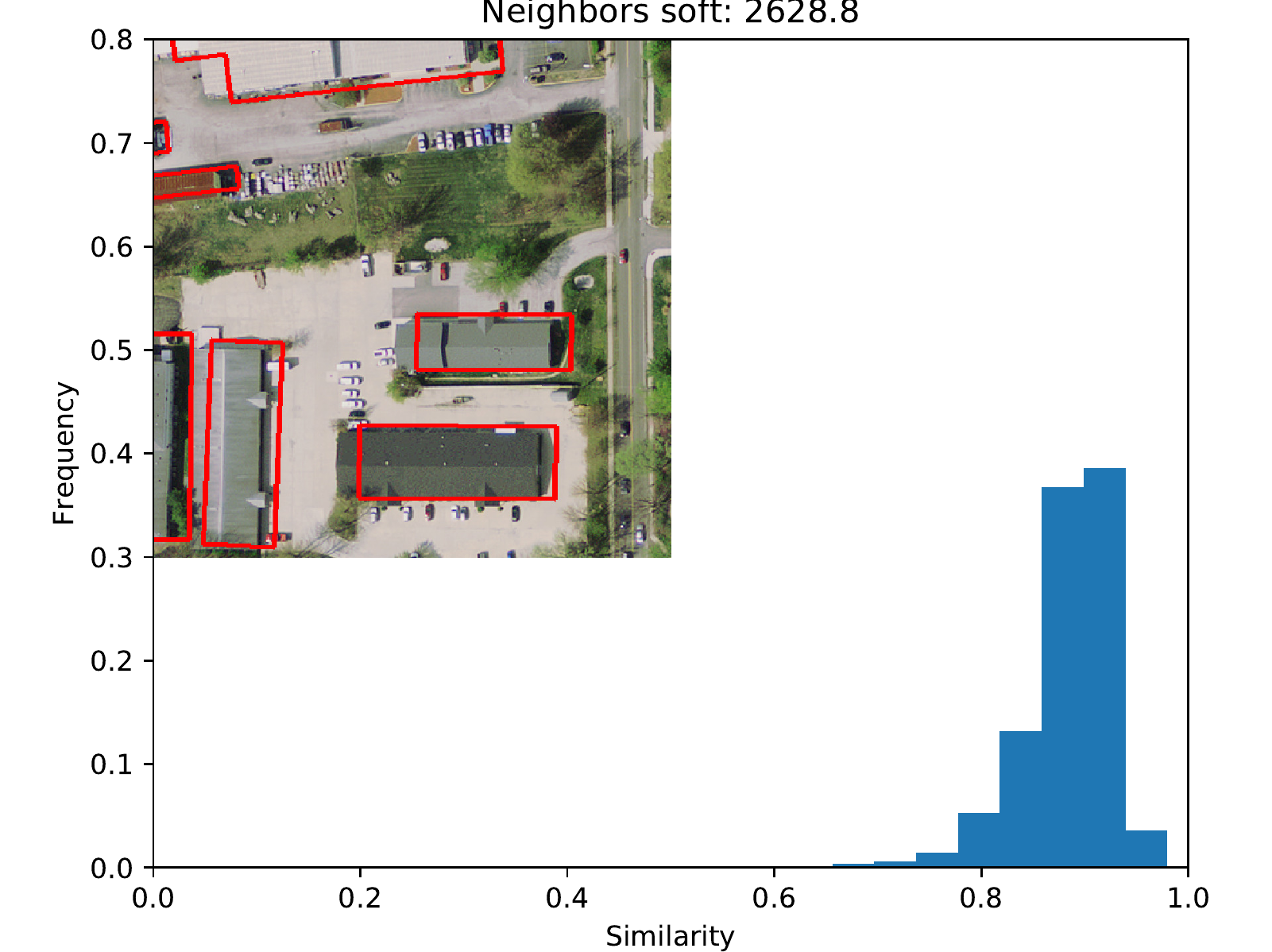}
		\includegraphics[width=0.7\linewidth]{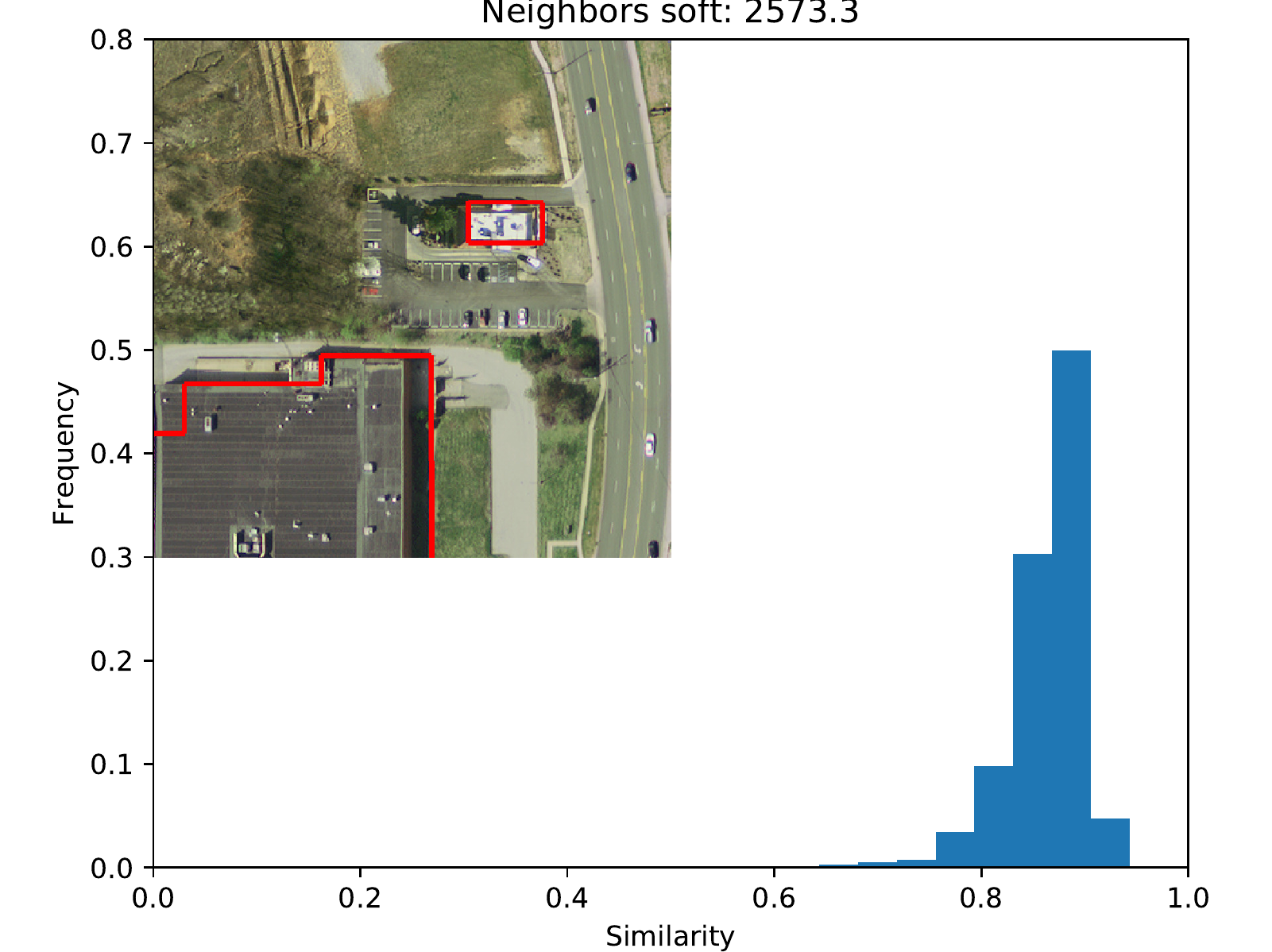}
		\includegraphics[width=0.7\linewidth]{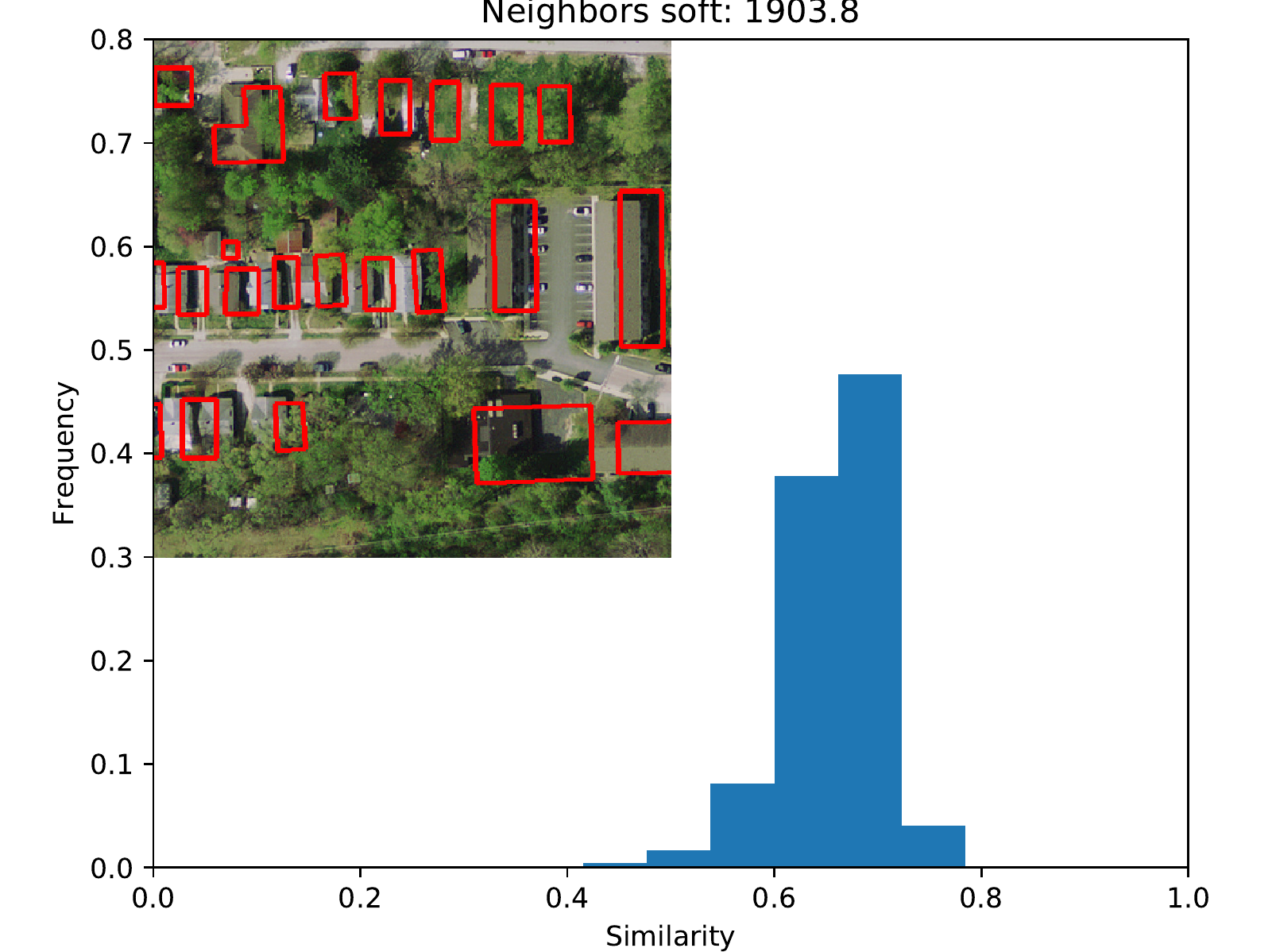}
		\includegraphics[width=0.7\linewidth]{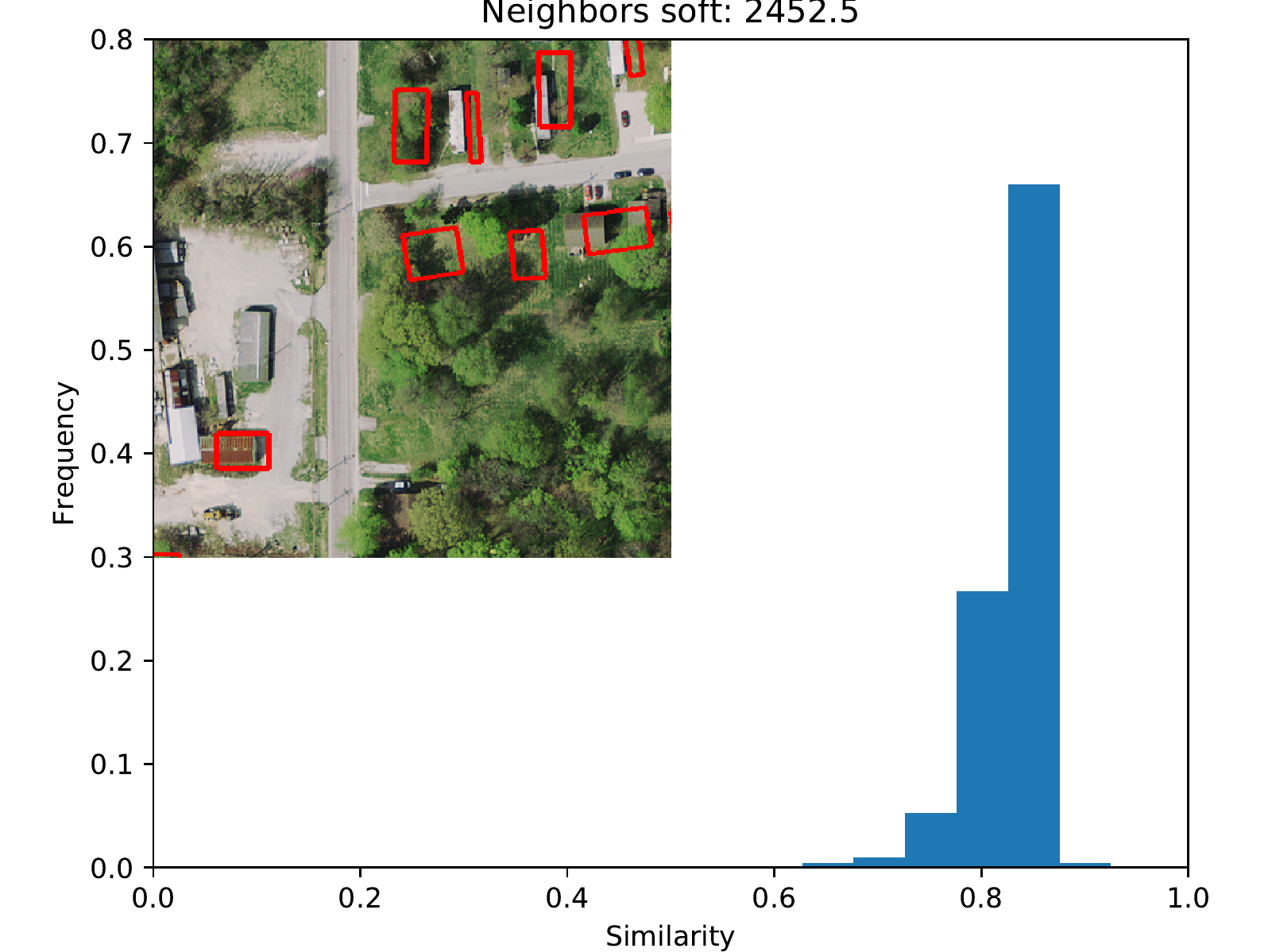}
		\includegraphics[width=0.7\linewidth]{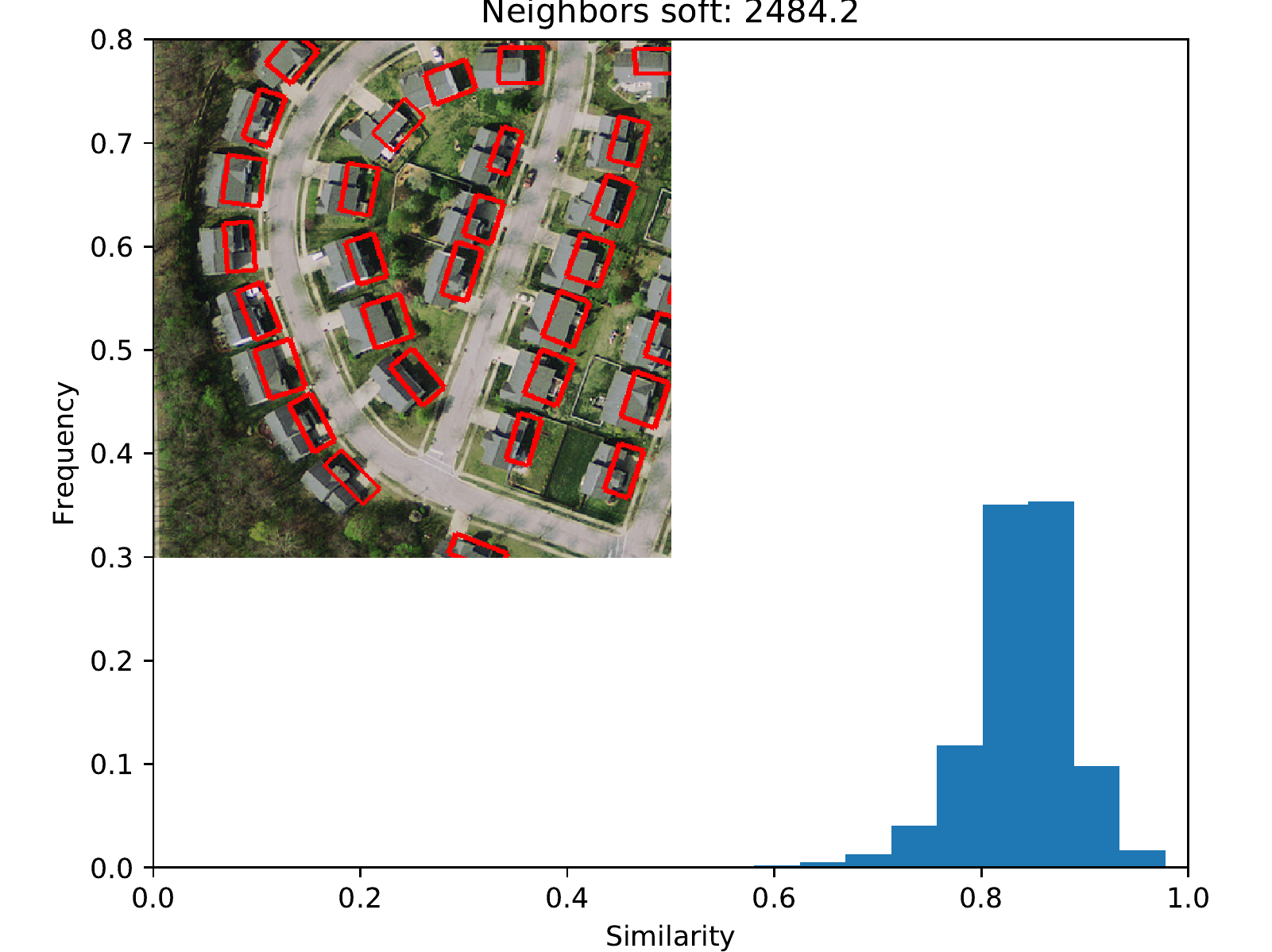}
		\includegraphics[width=0.7\linewidth]{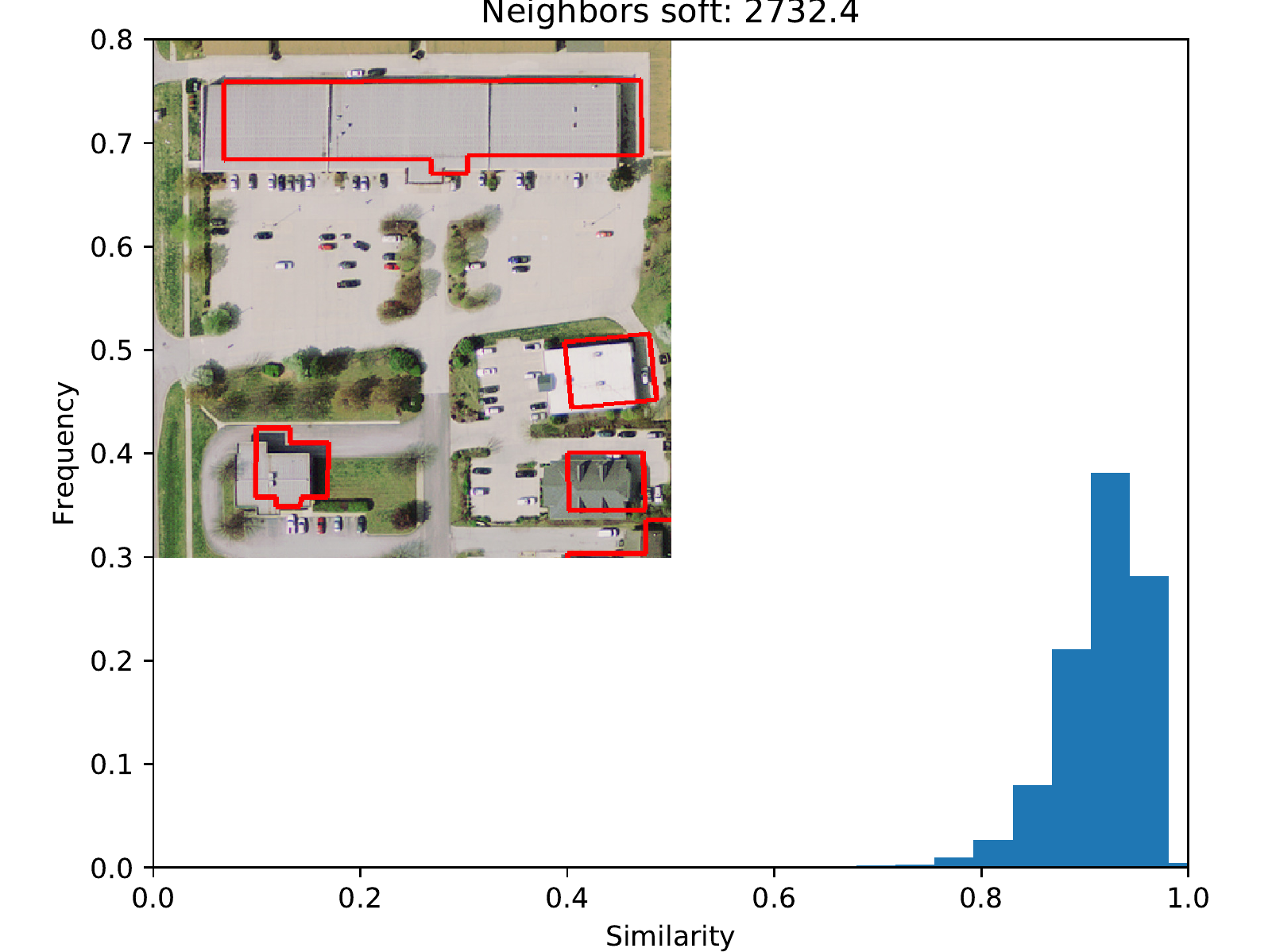}
	\end{subfigure}
	\begin{subfigure}[b]{0.3\textwidth}
		\centering
		\caption{Round 3}
		\includegraphics[width=0.7\linewidth]{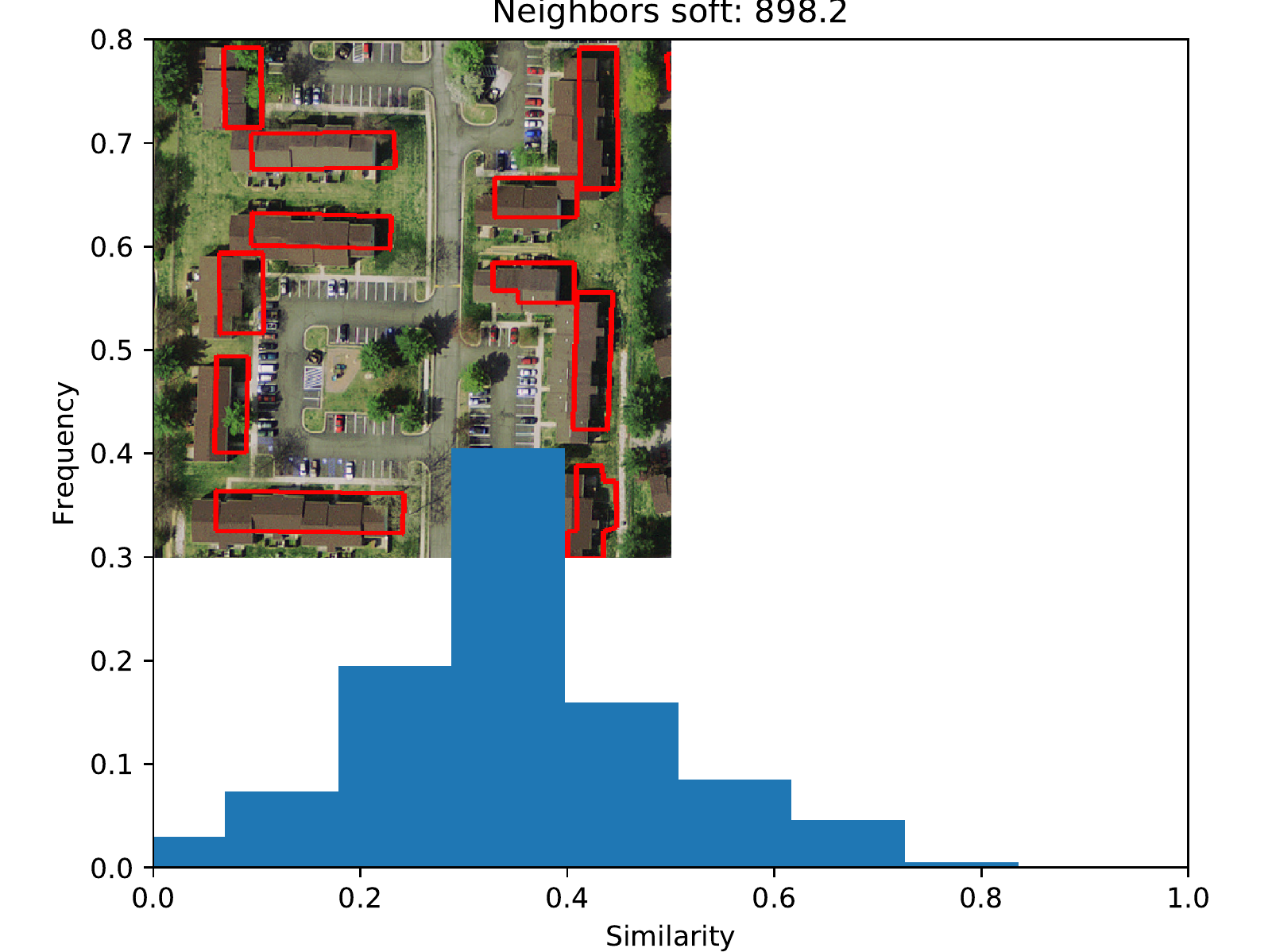}
		\includegraphics[width=0.7\linewidth]{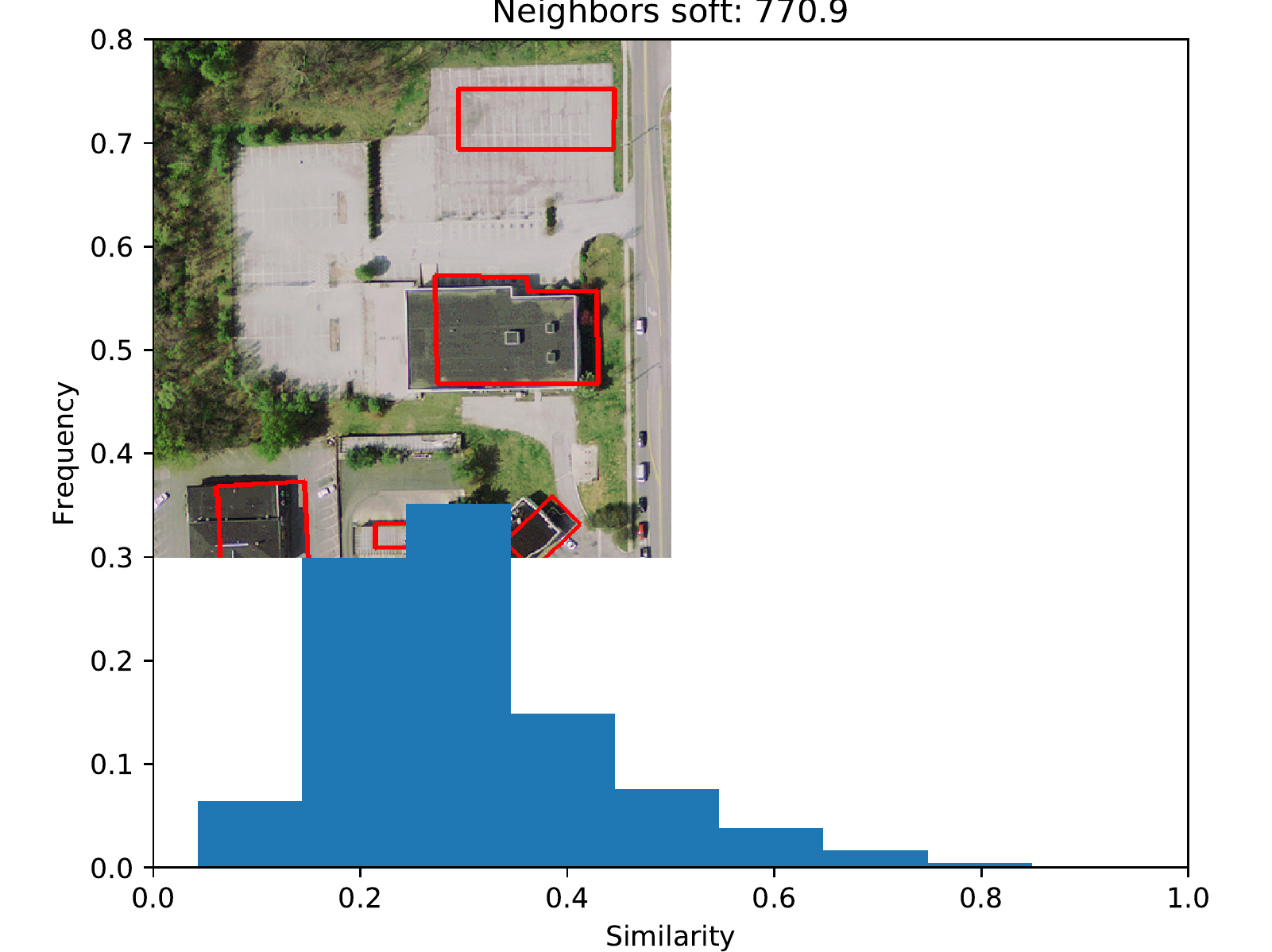}
		\includegraphics[width=0.7\linewidth]{netsimilarity_ds_fac_4_round_2_bloomington22_individual_hist_02-eps-converted-to.pdf}
		\includegraphics[width=0.7\linewidth]{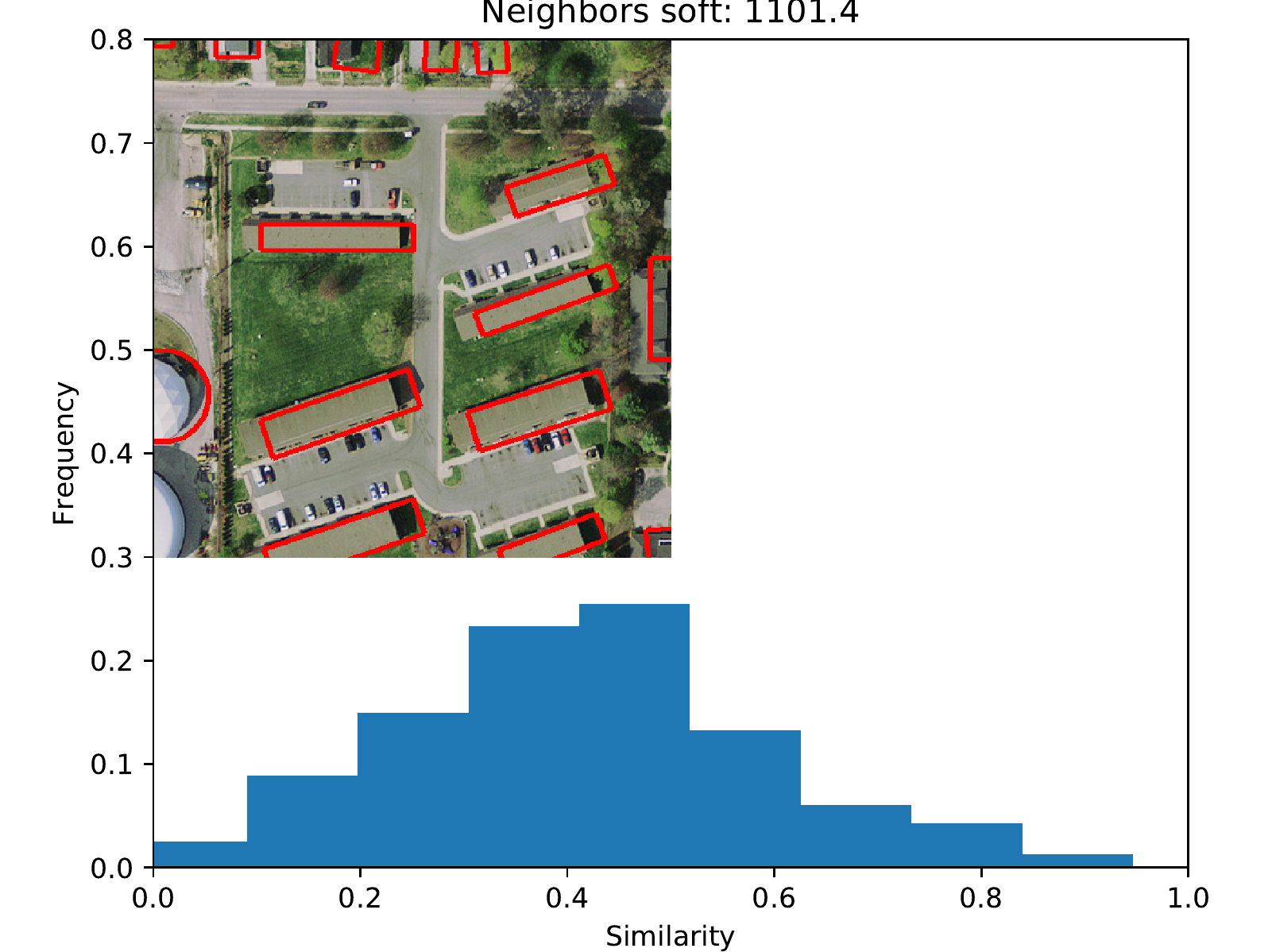}
		\includegraphics[width=0.7\linewidth]{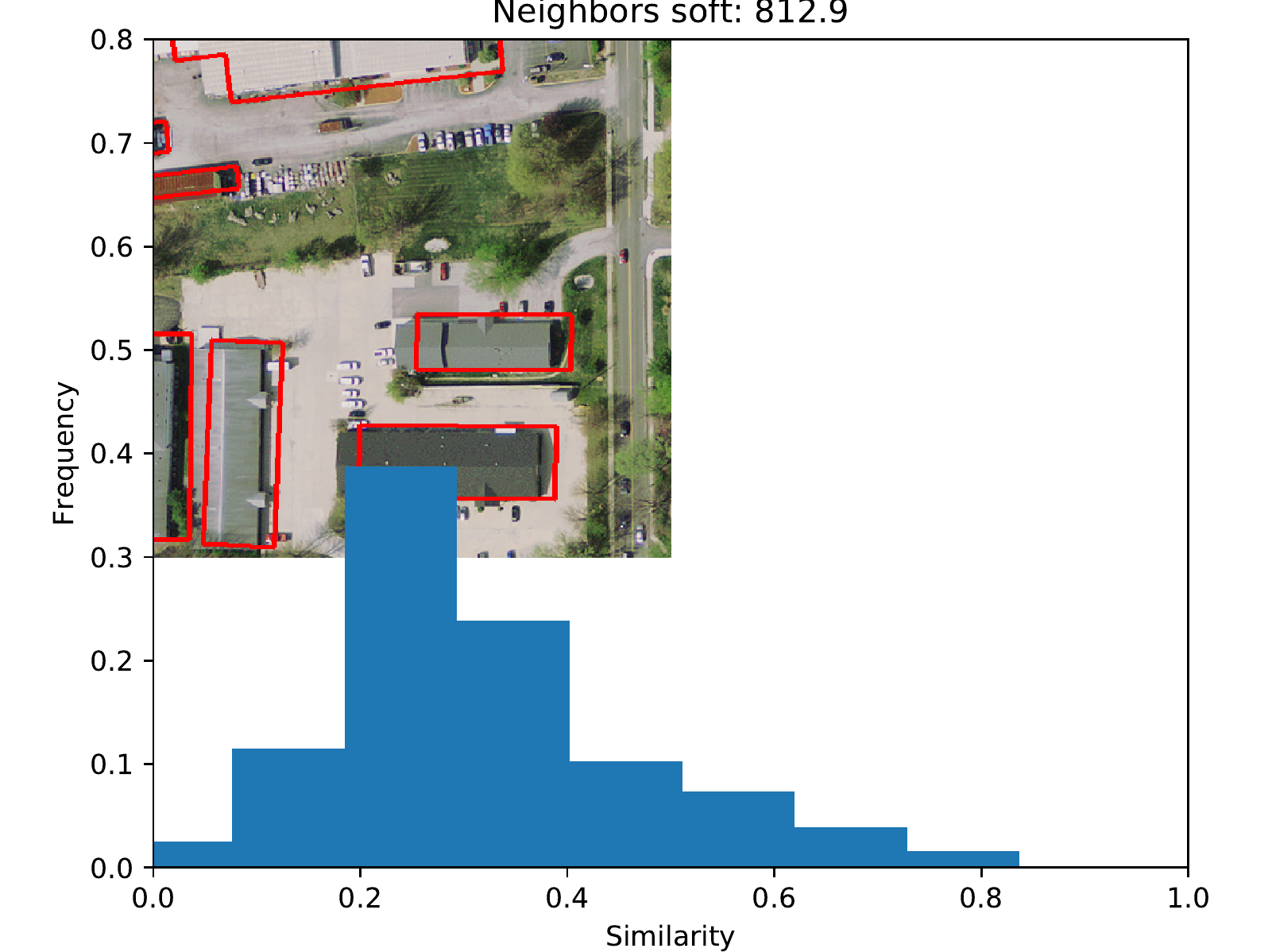}
		\includegraphics[width=0.7\linewidth]{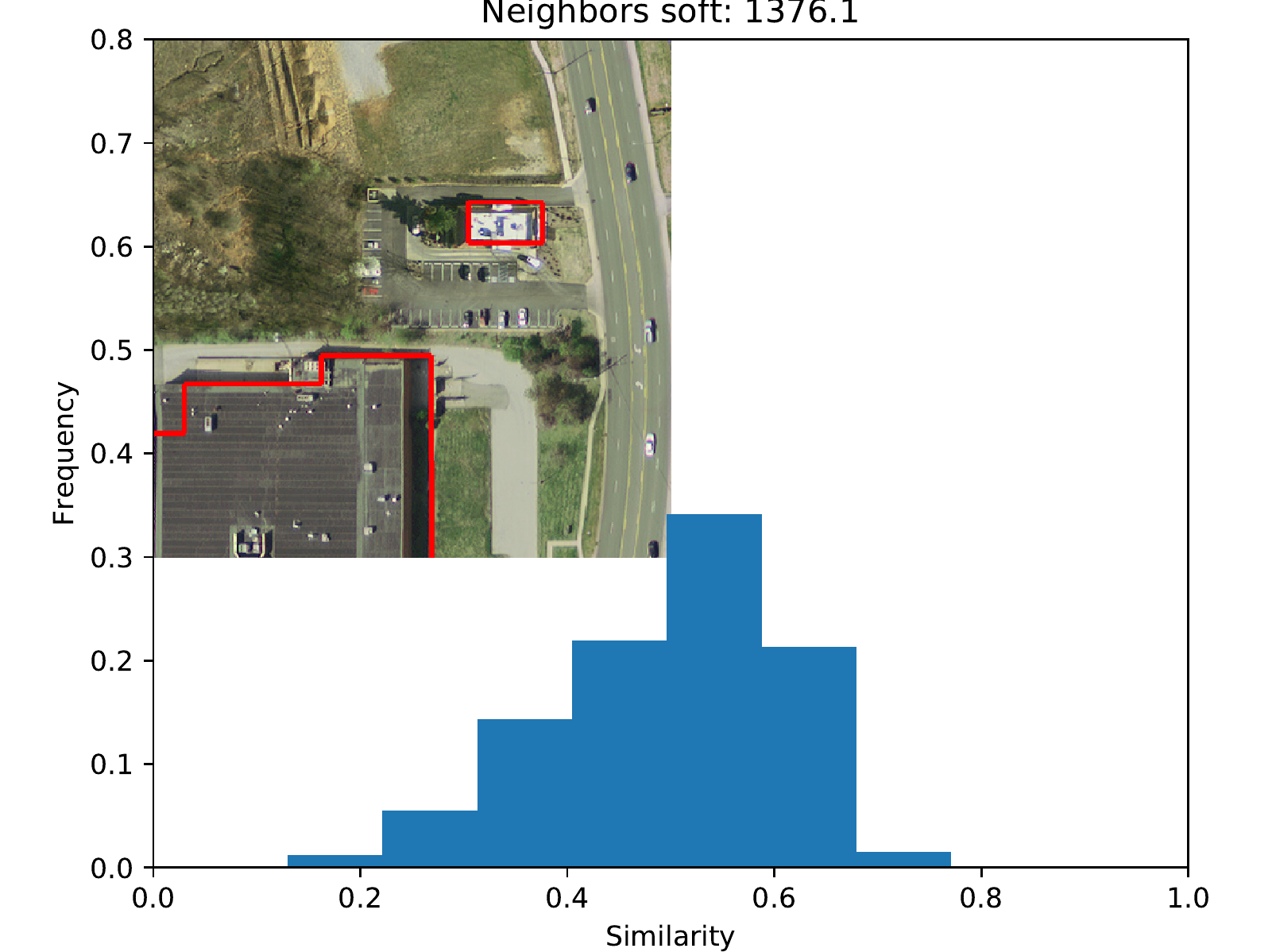}
		\includegraphics[width=0.7\linewidth]{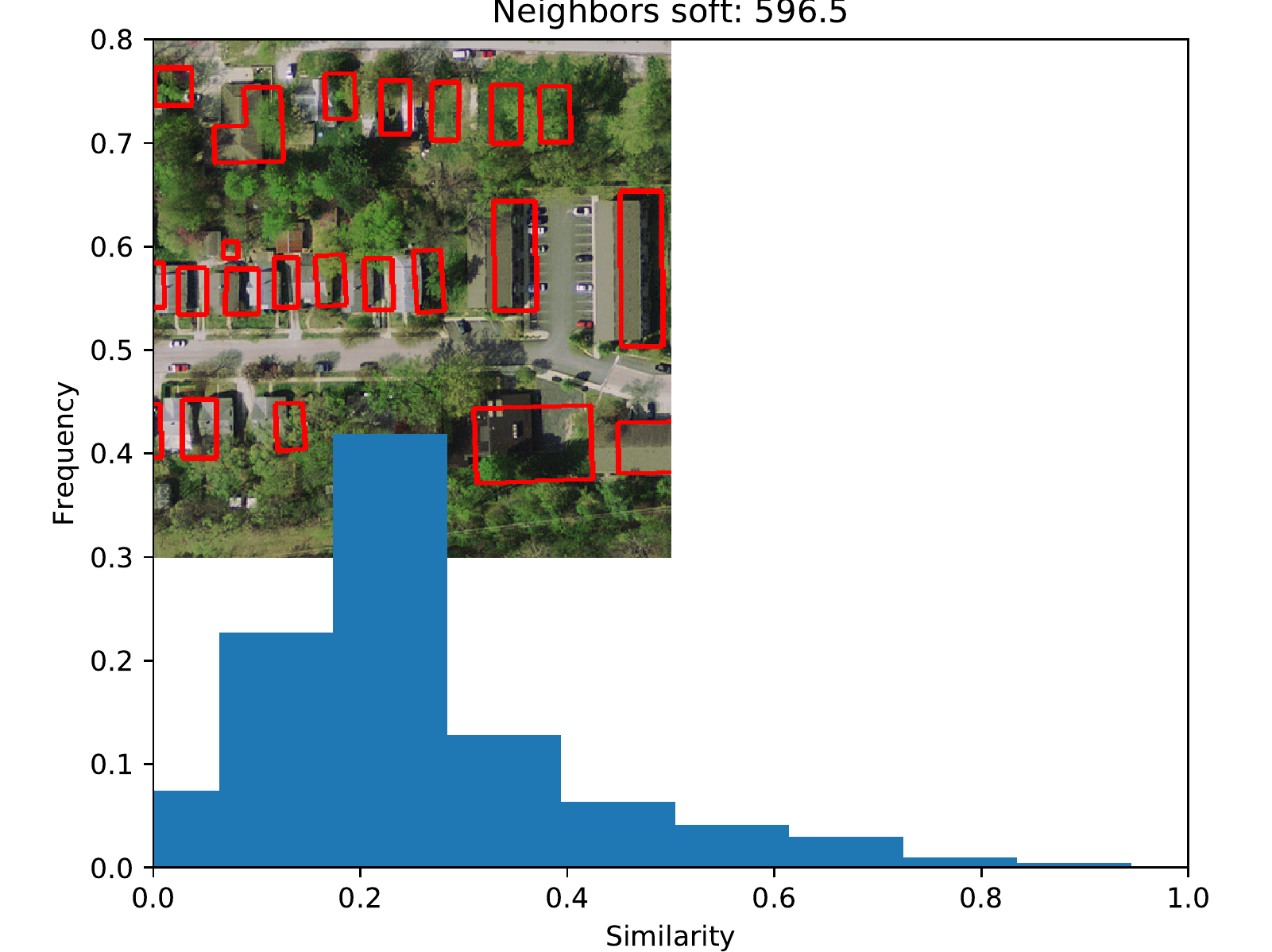}
		\includegraphics[width=0.7\linewidth]{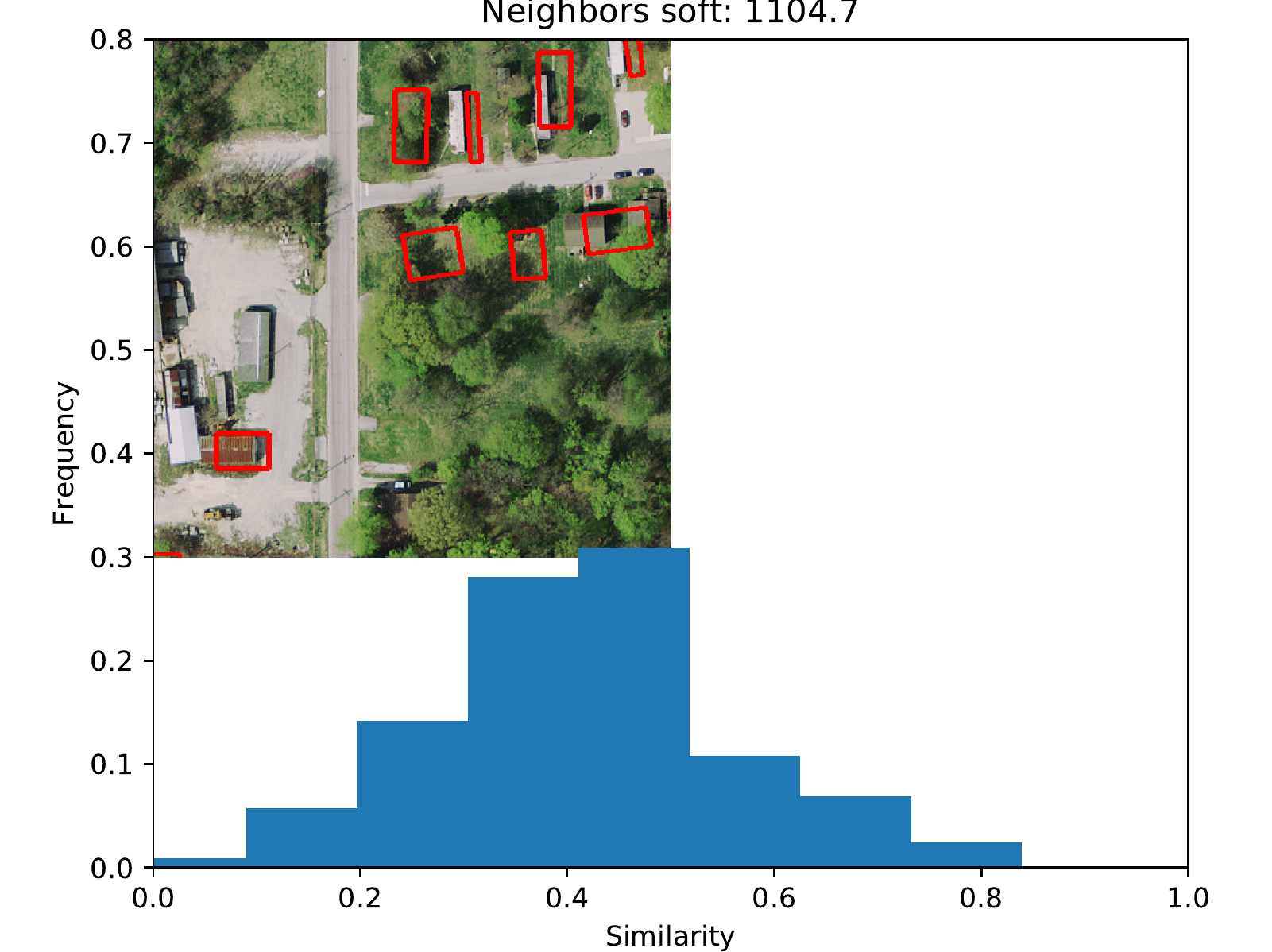}
		\includegraphics[width=0.7\linewidth]{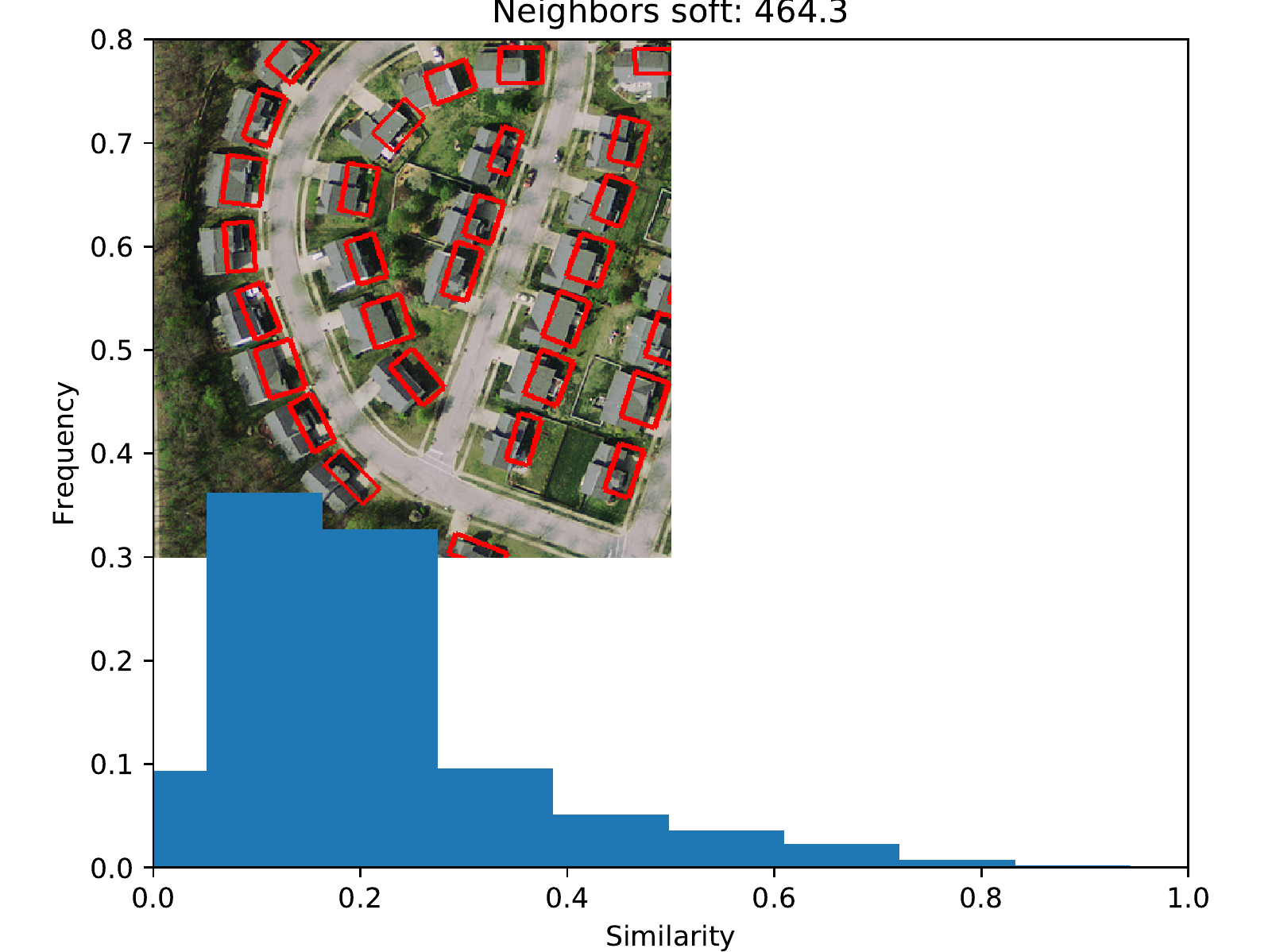}
		\includegraphics[width=0.7\linewidth]{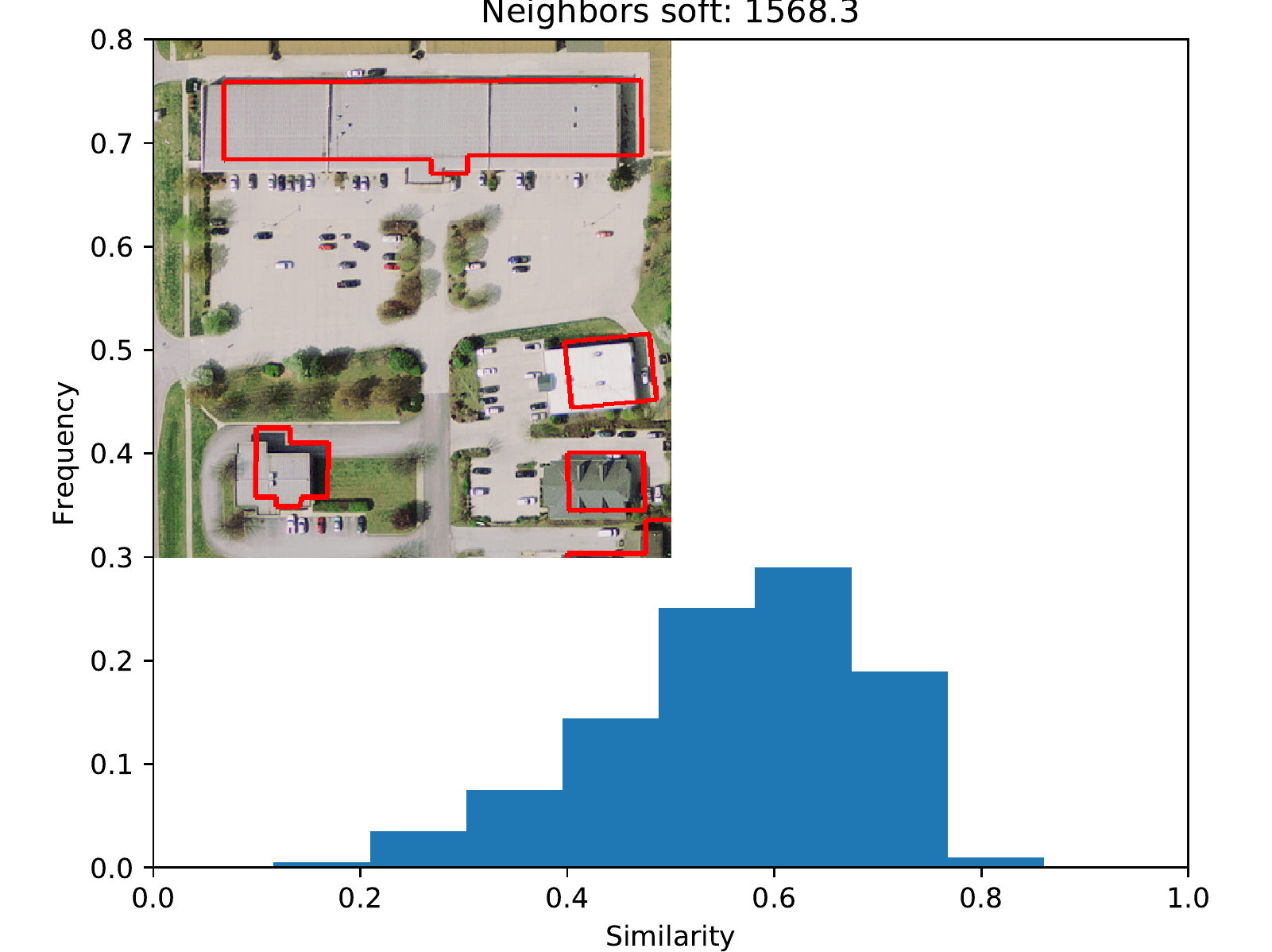}
	\end{subfigure}
	\caption{Histograms of similarities shown for the same 10 patches as in  Fig.\ref{fig:round_0_bloomington22_k_nearest}, \ref{fig:round_1_bloomington22_k_nearest}, \ref{fig:round_2_bloomington22_k_nearest}. Same patch selection across rounds.}
	\label{fig:bloomington22_individual_hist}
\end{figure}

%% file: preuves_denoising.tex
\section{Proof details of the self-denoising effect quantification}

\subsection{Magnitude of kernel-smoothed i.i.d. noise}

We show here that $\E_k[ \epsi ] \propto \var_\epsi(\E_k[ \epsi ])^{1/2} = \sigma_\epsi \, \| k^{IN}_\theta \|_{L2}$.

Let us denote by $\E_\epsi[\,]$ and $\var_\epsi(\,)$ the expectation and variance with respect to the random variable $\epsi$. As a reminder, by assumptions in the noise definition, $\epsi = (\epsi_i)_i$ is a random, i.i.d.~noise, centered and of variance $\sigma_\epsi$.

This is not to be confused with the symbol $\E_k[\,]$, which was defined as, for any vector field $a$:
$$\E_k [ a ] =\, \sum_j\, a_j\, k^{IN}_\theta(\x_j,\x_i) \; ,$$
\ie as the mean value of $a$ in the neighborhood of $i$, that is, the weighted average of the $a_j$ with weights $k^{IN}_\theta(\x_j,\x_i)$, which are positive and sum up to 1.

Given a network and its associated kernel $k^{IN}_\theta$, we are interested in to knowing the typical values of $\E_k[ \epsi ]$ for random $\epsi$.
First, the expectation over the noise of $\E_k[ \epsi ]$ is:
$$ \E_\epsi\left[ \E_k[ \epsi ] \right] \;=\; \E_\epsi\left[ \sum_j\, \epsi_j\, k^{IN}_\theta(\x_j,\x_i) \right] \; = \; \sum_j \E_\epsi[\epsi_j]\, k^{IN}_\theta(\x_j,\x_i)  \;=\; 0$$
as $\epsi$ is a centered noise.
Thus the random variable $\E_k[ \epsi ]$ is also centered, and therefore its typical values are described by its standard deviation, which is the square root of its variance:
$$\E_k[ \epsi ] \;\propto\; \var_\epsi\left(\E_k[ \epsi ]\right)^{1/2} \; .$$
The variance can be computed as follows:
\begin{align*}
  \var_\epsi\left( \E_k[ \epsi ] \right) &  \;=\; \E_\epsi\left[ \left( \sum_j\, \epsi_j\, k^{IN}_\theta(\x_j,\x_i) \right)^2 \right] \\
  & \;=\; \E_\epsi\left[ \sum_j\, \epsi_j^2\, \left( k^{IN}_\theta(\x_j,\x_i) \right)^2 \right] \;\;\;\;\;\text{as } \epsi \text{ is i.i.d.} \\
  & \;=\; \sigma^2_\epsi \, \sum_j\,  \left( k^{IN}_\theta(\x_j,\x_i) \right)^2 \\
  & \;=\; \sigma^2_\epsi\,  \left\| k^{IN}_\theta(\cdot,\x_i) \right\|^2_{L2} \; . \\
\end{align*}

As the weights $p_j = k^{IN}_\theta(\x_j,\x_i)$, for given $i$ and varying $j$, are positive and sum up to 1, they form a probability distribution. Hence the value of $\left\| k^{IN}_\theta(\cdot,\x_i) \right\|^2_{L2} = \|p\|^2_{L2}$ satisfies:
\begin{itemize}
\item $\|p\|_{L2} \leqslant 1$, $\;\;\;\;$ as $\sum_j p_j^2 \leqslant \sum_j p_j = 1$, with equality only when $p_j = p_j^2\; \forall j$, that is, all $p_j = 0$ except for one $p_{j^*} = 1$, which means $k^I_\theta(\x_j,\x_i) = 0 \;\;$ $\forall j \neq i$, which means that all data samples are fully independent from the network's point of view.
\item $\|p\|_{L2} \geqslant \frac{1}{\sqrt{N}}$ $\;\;\;\;$ as $1 = \sum_j 1 \times p_j \leqslant \| 1 \|_{L2}\; \|p\|_{L2} = \sqrt{N} \, \|p\|_{L2} $ (Cauchy-Bunyakovsky-Schwarz), with equality reached for the uniform distribution: $p_j = \frac{1}{N} \, \forall j$, where $N$ is the number of data samples. This implies that all $k^C_\theta(\x_j,\x_i)$ are equal, for all $i,j$, hence they are all equal to $k^C_\theta(\x_i,\x_i) = 1$. This is the case studied in~\cite{noise2noise}: all input points are identical.
\end{itemize}

The denoising factor $\| k^{IN}_\theta (\cdot,\x_i) \|_{L2}$, which depends on the data point $\x_i$ considered, thus expresses where the neighborhood of $\x_i$ lies, between these two extremes (all $\x_j$ very different from $\x_i$, or all identical).

Note: the results above remain valid when the output is higher-dimensional, under the supplementary assumption that the covariance matrix of the noise is proportional to the Identity matrix (\ie, the noises on the various coefficients of the label vector are independent from each other, and follow the same law, with standard deviation $\sigma_\epsi$). If not, the expression for $\mathrm{co}\!\var_\epsi\left( \E_k[ \epsi ] \right)$ is more complex, as $\Sigma_\epsi$ and $k^{IN}_\theta$ interact. Note that when the output is of dimension $d$, the kernel $k^{IN}_\theta(\x_j,\x_i)$ is a $d \times d$ matrix, thus the denoising factor $\left\| k^{IN}_\theta(\cdot,\x_i) \right\|^2_{L2}$ has to be replaced with the matrix $\sum_j k^{IN}_\theta(\x_j,\x_i)\, k^{IN}_\theta(\x_j,\x_i)^T$, which can be summarized by its trace, which is the $L^2$ norm of the Frobenius norms: $\Big\| \left\| k^{IN}_\theta(\cdot,\x_i) \right\|_F \Big\|^2_{L2}$.

\subsection{The function: gradient $\mapsto$ output is Lipschitz}

Theorem \ref{basicnet} implies that the application: $\frac{\nabla_{\!\theta} f_\theta(\x)}{\|\nabla_{\!\theta} f_\theta(\x)\|} \mapsto f_\theta(\x)$ is well-defined. We show here that this application is also Lipschitz, with a network-dependent constant, under mild hypotheses.

We consider the same assumptions as in Theorem~\ref{basicnet}~: $f_\theta$ is a real-valued network, whose last layer is a linear layer or a standard activation function thereof (such as sigmoid, tanh, ReLU...), without parameter sharing (in that last layer). We will also require that the derivative of the activation function is bounded, which is a safe assumption for all networks meant to be trained by gradient descent. Another, technical property (bounded input space) will be assumed in order to imply bounded gradients. A side note indicates how to rewrite the desired property if the input space is not bounded.

\newcommand{\uu}{\mathbf{u}}

Let $\x$ and $\x'$ be any two inputs.
We want to bound $\left|f_\theta(\x) - f_\theta(\x')\right|$ by $\|\uu - \uu'\|_2$ times some constant, where $\uu = \frac{\nabla_{\!\theta} f_\theta(\x)}{\|\nabla_{\!\theta} f_\theta(\x)\|}$ and $\uu' = \frac{\nabla_{\!\theta} f_\theta(\x')}{\|\nabla_{\!\theta} f_\theta(\x')\|}$.

Let us denote the non-normalized gradients by $\vv = \nabla_{\!\theta} f_\theta(\x)$ and $\vv' = \nabla_{\!\theta} f_\theta(\x')$. We have $\uu = \frac{\vv}{\|\vv\|}$ and $\uu' = \frac{\vv'}{\|\vv'\|}$.

We will proceed in two steps: bounding $\left|f_\theta(\x) - f_\theta(\x')\right|$ by $\|\vv - \vv'\|_2$, and then $\|\vv - \vv'\|_2$ by $\|\uu - \uu'\|_2$. The first step is easy and actually sufficient to bound with a non-normalized similarity kernel $k_\theta = \vv \cdot \vv'$  the shift from the average prediction in the neighborhood. The second step provides a more elegant bound, in that it makes use of the normalized similarity kernel $k^C_\theta = \uu \cdot \uu'$, but that bound is a priori not as tight and requires more assumptions.

\medskip

\textbf{Case where the last layer is linear} 

The output of the network is of the form $$f_\theta(\x) = \sum_i w_i a_i(\x) + b \;, $$ where $w_i$ and $b$ are parameters in $\R$ and $a_i(\x)$ activities from previous layers. Thus:
\begin{align*}
  \left|f_\theta(\x) - f_\theta(\x')\right| & = \; \left| \sum_i w_i \, (a_i(\x) - a_i(\x'))\right| \\
& \leqslant \; \| \mathbf{w} \|_2 \, \| \mathbf{a}(\x) - \mathbf{a}(\x') \|_2\\
& \leqslant \; \| \mathbf{w} \|_2 \, \sqrt{ \sum_i ( v_i - v'_i )^2 }
  \end{align*}
where the sum is taken over parameters $i$ in the last layer only, using the fact that activities $a_i$ in the last layer are equal to some of the coefficients of the gradient:  $v_i := \frac{\partial f_\theta(\x)}{\partial w_i} = a_i(\x)$.

Note that the derivative with respect to the shift $b$ is $v_b := \frac{\partial f_\theta(\x)}{\partial b} = 1$, which ensures that the norm of $\vv$ is at least 1. This implies:
$$ \| \uu - \uu' \|_2 \; \geqslant \; \left| u_b - u'_b \right| \; = \; \left| \frac{1}{\|\vv\|} - \frac{1}{\|\vv'\|} \right| $$
which, combined with:
$$ \left|v_i - v'_i\right| \;\;=\;\; \|\vv'\|\; \left|\frac{1}{\|\vv'\|} v_i - \frac{v'_i}{\|\vv'\|}\right| \;\;=\;\; \|\vv'\|\; \left|
\, u_i - u'_i + \left( \frac{1}{\|\vv'\|} -  \frac{1}{\|\vv\|} \right) v_i \, \right| $$
yields:
$$ \left|v_i - v'_i\right| \;\; \leqslant \;\;  \|\vv'\| \,\Big( \left| \, u_i - u'_i \,\right| \,+\,  \| \uu - \uu' \|_2 \,|v_i| \Big) \;\; \leqslant \;\; \|\vv'\|  \,  \| \uu - \uu' \|_2 \, (1+|v_i|)
$$ 
from which we finally obtain:
$$  \left|f_\theta(\x) - f_\theta(\x')\right| \; \leqslant \; \left[ \| \mathbf{w} \|_2 \, \|\vv'\|  \,  \sqrt{ \sum_i (1+|v_i|)^2 } \right] \,  \| \uu - \uu' \|_2 $$
which is the bound we were searching for. For the term between brackets to be bounded by a network-dependent constant, one can suppose for instance 
that the derivative of the activation functions is bounded (which is usually the case for networks meant to be trained by gradient descent),
and that the input space is bounded as well; in such cases indeed
all coefficients of the gradient vector $\vv$ or $\vv'$ are bounded, as derivatives of a function composed of constant linear applications (except for the first layer which is a linear application whose factors are bounded inputs, when seen as an application defined on parameters) and of bounded-derivatives activation functions.

\textbf{Note for unbounded input spaces: } If the input space is not bounded, the gradients are not bounded absolutely, as for instance the gradient with respect to a weight in the first layer is the input itself (times a chain product). In that case the application $\x \mapsto \vv$ still satisfy a bound of the form $\|\vv\| \leqslant (1+\|\x\|)\,A$, with $A$ a network-dependent constant (product of determiners of layer weight matrices and of the bound on activation function derivatives to the power: network depth), and thus the application $\uu \mapsto f_\theta(\x)$ still satisfies a bound of the form, for any $\x$, $\x'$: 
$$  \left|f_\theta(\x) - f_\theta(\x')\right| \;\; \leqslant \;\; B \;(1+\|\x\|)\; (1+\|\x'\|)  \;  \| \uu - \uu' \|_2 \; .$$

The last statement in the paper then becomes
$$\big|\,\E_k[ \yh_i - \yh ]\,\big| \;\;\leqslant\;\; \sqrt{2}\, B \;(1+\|\x_i\|)\; \max_j (1+\|\x_j\|) \; \E_k\!\left[ \sqrt{1 - k_\theta^C(\x_i,\cdot)}\,\right]  $$
which in practice rewrites as the original formulation:
$$\big|\,\E_k[ \yh_i - \yh ]\,\big| \;\;\leqslant\;\; \sqrt{2}\, C\, \E_k\!\left[ \sqrt{1 - k_\theta^C(\x_i,\cdot)}\,\right]  $$
by taking $C = B \max_j \left (1+\|\x_j\| \right)^2$, considering the actual diameter of the given dataset.

\medskip

\textbf{Case where the last layer is an activation function of a linear layer} 

The output of the network is of the form
$$f_\theta(\x) = \sigma\left( \sum_i w_i a_i(\x) + b \right)\;, $$
and, as the derivative of $\sigma$ is assumed to be bounded, and as the weights $w_i$ are fixed, $f_\theta(\x)$ is a Lipschitz function of the last layer activities $a_i(\x)$. Therefore:
\begin{align*}
  \left|f_\theta(\x) - f_\theta(\x')\right|
 & \leqslant \; K \, \| \mathbf{a}(\x) - \mathbf{a}(\x') \|_2 \; .
  \end{align*}
We will denote by $\alpha$ and $\alpha'$ the derivatives with respect to the shift $b$, which are this time:
$$\left. \alpha := \vv_b := \frac{\partial f_\theta(\x)}{\partial b} = \sigma'\right|_{\sum_i w_i a_i(\x) + b}  \;\;\;\;\;\;\;\; \text{and} \;\;\;\;\;\;\;\; \left. \alpha' := \vv'_b := \frac{\partial f_\theta(\x')}{\partial b} = \sigma'\right|_{\sum_i w_i a_i(\x') + b} \; .$$

We proceed as previously:
$$ \| \uu - \uu' \|_2 \; \geqslant \; \left| u_b - u'_b \right| \; = \; \left| \frac{\alpha}{\|\vv\|} - \frac{\alpha'}{\|\vv'\|} \right| $$
which, combined with:
$$ \left|a_i - a'_i\right| \;\;=\;\; \left|\frac{v_i}{\alpha} - \frac{v'_i}{\alpha'}\right| \;\;=\;\; \frac{\|\vv'\|}{\alpha'}\; \left|\frac{\alpha'}{\alpha \|\vv'\|} v_i - \frac{v'_i}{\|\vv'\|}\right| \;\;=\;\; \frac{\|\vv'\|}{\alpha'}\; \left|
\, u_i - u'_i + \frac{v_i}{\alpha}\left( \frac{\alpha'}{\|\vv'\|} -  \frac{\alpha}{\|\vv\|} \right) \, \right| $$
yields:
$$ \left|a_i - a'_i\right| \;\; \leqslant \;\;  \frac{\|\vv'\|}{\alpha'} \,\Big( \left| \, u_i - u'_i \,\right| \,+\,  \| \uu - \uu' \|_2 \,|a_i| \Big) \;\; \leqslant \;\; \frac{\|\vv'\|}{\alpha'} \, (1+|a_i|) \,  \| \uu - \uu' \|_2 
$$ 
from which we finally obtain:
$$  \left|f_\theta(\x) - f_\theta(\x')\right| \; \leqslant \; \left[ K  \, \frac{\|\vv'\|}{\alpha'}  \,  \sqrt{ \sum_i (1+|a_i|)^2 } \right] \,  \| \uu - \uu' \|_2 \; . $$
Note that $\alpha'$ is actually a factor of each coefficient of $\vv'$, as the derivative of $f_\theta(\x')$ with respect to any parameter is a chain rule starting with $\left.\frac{\partial f_\theta(\x')}{\partial b} = \sigma'\right|_{\sum_i w_i a_i(\x') + b} = \alpha'$.
To bound the term between brackets, the same assumptions as previously are sufficient. One can assume that $\alpha$ and $\alpha'$ are not 0, as, if they are, the problem is of little interest ($\uu$ or $\uu'$ being then not defined).

\subsection{Additional proof detail}

The kernel $k_\theta^C(\x,\x')$, by definition, is the $L^2$ inner product between two unit vectors:
$$k_\theta^C(\x,\x') = \frac{\nabla_{\!\theta} f_\theta(\x)}{\|\nabla_{\!\theta} f_\theta(\x)\|} \cdot \frac{\nabla_{\!\theta} f_\theta(\x')}{\|\nabla_{\!\theta} f_\theta(\x')\|} \, .$$

As, for any two unit vectors $a$ and $b$:
$$\| a-b \|^2 \;=\; a^2 + b^2 - 2\, a\cdot b \;=\; 2 \,(1 - a \cdot b) \; ,$$

we get:

$$ \left\| \frac{\nabla_{\!\theta} f_\theta(\x)}{\|\nabla_{\!\theta} f_\theta(\x)\|} - \frac{\nabla_{\!\theta} f_\theta(\x')}{\|\nabla_{\!\theta} f_\theta(\x')\|}  \right\| = \sqrt{2} \sqrt{1 - k_\theta^C(\x,\x')} \; .$$